\documentclass[10pt, oneside]{article}  
\usepackage{subfiles}
\usepackage{xcolor}
\usepackage{amsmath,amsfonts,amssymb,amsthm, mathtools,xfrac,bm}
\usepackage{makecell}
\newtheorem{theorem}{Theorem}[section]

\newtheorem{proposition}[theorem]{Proposition}
\newtheorem*{proposition*}{Proposition}

\newtheorem*{remark}{Remark}

\usepackage{textgreek}
\usepackage{pifont}
\usepackage{float}
\newcommand{\cmark}{\text{\ding{51}}}
\newcommand{\xmark}{\text{\ding{55}}}
\usepackage{algorithmic}
\usepackage{graphicx}
\usepackage{booktabs}
\usepackage[utf8]{inputenc}
\usepackage{textgreek}
\usepackage[english]{babel}
\usepackage[mathlines]{lineno}
\usepackage{listings}
\usepackage{textcomp}
\usepackage{stmaryrd}
\usepackage{booktabs}
\usepackage{siunitx}[=v2]
\usepackage[linesnumbered,ruled,vlined]{algorithm2e}
\usepackage[affil-it]{authblk}
\usepackage[outdir=./Figures/epstopdf/]{epstopdf}
\numberwithin{equation}{section}
\usepackage{siunitx}
\usepackage[numbers]{natbib}
\usepackage[unicode, pdfpagelabels,bookmarks,hyperindex,hyperfigures]{hyperref}
\usepackage{doi}
\usepackage{cleveref}

\usepackage{multicol}
\usepackage{multirow}
\usepackage{commath2}
\usepackage{nicefrac,xfrac}
\usepackage{graphicx}
\usepackage{subcaption}
\usepackage[margin=1in]{geometry}
\usepackage{placeins}
\usepackage{booktabs}
\usepackage{xcolor}

\definecolor{darkred}{RGB}{180,30,30}
\usepackage{tikz}
\usetikzlibrary{positioning,arrows.meta}
\usepackage{tcolorbox}
\tcbuselibrary{theorems, skins, breakable}
\usepackage{enumitem}
\newtcolorbox{keybox}[1][]{
  enhanced, breakable,
  colback=blue!4!white, colframe=blue!50!black,
  fonttitle=\bfseries, title={#1},
  left=6pt, right=6pt, top=4pt, bottom=4pt
}
\usepackage{tikz}
\usetikzlibrary{positioning, arrows.meta}
\usetikzlibrary{decorations.pathreplacing, positioning}
\usetikzlibrary{math,quotes,arrows,arrows.meta,3d,shapes,positioning}
\setcitestyle{numbers} 
 
\graphicspath{{figs},{files/figs},{files/figs},{figs}}
\usepackage{cleveref}

\author[1]{Abhishek Srivastava \thanks{\,Equal contribution. Email:\texttt{abhiresearch24@gmail.com
}}}
\author[1]{Arijit Hazra \thanks{\,Equal contribution. \,Corresponding author. Email:\texttt{ahazra@iitpkd.ac.in}}}
\author[2]{Rajesh Dubbaku \thanks{Email:\texttt{rajeshdubbaku95@gmail.com}}}

\affil[1]{\small Department of Mechanical Engineering, Indian Institute of Technology Palakkad}
\affil[2]{\small Department of Mathematics, Indian Institute of Technology Palakkad}
\title{Variational objectives for amortized Bayesian inference in inverse problems: The role of posterior conditioning}

\date{}

\begin{document}
\sloppy
\maketitle
\tableofcontents

\begin{abstract}
	Variational autoencoders (VAEs) offer an efficient approach to amortized
	Bayesian inference for inverse problems, but posterior accuracy can depend
	strongly on the choice of variational regularization, particularly when
	the inverse problem contains weakly identified parameter directions. This
	study investigates three objectives: a reverse Kullback--Leibler
	formulation (VAE-KL), an asymmetric Jensen--Shannon formulation (VAE-JS),
	and a Jensen--Shannon--Wasserstein formulation (VAE-JSWA), which replaces
	the reverse Kullback--Leibler regularizer with the squared 2-Wasserstein
	distance while retaining forward-Kullback--Leibler posterior supervision.
	A full-covariance Gaussian encoder and a pre-trained physics-based
	surrogate are used for amortized posterior inference. A local
	linear--Gaussian analysis in the generalized Fisher basis is developed to
	characterize the variance-dependent gradients of the three objectives.
	The formulations are first evaluated using linear--Gaussian benchmarks
	with known posterior solutions and subsequently tested on nonlinear
	physics-based inverse problems, including an inverse problem governed by a
	linear ODE and two PDE-constrained problems. VAE-KL performs slightly better than the other formulations in the well-conditioned benchmark, where all three approaches yield comparable posterior approximations, whereas VAE-JSWA provides substantially lower posterior errors in the strongly ill-conditioned benchmark. The nonlinear physics-based problems exhibit a similar conditioning-dependent trend, with JS-based formulations providing greater benefit as posterior ill-conditioning increases. These results indicate that posterior conditioning is an important factor in selecting variational objectives and motivate geometry-adaptive variational inference for Bayesian inverse problems.
\end{abstract}


\section{Introduction}
\label{sec:introduction}

Inverse problems arise in a wide range of scientific and engineering disciplines where the objective is to infer unknown model parameters from observed measurements. Given a forward model that maps parameters and independent variables to observable quantities, the inverse problem seeks to recover the parameters from typically noisy observations. Such problems are encountered in medical imaging modalities including MRI \cite{uecker_image_2008, lustig_sparse_2007, narnhofer_bayesian_2022} and CT \cite{sidky_image_2008}, subsurface characterization in geophysics \cite{virieux_overview_2009, fichtner_adjoint_2006b}, image deblurring \cite{hansen_deblurring_2006, babacan_variational_2009}, non-destructive testing \cite{shull_nondestructive_2002}, elastography \cite{doyley_modelbased_2012}, inverse acoustic and electromagnetic scattering \cite{colton_inverse_2019}, data assimilation \cite{evensen_data_2009}, climate modeling \cite{lorenc_analysis_1986}, structural health monitoring and damage detection \cite{friswell_damage_2007}, electrical impedance tomography \cite{borcea_electrical_2002a, kaipio_statistical_2000}, biomechanics \cite{gokhale_solution_2008}, and thermal systems~\cite{ozisik_inverse_2021, higdon_bayesian_2008}, among many others.

Inverse problems are almost always ill-posed in the sense of Hadamard --- solutions may fail to exist, may not be unique, or may be unstable with respect to a small amounts of measurement noise, making naive inversion unreliable. To overcome these challenges, two broad strategies have been developed. Regularization-based methods~\cite{engl_regularization_1996, vogel_computational_2002} transform the ill-posed problem into an approximately well-posed one by introducing penalty terms that encode prior knowledge or smoothness assumptions. While computationally efficient, these approaches yield a single point estimate without any estimate of uncertainty in the recovered parameters. The Bayesian approach~\cite{dashti_bayesian_2017, stuart_inverse_2010, kaipio_statistical_2005, higdon_bayesian_2008}, instead produces a full probability distribution over the parameters, combining prior knowledge with observed data. That distribution is what makes rigorous uncertainty quantification possible --- essential for risk assessment, decision-making, and model validation in engineering.

Despite its theoretical appeal, Bayesian inference for inverse problems faces significant computational challenges. The posterior distribution is rarely available in closed form, and standard sampling methods such as Markov Chain Monte Carlo (MCMC)~\cite{robert_monte_2004}--including the Metropolis--Hastings algorithm~\cite{metropolis_equation_1953, hastings_monte_1970}, Hamiltonian Monte Carlo~\cite{neal_mcmc_2011}, and adaptive variants like the Delayed Rejection Adaptive Metropolis (DRAM) algorithm~\cite{haario_dram_2006}--require repeated evaluation of the forward model. This can be prohibitively expensive when the forward model involves the numerical solution of partial differential equations. Sampling methods
also struggle in high-dimensional parameter spaces, where mixing is slow and
convergence requires careful tuning and extensive diagnostic analysis ~\cite{gelman_inference_1992, brooks_general_1998}.
Variational inference (VI)~\cite{blei_variational_2017, jordan_introduction_1998} addresses this by recasting posterior computation as an
optimization problem, searching for a tractable
distribution that best approximates the true posterior ---
and is the approach pursued in this paper.



The past decade has produced a wide range of data-driven approaches to inverse
problems. Physics-informed neural networks (PINNs)~\cite{raissi_physicsinformed_2019}  embed governing equations directly into the training loss, enabling simultaneous solution of forward and inverse problems. Neural operators -- DeepONet~\cite{lu_deeponet_2021} and the Fourier Neural Operator (FNO)~\cite{li_fourier_2021a}, learn mappings between function spaces and serve as surrogate forward models that make likelihood evaluations orders of magnitude cheaper.

For uncertainty-aware inference, several deep generative approaches have been explored. Normalizing flows~\cite{rezende_variational_2015, ardizzone_analyzing_2019} provide flexible posterior approximations through invertible transformations. Score-based generative models, including diffusion
models~\cite{song_scorebased_2021} and flow
matching~\cite{lipman_flow_2023}, have recently been applied to
Bayesian inverse problems~\cite{song_solving_2022, wildberger_flow_2023}.  Conditional generative adversarial
networks have been used for posterior sampling in imaging inverse
problems~\cite{adler_deep_2018}. Combined generative adversarial network (GAN)-and-normalizing-flow priors have similarly been used to solve large-scale, physics-based Bayesian inverse problems via variational inference in the generator's latent space~\cite{patel_solution_2022a}, and related neural variational inference approaches have been applied, without an explicit forward model, to elastography~\citep{scholz_weak_2025}. Simulation-based inference methods using
neural density estimators~\cite{cranmer_frontier_2020} offer a general
framework for likelihood-free Bayesian inference.

The Variational Autoencoder (VAE) ~\citep{kingma_autoencoding_2014, rezende_stochastic_2014} combines
variational inference with deep neural networks and
employs amortized inference: a shared encoder network
maps any new observation directly to the posterior
parameters in a single forward pass, eliminating the
need to re-run optimization for each new observation. The standard VAE regularizes this posterior with the reverse Kullback--Leibler divergence, which appears naturally in the evidence lower bound (ELBO). This choice is computationally convenient but mode-seeking, unbounded ~\cite{blei_variational_2017, murphy_probabilistic_2023}, and prone to posterior collapse~\cite{he_lagging_2019}: three pathologies that are particularly damaging for nonlinear inverse problems.

The modified VAE framework by~\citet{goh_solving_2022} adapts this architecture for inverse problems in computational mechanics and physics by reinterpreting the latent space as the parameter space of interest, replacing the standard decoder with a pre-trained surrogate of the forward model. Crucially, they derive the training objective directly from an asymmetric Jensen--Shannon divergence ~\citep{sutter_multimodal_2020a} rather than the standard ELBO, thereby mitigating the mode-seeking behaviour associated with the reverse-KL formulation. The resulting encoder learns a posterior distribution over the unknown parameters conditioned on the observations, providing both parameter estimates and their associated uncertainty through the learned posterior mean and covariance.

While this framework provides an attractive approach for amortized Bayesian inference, an important question remains: does the choice of the variational objective matter for different types of inverse problems? In particular, inverse problems can have very different levels of identifiability, leading to posterior distributions that range from relatively well conditioned to strongly anisotropic and ill conditioned. We therefore investigate how the choice of divergence affects posterior approximation across these different regimes.

In this work, we investigate this question within a $\beta$-weighted VAE framework, in which the relative contribution of the divergence-based regularization term is controlled by a weighting parameter $\beta$. We consider three formulations based on the reverse Kullback–Leibler (KL) divergence, the asymmetric Jensen–Shannon (JS) divergence, and a Jensen–Shannon–Wasserstein (JSWA) regularization in which the corresponding regularization term is replaced by a squared 2-Wasserstein distance. This formulation allows us to systematically examine how the choice and weighting of the regularization influence posterior approximation across inverse problems with different conditioning.

The contributions of this work are as follows. We develop a Jensen–Shannon–Wasserstein variational autoencoder (JSWA-VAE) for amortized Bayesian inference in PDE-governed inverse problems, replacing the regularization term in the modified JS VAE formulation with a squared 2-Wasserstein distance. Beyond this methodological modification, we provide a systematic analysis of how the choice of variational objective interacts with the information geometry and conditioning of the inverse problem. Using a linear-Gaussian benchmark with an analytically known posterior, we show that the relative performance of KL, JS, and JSWA depends strongly on the distribution of information across parameter directions: KL provides the most accurate approximation for well-conditioned posteriors, whereas JS and JSWA become advantageous as the posterior becomes increasingly ill-conditioned and weakly informed directions dominate the uncertainty. We then examine whether this behaviour persists for nonlinear, physics-based inverse problems by comparing the three formulations against MCMC reference solutions using posterior mean, covariance, and Wasserstein-distance metrics. Collectively, these results suggest that the choice of divergence should be viewed as problem-dependent rather than universally optimal, with posterior conditioning providing a useful indicator for selecting an appropriate variational objective.

The remainder of the paper is structured as follows. \Cref{sec:background} provides the mathematical background on Bayesian inverse problems, variational inference, the modified VAE formulation, and Fisher information and posterior geometry. \Cref{sec:vae_framework} presents the computational formulation and implementation of the KL, JS, and proposed JS--Wasserstein formulations, including their implementable loss functions and training procedure.. \Cref{sec:theory} provides a comparative theoretical analysis of the three formulations, including posterior approximation error bounds and modal gradient behavior in the generalized Fisher basis. \Cref{sec:numerical_exp} reports numerical experiments comparing the three formulations on four test problems and validating the results against MCMC. \Cref{sec:conclusions} concludes with remarks and directions for future work.

\section{Mathematical background}
\label{sec:background}

This section develops the mathematical framework underlying the proposed approach. We first formulate the Bayesian inverse problem and introduce variational inference as an alternative to sampling-based methods. The standard and modified VAE formulations are then described, followed by a comparison of the KL, JS, and JSWA distance measures. We next examine the role of Fisher information in determining posterior conditioning and use this framework to analyze the behavior of the different regularizers.

\subsection{Forward model and Bayesian formulation}
\label{sec:forward_model}

Consider a physical system governed by a deterministic forward model
\begin{equation}\label{eq:forward_model}
	u = \mathcal{F}(\chi,q),
\end{equation}
where $q\in\mathbb{R}^n$ denotes the unknown parameters of interest,
$\chi$ represents the independent variables (e.g., spatial coordinates or
time), and $u\in\mathbb{R}^m$ denotes the observable output. The forward
operator $\mathcal{F}$ represents the governing physics, for example through
the solution operator of a partial differential equation.

The observations are modeled as
\begin{equation}\label{eq:obs_model}
	U = \mathcal{F}(\chi,Q) + \eta,
\end{equation}
where $Q$ and $\eta$ denote the random parameter and measurement-noise
variables, respectively, and $u_{\mathrm{obs}}$ is a realization of $U$.
Throughout this work, the measurement noise is assumed to be additive
Gaussian,
\begin{equation}
	\eta \sim \mathcal{N}(\mu_\eta,\Sigma_\eta).
\end{equation}
Given a prior distribution $P_{\mathrm{pr}}(q)$, Bayes' theorem gives the
posterior distribution
\begin{equation}\label{eq:bayes}
	P_{\mathrm{post}}(q\mid u_{\mathrm{obs}})
	\propto
	P_{\mathrm{lklh}}(u_{\mathrm{obs}}\mid q)\,
	P_{\mathrm{pr}}(q),
\end{equation}
where, under the Gaussian noise model,
\begin{equation}\label{eq:likelihood}
	P_{\mathrm{lklh}}(u_{\mathrm{obs}}\mid q)
	\propto
	\exp\!\left[
	-\frac{1}{2}
	\left\|
	u_{\mathrm{obs}}-\mathcal{F}(\chi,q)-\mu_\eta
	\right\|_{\Sigma_\eta^{-1}}^2
	\right],
\end{equation}
with
\[
\|v\|_{\Sigma_\eta^{-1}}^2
=
v^\top\Sigma_\eta^{-1}v.
\]
The posterior therefore provides a full probabilistic characterization of
the parameters, including both point estimates and uncertainty.

The Bayesian formulation regularizes inference for ill-posed inverse problems through the prior while retaining a probabilistic description of the
remaining uncertainty. For computationally expensive forward models,
however, evaluating the posterior remains challenging. Sampling-based
methods such as Markov chain Monte Carlo (MCMC) may require a large number
of forward-model evaluations, which becomes prohibitive when
$\mathcal{F}$ involves the repeated solution of a PDE. This motivates the
variational and amortized inference framework developed below.

\paragraph{Well-posedness.}
For the Bayesian formulation to be meaningful, the posterior must define a
well-defined probability measure and depend continuously on the observed
data. Following the framework of \citet{stuart_inverse_2010}, these
properties can be established under suitable regularity assumptions on the
negative log-likelihood potential
\begin{equation}
	\Phi(q;u)
	=
	\frac{1}{2}
	\left\|
	u-\mathcal{F}(\chi,q)
	\right\|_{\Sigma_\eta^{-1}}^2.
\end{equation}
For the bounded neural-network surrogate considered here, the required
local boundedness and continuity conditions follow from the regularity of
the forward map. Under these assumptions, the posterior is well defined and
depends continuously on the observations in the Hellinger metric
\citep[Theorem~4.2]{stuart_inverse_2010}:
\begin{align}
	d_{\mathrm{Hell}}\!\left(
	P_{\mathrm{post}}(\cdot\mid u_1),
	P_{\mathrm{post}}(\cdot\mid u_2)
	\right)
	\leq
	C_s\|u_1-u_2\|,
	\label{eq:hellinger}
\end{align}
where $C_s>0$ depends on the prior and the forward operator.
\subsection{Variational inference and variational autoencoders}
\label{sec:vi_vae}

Variational inference (VI) reformulates Bayesian inference as a deterministic
optimization problem. Instead of drawing samples from the posterior
$P_{\mathrm{post}}(q\mid u)$, VI seeks an approximation from a tractable
family of distributions,
\begin{equation}
	Q_\phi(q\mid u), \qquad \phi\in\Phi,
\end{equation}
by minimizing a divergence between the approximate and target posteriors:
\begin{equation}
	Q_\phi^*
	=
	\arg\min_{\phi\in\Phi}
	D\!\left(
	Q_\phi(q\mid u)
	\,\|\, 
	P_{\mathrm{post}}(q\mid u)
	\right).
	\label{eq:vi_general}
\end{equation}
The choice and direction of the divergence are important, since different
divergences impose different notions of similarity between probability
distributions and can therefore lead to different posterior approximations.

A variational autoencoder (VAE)~\cite{Kingma_2013_AutoEncodingVB,
	rezende_stochastic_2014} combines variational inference with neural networks
to construct an amortized posterior approximation. The encoder
$\Psi_e$ maps an observation $u$ to the parameters of an approximate
posterior distribution,
\begin{equation}\label{eq:encoder}
	\Psi_e(u;\omega_e)
	\longrightarrow
	\bigl(\mu_{\mathrm{post}},\Sigma_{\mathrm{post}}\bigr),
\end{equation}
while the decoder $\Psi_d$ maps a latent variable $q$ back to the observation
space,
\begin{equation}\label{eq:decoder}
	\Psi_d(q;\omega_d)
	\longrightarrow
	\hat{u}.
\end{equation}
In the standard VAE, the encoder and decoder are trained by minimizing the
negative evidence lower bound (ELBO),
\begin{equation}\label{eq:vae_loss_standard}
	\mathcal{L}_{\mathrm{VAE}}
	=
	\mathbb{E}_{Q_\phi(q\mid u)}
	\left[-\log P(u\mid q)\right]
	+
	D_{\mathrm{KL}}
	\bigl(
	Q_\phi(q\mid u)\,\|\,P_{\mathrm{pr}}(q)
	\bigr).
\end{equation}
The first term measures the reconstruction error, while the second
regularizes the approximate posterior toward the prior.

To enable gradient-based optimization through the stochastic sampling step,
the reparameterization trick~\cite{Kingma_2013_AutoEncodingVB} expresses a
sample from the approximate posterior as
\begin{equation}\label{eq:reparam}
	q_{\mathrm{draw}}
	=
	\mu_{\mathrm{post}} + L\epsilon,
	\qquad
	\epsilon\sim\mathcal{N}(0,I),
\end{equation}
where $L$ denotes a factor of the posterior covariance satisfying
$\Sigma_{\mathrm{post}}=LL^\top$. This formulation transfers the
stochasticity to the auxiliary variable $\epsilon$ and allows gradients to
propagate through the parameters of the approximate posterior.

The standard VAE is designed as a generative model in which the latent
variable provides a representation of the observed data. In Bayesian inverse
problems, however, the latent variable can instead be identified directly
with the unknown physical parameters $q$. The encoder then provides an
amortized approximation to the posterior $P_{\mathrm{post}}(q\mid u)$,
while the decoder represents the forward mapping from parameters to
observables. This interpretation provides the basis for the modified VAE
framework for Bayesian inverse problems introduced in the next subsection.

\subsection{Modified VAE for Bayesian inverse problems}
\label{sec:modified_vae}
\begin{figure}[ht]
	\centering
	\includegraphics[width=0.65\textwidth]{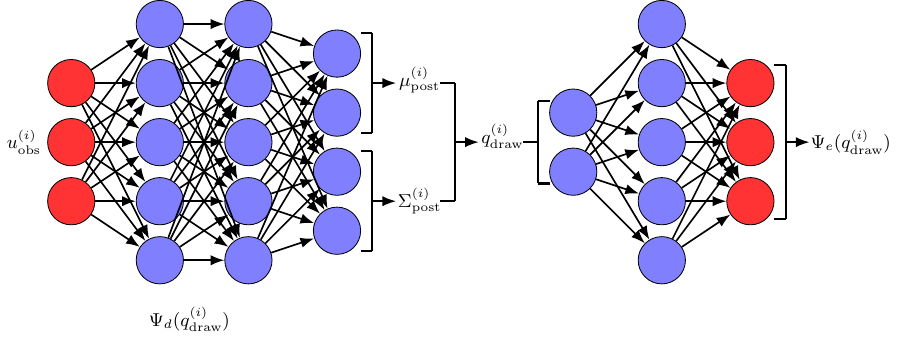}	
    \caption{Architecture of the modified VAE for Bayesian inverse problems. The encoder $\Psi_e$ maps the observed data $u_{\mathrm{obs}}$ to the posterior parameters $(\mu_{\mathrm{post}}, \Sigma_{\mathrm{post}})$. Latent samples $q_{\mathrm{draw}}$ are generated via the reparameterization trick~\Cref{eq:reparam} and passed through the pre-trained, frozen decoder $\Psi_d$, which acts as a surrogate for the forward model $\mathcal{F}$. The loss function comprises a reconstruction term (comparing $\Psi_d(q_{\mathrm{draw}})$ with $u_{\mathrm{obs}}$) and a regularization term (measuring the divergence between the approximate posterior and the prior).}
    \label{fig:vae_architecture}
\end{figure}
In the context of parameter estimation and inverse problems, the standard VAE framework admits a physically meaningful reinterpretation, as proposed by Goh et al.~\cite{goh_solving_2022}. The key modifications are:

\paragraph{Latent space as parameter space.} The latent variable $q$ is identified with the parameter of interest (POI). The encoder thus learns a mapping from the observation space to the parameter space, directly providing the posterior distribution over the physical parameters:
\begin{equation}\label{eq:encoder_inverse}
    \Psi_e(u_{\mathrm{obs}};\, \omega_e) \longrightarrow Q_\phi(q \mid u_{\mathrm{obs}}) = \mathcal{N}(\mu_{\mathrm{post}}, \Sigma_{\mathrm{post}}).
\end{equation}
Here, $\mu_{\mathrm{post}}$ serves as the point estimate of the parameters (analogous to the maximum a posteriori estimate), and $\Sigma_{\mathrm{post}}$ quantifies the associated uncertainty.

\paragraph{Decoder as forward model surrogate.} The decoder $\Psi_d$ approximates the forward operator $\mathcal{F}$, serving as a parameter-to-observation map:
\begin{equation}\label{eq:decoder_inverse}
    \Psi_d(q;\, \omega_d) \approx \mathcal{F}(\chi, q).
\end{equation}
This embedding of the governing physics into the decoder architecture is the defining feature of the modified VAE. By replacing the standard learned decoder with a physics-informed surrogate, the framework ensures that the latent space has a clear physical interpretation and that the reconstruction loss directly measures the discrepancy between predicted and observed physical quantities.

\paragraph{Mean-field vs.\ full covariance.} In many standard VAE implementations, the approximate posterior is restricted to a diagonal (mean-field) covariance: $$\Sigma_{\mathrm{post}} = \mathrm{diag}(\sigma_1^2, \ldots, \sigma_n^2)$$ This reduces the number of learnable parameters from $\mathcal{O}(n^2)$ to $\mathcal{O}(n)$ but assumes independence among the latent dimensions. For inverse problems where the model parameters are often correlated (e.g., stiffness and loading in a structural problem), this assumption can lead to poor posterior approximations. In this work, we employ a \emph{full covariance} parameterization, outputting the Cholesky factor $L$ directly from the encoder. This enables the approximate posterior to capture parameter correlations, which is important for accurately representing posterior uncertainty.

\paragraph{Surrogate decoder and prior specification.} Unlike a conventional VAE, the decoder in the present framework represents a surrogate of the physical forward operator rather than a jointly trained generative model. The surrogate is first trained offline using parameter--observation pairs generated from the forward model over the prescribed parameter domain and is subsequently frozen during encoder training. We assume a Gaussian prior for the parameters,
\begin{equation}
	P_{\mathrm{pr}}(q)
	=
	\mathcal{N}(\mu_{\mathrm{pr}},\Sigma_{\mathrm{pr}}),
\end{equation}
which provides the reference distribution for the variational regularization. The same prior is used consistently in the KL-, JS-, and JSWA-based formulations. This separation allows the encoder to focus on approximating the posterior while avoiding repeated evaluations of the computationally expensive forward model during inference.

The complete architecture of the modified VAE is illustrated in \Cref{fig:vae_architecture}. During inference, a new observation $u_{\mathrm{obs}}$ is passed through the trained encoder, which outputs $\mu_{\mathrm{post}}$ and $\Sigma_{\mathrm{post}}$ in a single forward pass. This \emph{amortized inference} capability is a major practical advantage: once trained, the encoder provides instantaneous posterior estimates without any iterative sampling, making it orders of magnitude faster than MCMC for repeated inference tasks.
\subsubsection{Distance measures and their properties}
\label{sec:distances}

The choice of divergence is important because different divergences induce different approximation behaviour and optimization properties. We therefore compare the three distance measures underlying the KL-, JS-, and JSWA-based formulations.

\paragraph{Limitations of the reverse KL divergence.}
The standard VAE minimises $D_{\mathrm{KL}}(Q_\phi \| P_{\mathrm{post}})$, which
suffers from three structural pathologies that make it poorly suited to
nonlinear inverse problems:
\begin{itemize}
    \item \emph{Support sensitivity.} When $\operatorname{supp}(Q_\phi)
    \not\subseteq \operatorname{supp}(P_{\mathrm{post}})$, the divergence is
    $+\infty$ and returns no gradient ~\citep{arjovsky_wasserstein_2017}. For general distributions, the reverse KL can become infinite when the approximate posterior assigns mass outside the support of the target distribution. This issue is less relevant for the nondegenerate Gaussian distributions considered in the present study, which have common support.
    \item \emph{Mode-seeking behaviour.} The reverse KL penalises mass that
    $Q_\phi$ places outside the support of $P_{\mathrm{post}}$ heavily, but
    is indifferent to regions of high posterior density that $Q_\phi$ fails
    to cover~\citep{murphy_probabilistic_2023}. The reverse KL is mode-seeking and may therefore favour approximations concentrated around a dominant mode when the target posterior is multimodal.
    \item \emph{Posterior collapse.} 
    Strong KL regularization can also drive the approximate posterior toward the prior, reducing the information carried by the latent variables. This phenomenon is commonly referred to as posterior collapse in VAE models. In the present inverse-problem setting, the decoder is a frozen, pre-trained physics surrogate rather than a jointly optimized generative decoder. Thus, the role of KL regularization differs from that in a conventional VAE with a jointly trained decoder. Nevertheless, the strength of the KL regularization remains controlled through the $\beta$-VAE formulation, and its effect on the posterior approximation is considered in the subsequent analysis~\citep{he_lagging_2019}.
\end{itemize}
Together, these pathologies motivate replacing $D_{\mathrm{KL}}(Q_\phi \|
P_{\mathrm{post}})$ with a divergence that remains finite under support
mismatch and does not privilege a single mode of the posterior -- the properties examined next.
\paragraph{The Jensen--Shannon divergence.}
The $\alpha$-JSD~\citep{nielsen_family_2011} is defined via the mixture
$M_\alpha = \alpha P + (1-\alpha)Q$ as
\begin{align}
    D_{\mathrm{JS},\alpha}(P \| Q)
    = \alpha\, D_{\mathrm{KL}}(Q \| M_\alpha)
    + (1-\alpha)\, D_{\mathrm{KL}}(P \| M_\alpha).
    \label{eq:jsd}
\end{align}
Three properties of $D_{\mathrm{JS},\alpha}$ are relevant here:
\begin{itemize}
    \item\emph{Boundedness.} The $\alpha$-Jensen--Shannon divergence satisfies
    \[
    0 \leq D_{\mathrm{JS},\alpha}(P\|Q)
    \leq H(\alpha),
    \]
    where
    \[
    H(\alpha)
    =
    -\alpha\log\alpha-(1-\alpha)\log(1-\alpha)
    \]
    is the binary entropy. Thus, unlike the reverse KL divergence, the
    $\alpha$-JSD remains finite even when the distributions have disjoint
    support. This boundedness prevents the divergence itself from becoming
    arbitrarily large during optimization, although boundedness alone does not
    guarantee non-vanishing or well-behaved gradients.
    \item \emph{Mode coverage.} The forward-KL component $D_{\mathrm{KL}}(P \| M_\alpha)$ encourages the mixture $M_\alpha$, and hence $Q$ to account for regions of posterior mass that may be neglected by a purely reverse-KL objective.
\end{itemize}
The residual limitation of $D_{\mathrm{JS},\alpha}$ is that it saturates at
$H(\alpha)$ under large support separation, providing no information on
geometric distance between the distributions.

\paragraph{The $2$-Wasserstein distance.}
The $2$-Wasserstein distance
\begin{align}
	W_2(P, Q)
	= \left(\inf_{\gamma \in \Pi(P,Q)}
	\int \|q - q'\|^2\, d\gamma(q,q')\right)^{1/2}
	\label{eq:w2_def}
\end{align}
provides a complementary notion of discrepancy that retains sensitivity to
the geometric displacement between probability distributions. For
distributions with finite second moments, $W_2$ is finite. In particular, for
$P = \mathcal{N}(0,\sigma^2 I)$ and
$Q = \mathcal{N}(\mu,\sigma^2 I)$,
\[
W_2(P,Q)=\|\mu\|,
\]
and therefore continues to increase as the distributions become increasingly
separated, whereas the Jensen--Shannon divergence remains bounded. For
Gaussian distributions, $W_2$ admits the Bures--Wasserstein closed form
\citep{dowson_frechet_1982,bhatia_bures_2019a}
\begin{align}
	\Aboxed{
		W_2^2\!\bigl(\mathcal{N}(\mu_P,\Sigma_P),
		\mathcal{N}(\mu_Q,\Sigma_Q)\bigr)
		=
		\|\mu_P-\mu_Q\|_2^2
		+
		\operatorname{tr}\!\left(
		\Sigma_P+\Sigma_Q
		-2\bigl(
		\Sigma_P^{1/2}\Sigma_Q\Sigma_P^{1/2}
		\bigr)^{1/2}
		\right),
	}
	\label{eq:W2_gaussian}
\end{align}
which can be evaluated directly for the Gaussian posterior and prior used
in the present framework. This makes the Wasserstein regularizer tractable
for the moderate-dimensional problems considered here.

Moreover, Wasserstein distance provides direct control over differences in
expectations of Lipschitz functions. For any Lipschitz function $f$ with
constant $L_f$~\citep[Remark~6.5]{villani_optimal_2008},
\begin{align}
	\bigl|\mathbb{E}_P[f] - \mathbb{E}_Q[f]\bigr|
	\leq L_f\, W_1(P,Q)
	\leq L_f\, W_2(P,Q).
	\label{eq:wass_expectation}
\end{align}
For the posterior quantities of interest here, this implies that the
Wasserstein distance also controls discrepancies in the posterior mean and
variance,
\begin{align}
    \left\|\mathbb{E}_{P_{\mathrm{post}}}[q]
    - \mathbb{E}_{Q_\phi}[q]\right\|_2
    &\leq W_2(P_{\mathrm{post}},Q_\phi),
    \label{eq:mean_control} \\[3pt]
    \bigl|\mathrm{Var}_{P_{\mathrm{post}}}[q_i]
    - \mathrm{Var}_{Q_\phi}[q_i]\bigr|
    &\leq C_{\mathrm{var}}\,
    W_2(P_{\mathrm{post}},Q_\phi),
    \label{eq:var_control}
\end{align}
where $C_{\mathrm{var}}$ is a finite constant depending on suitable moment
bounds for the two distributions. Thus, reducing the Wasserstein distance
provides a distribution-level criterion that also controls errors in the
posterior mean and variance used for uncertainty quantification.

Table~\ref{tab:divergence_comparison} summarizes the key properties of the
three distance measures. No single divergence possesses all of the desired
properties; the JSWA objective therefore combines complementary features of
the forward-KL and Wasserstein terms.

\begin{table}[htbp]
	\centering
	\caption{Comparison of distance measures used in the variational
		inference framework. \emph{Support mismatch}: remains finite when the
		distributions have non-overlapping support. \emph{Geometric sensitivity}:
		retains information about the spatial displacement between distributions.
		\emph{Mode coverage}: encourages the approximation to account for multiple
		regions of posterior mass. \emph{Closed form}: analytic expression for
		Gaussian arguments.}
	\label{tab:divergence_comparison}
	\setlength{\tabcolsep}{8pt}
	\renewcommand{\arraystretch}{1.25}
	\begin{tabular}{@{} l c c c @{}}
		\toprule
		\textbf{Property}
		& $D_{\mathrm{KL}}(Q_\phi \| P)$
		& $D_{\mathrm{JS},\alpha}$
		& $W_2$ \\
		\midrule
		Bounded above
		& \xmark
		& \cmark
		& \xmark \\
		Finite under support mismatch
		& \xmark
		& \cmark
		& \cmark\textsuperscript{$\dagger$} \\
		Geometric sensitivity
		& \xmark
		& Limited
		& \cmark \\
		Mode coverage
		& \xmark
		& \cmark
		& --- \\
		Closed form for Gaussians
		& \cmark
		& ---
		& \cmark \\
		\bottomrule
		\multicolumn{4}{@{}l}{\footnotesize
			$\dagger$~For distributions with finite second moments; $W_2$ is finite on
			$\mathcal{P}_2$.}
	\end{tabular}
\end{table}

\paragraph{Abstract loss functions.}
The three formulations considered in this work share a common reconstruction
term and differ in the regularization or posterior-supervision terms applied
to the encoder distribution $Q_\phi(q\mid u)$. We state the three objectives
here in their abstract form; their closed-form Gaussian implementations are
derived in \cref{app:derivations}, while their posterior approximation and modal gradient behavior are analyzed in \Cref{subsec:modal_gradient}. Let $\beta_1,\beta_2>0$ denote the
corresponding regularization weights.

The baseline VAE formulation uses the reverse KL divergence between the
encoder posterior and the parameter-space prior:
\begin{align}
	\mathcal{L}_{\mathrm{KL}}
	&=
	\underbrace{
		\mathbb{E}_{Q_\phi}
		\!\left[-\ln P(u\mid q)\right]
	}_{\text{reconstruction}}
	+
	\underbrace{
		\beta_1
		D_{\mathrm{KL}}
		\!\left(
		Q_\phi(q\mid u)\|P_{\mathrm{pr}}(q)
		\right)
	}_{\text{reverse-KL regularisation}}.
	\label{eq:loss_kl_abstract}
\end{align}

Following Goh et al.~\cite{goh_solving_2022}, the JS formulation augments
the reconstruction objective with a forward-KL term that compares the
encoder posterior with the parameter-conditioned target distribution. When
paired parameter--observation data are available, this term can be evaluated
using the corresponding known parameter values:
\begin{align}
	\mathcal{L}_{\mathrm{JS}}
	&=
	\underbrace{
		\beta_1
		D_{\mathrm{KL}}
		\!\left(
		P(q\mid u)\|Q_\phi(q\mid u)
		\right)
	}_{\text{forward-KL term}}
	+
	\underbrace{
		\mathbb{E}_{Q_\phi}
		\!\left[-\ln P(u\mid q)\right]
	}_{\text{reconstruction}}
	+
	\underbrace{
		\beta_2
		D_{\mathrm{KL}}
		\!\left(
		Q_\phi(q\mid u)\|P_{\mathrm{pr}}(q)
		\right)
	}_{\text{reverse-KL regularisation}}.
	\label{eq:loss_js_abstract}
\end{align}

The proposed JSWA formulation retains the forward-KL and reconstruction
terms of $\mathcal{L}_{\mathrm{JS}}$, but replaces the reverse-KL
regularization with the squared $2$-Wasserstein distance:
\begin{align}
	\mathcal{L}_{\mathrm{JSWA}}
	&=
	\underbrace{
		\beta_1
		D_{\mathrm{KL}}
		\!\left(
		P(q\mid u)\|Q_\phi(q\mid u)
		\right)
	}_{\text{forward-KL term}}
	+
	\underbrace{
		\mathbb{E}_{Q_\phi}
		\!\left[-\ln P(u\mid q)\right]
	}_{\text{reconstruction}}
	+
	\underbrace{
		\beta_2
		W_2^2
		\!\left(
		Q_\phi(q\mid u),P_{\mathrm{pr}}(q)
		\right)
	}_{\text{Wasserstein regularisation}}.
	\label{eq:loss_jswa_abstract}
\end{align}

The three terms play distinct roles. The reconstruction term enforces
consistency with the observations through the likelihood, while the
forward-KL term provides parameter-space supervision when paired
training data are available. The final regularization term controls the
departure of the encoder posterior from the prior. The KL and JS
formulations therefore retain the reverse-KL regularizer, whereas JSWA
replaces it with the squared Wasserstein distance.
\subsection{Fisher information and posterior geometry}
\label{sec:fisher}
The posterior distribution reflects the combined influence of the prior and
the information provided by the observations. The extent to which the data
constrain different parameter directions depends on the sensitivity of the
forward operator to those directions. The Fisher information matrix provides
a local measure of this data-driven sensitivity and therefore offers a natural
way to characterize the geometry and conditioning of the posterior. For further discussion of Fisher information, parameter sensitivity, and identifiability, see ~\citep{walter_identification_1997} and the references therein.

\paragraph{Sensitivity.}
For a Gaussian likelihood
$P(u_{\mathrm{obs}} \mid \mathrm{q})
= \mathcal{N}(\mathcal{F}(\mathrm{q}), \Sigma_\eta)$,
the local Fisher information matrix associated with the data is
\begin{align}
	\mathcal{I}(\mathrm{q})
	= J^\top \Sigma_\eta^{-1} J,
	\qquad
	J =
	\frac{\partial \mathcal{F}}
	{\partial \mathrm{q}}
	\bigg|_{\hat{\mathrm{q}}}
	\in \mathbb{R}^{m\times d},
	\label{eq:fisher_def}
\end{align}
where $J$ is the Jacobian of the forward operator evaluated at the
reference parameter $\hat{\mathrm{q}}$, taken here to be the posterior
mode. The matrix $\mathcal{I}$ therefore quantifies the local information
provided by the observations about different parameter directions.

For isotropic noise
$\Sigma_\eta = \sigma_\eta^2 I_m$ and the linear model
$\mathcal{F}(\mathrm{q}) = A\mathrm{q}$, this reduces to
\begin{equation}
	\mathcal{I}
	= \frac{1}{\sigma_\eta^2}A^\top A.
\end{equation}
For a nonlinear forward model, $\mathcal{I}(\hat{\mathrm{q}})$ provides
the local data-information matrix obtained by linearizing the forward
operator about the posterior mode. It therefore forms the
data-information component of the local Laplace approximation to the
posterior.
\paragraph{Generalized Fisher basis.}
For a Gaussian prior
$\mathcal{N}(\mu_{\mathrm{pr}},\Sigma_{\mathrm{pr}})$, let
$L_{\mathrm{pr}}$ denote the lower-triangular Cholesky factor of the prior
precision,
\begin{align}
	\Sigma_{\mathrm{pr}}^{-1}
	=
	L_{\mathrm{pr}}L_{\mathrm{pr}}^\top. 
\end{align}
Introducing the whitened coordinates
\begin{equation}
	\tilde{\mathrm q}
	=
	L_{\mathrm{pr}}^\top
	(\mathrm q-\mu_{\mathrm{pr}}),
	\qquad
	\tilde J=J L_{\mathrm{pr}}^{-\top},\label{eq:whitening}
\end{equation}
transforms the prior to $\mathcal{N}(0,I_d)$ and gives the generalized
Fisher information matrix
\begin{equation}
	\tilde{\mathcal I}
	=
	\tilde J^\top\Sigma_\eta^{-1}\tilde J
	=
	L_{\mathrm{pr}}^{-1}
	\mathcal I
	L_{\mathrm{pr}}^{-\top}.
	\label{eq:fisher_whitened}
\end{equation}
Its eigendecomposition,
\begin{align}
	\tilde{\mathcal I}
	=
	\tilde V
	\operatorname{diag}
	\left(
	\tilde\lambda_1^{\mathrm{Fisher}},\ldots,
	\tilde\lambda_d^{\mathrm{Fisher}}
	\right)
	\tilde V^\top,
	\qquad
	\tilde\lambda_1^{\mathrm{Fisher}}\geq\cdots\geq
	\tilde\lambda_d^{\mathrm{Fisher}}, \label{eq:fisher_eigen_decomp}
\end{align}
defines the generalized Fisher basis. Large eigenvalues correspond to
strongly data-informed parameter directions, whereas small eigenvalues
identify weakly informed directions.

Its eigen-decomposition is
\begin{align}
	\tilde{\mathcal{I}}
	= \tilde{V}\,\tilde{\Lambda}^{\mathrm{Fisher}}\,\tilde{V}^\top,
	\qquad
	\tilde{\Lambda}^{\mathrm{Fisher}}
	= \mathrm{diag}\!\left(
	\tilde{\lambda}_1^{\mathrm{Fisher}},
	\ldots,
	\tilde{\lambda}_d^{\mathrm{Fisher}}
	\right),
	\label{eq:fisher_eig}
\end{align}
with
$\tilde{\lambda}_1^{\mathrm{Fisher}}\geq\cdots\geq
\tilde{\lambda}_d^{\mathrm{Fisher}}\geq0$.
The columns $\{\tilde{\mathbf v}_k\}$ of $\tilde V$ define the
\emph{generalized Fisher basis}, with the corresponding eigenvalues
quantifying the information supplied by the observations along these
directions. For isotropic observation noise,
$\Sigma_\eta=\sigma_\eta^2 I_m$, the eigenvalues satisfy
\begin{align}
	\tilde{\lambda}_k^{\mathrm{Fisher}}
	=
	\frac{\tilde{\sigma}_{J,k}^2}{\sigma_\eta^2},
	\label{eq:snr}
\end{align}
where $\tilde{\sigma}_{J,k}$ is the $k$-th singular value of
$\tilde J$. Thus, $\tilde{\lambda}_k^{\mathrm{Fisher}}$ can be interpreted
as a squared effective signal-to-noise ratio for parameter direction
$\tilde{\mathbf v}_k$: large values correspond to strongly informed
directions, whereas small values indicate weakly informed directions.

For the isotropic prior $\Sigma_{\mathrm{pr}}=I_d$, the whitening is
trivial, so that $\tilde J=J$, $\tilde{\mathbf v}_k=\mathbf v_k$, and
$\tilde{\lambda}_k^{\mathrm{Fisher}}=\lambda_k^{\mathrm{Fisher}}$,
recovering the standard Fisher basis.
\paragraph{Posterior covariance in the generalized Fisher basis.}
For the linear--Gaussian inverse problem
\begin{align}
    u &= Aq+\eta,
    \qquad
    q\sim\mathcal{N}(0,\Sigma_{\mathrm{pr}}),
    \qquad
    \eta\sim\mathcal{N}(0,\Sigma_\eta),
    \label{eq:linear_gaussian}
\end{align}
the posterior precision is
\begin{align}
	\Lambda_{\mathrm{post}}
	=
	\mathcal{I}(\mathrm{q})
	+
	\Sigma_{\mathrm{pr}}^{-1},
	\label{eq:post_precision}
\end{align}
while for a nonlinear forward model the same expression gives the local
Laplace approximation about the reference parameter. In the whitened
coordinates introduced above, this becomes
\begin{align}
	\tilde{\Lambda}_{\mathrm{post}}
	=
	\tilde{\mathcal{I}}+I_d
	=
	\tilde{V}\,
	\mathrm{diag}\!\left(
	1+\tilde{\lambda}_k^{\mathrm{Fisher}}
	\right)\tilde{V}^{\top}.
	\label{eq:post_precision_whitened}
\end{align}
Consequently, the posterior covariance in the whitened generalized Fisher
basis is
\begin{align}
	\Aboxed{
		\widetilde{\Sigma}_{\mathrm{post}}^*
		=
		\mathrm{diag}\!\left(
		\frac{1}
		{1+\tilde{\lambda}_k^{\mathrm{Fisher}}}
		\right)_{k=1}^{d}.
	}
	\label{eq:post_cov_fisher}
\end{align}
Thus, for the linear-Gaussian problem, the posterior variance along
generalized Fisher mode $k$ is
\begin{align}
	\tilde{\sigma}_k^{*2}
	=
	\frac{1}
	{1+\tilde{\lambda}_k^{\mathrm{Fisher}}},
	\label{eq:posterior_variance}
\end{align}
providing a direct reference for assessing the learned mode-wise
variances. 

The full derivation is given in \cref{app:fisher_derivation}. The result
shows that, in the whitened Fisher basis, strongly informed directions
with large $\tilde{\lambda}_k^{\mathrm{Fisher}}$ have substantially reduced
posterior variance, whereas weakly informed directions with small
$\tilde{\lambda}_k^{\mathrm{Fisher}}$ remain closer to the prior variance.
\paragraph{Shrinkage and the MAP estimate.}
In the whitened generalized Fisher basis, the prior is
$\mathcal{N}(0,I_d)$, so each parameter direction is regularized toward
zero. For the linear-Gaussian problem, the posterior mean in mode $k$
exhibits mode-dependent shrinkage, with the relative weight assigned to
the data determined by the corresponding generalized Fisher eigenvalue
$\tilde{\lambda}_k^{\mathrm{Fisher}}$. Since the posterior is Gaussian,
its mean coincides with its MAP estimate. In particular,
\begin{align}
	\tilde{\mu}_{\mathrm{post},k}^*
	=
	\frac{\tilde{\lambda}_k^{\mathrm{Fisher}}}
	{1+\tilde{\lambda}_k^{\mathrm{Fisher}}}
	\tilde q_k^{\mathrm{data}},
	\label{eq:shrinkage}
\end{align}
where $\tilde q_k^{\mathrm{data}}$ denotes the data-informed estimate in
the $k$th generalized Fisher direction. Thus, the factor
$\tilde{\lambda}_k^{\mathrm{Fisher}}/
(1+\tilde{\lambda}_k^{\mathrm{Fisher}})$ determines the degree to which
the posterior mean follows the data rather than the prior.

For $\tilde{\lambda}_k^{\mathrm{Fisher}}\gg1$, the data dominate and
$\tilde{\mu}_{\mathrm{post},k}^*$ approaches the data-informed estimate.
When $\tilde{\lambda}_k^{\mathrm{Fisher}}=1$, the data contribution has
weight $1/2$. For $\tilde{\lambda}_k^{\mathrm{Fisher}}\ll1$, the posterior
mean is strongly shrunk toward the prior mean, which is zero in the
whitened coordinates. This behaviour is not a failure of inference:
when the observations provide little information about a parameter
direction, the Bayesian estimator appropriately retains the prior
information rather than producing an artificially precise estimate.

The corresponding posterior standard deviation,
\begin{align}
	\tilde{\sigma}_k^*
	=
	\frac{1}
	{\sqrt{1+\tilde{\lambda}_k^{\mathrm{Fisher}}}},
	\label{eq:posterior_std_shrinkage}
\end{align}
quantifies the uncertainty remaining in that direction. Hence, strongly
data-informed directions exhibit both weak shrinkage of the posterior
mean and strong posterior contraction, whereas weakly informed directions
remain close to the prior in both their mean and variance.
\paragraph{Identifiability and ill-conditioning.}
The generalized Fisher eigenvalue spectrum characterizes the local
identifiability structure of the inverse problem. A parameter direction is
strongly informed when
$\tilde{\lambda}_k^{\mathrm{Fisher}}\gg1$ and weakly informed when
$\tilde{\lambda}_k^{\mathrm{Fisher}}\ll1$. A large spread in these
eigenvalues produces an anisotropic posterior covariance and, when
sufficiently large, an ill-conditioned posterior. From
\eqref{eq:post_cov_fisher}, the posterior condition number is
\begin{align}
	\kappa(\Sigma_{\mathrm{post}}^*)
	=
	\frac{1+\tilde{\lambda}_{\max}^{\mathrm{Fisher}}}
	{1+\tilde{\lambda}_{\min}^{\mathrm{Fisher}}},
	\label{eq:kappa}
\end{align}
so that large $\kappa$ corresponds to a highly elongated posterior
ellipsoid in the whitened Fisher coordinates, with narrow directions
associated with strong data information and broad directions associated
with weak information.

\section{Computational formulation and implementation}
\label{sec:vae_framework}

This section specifies the computational realization of the three variational formulations introduced in \cref{sec:distances}. We first give the implementable loss functions used for training, followed
by the encoder parameterization and training procedure. Detailed derivations of the
implementable loss functions are provided in \cref{app:derivations}.

\subsection{Implementable loss functions} \label{subsec:loss}
The abstract objectives introduced in \cref{sec:distances} are
specialized here to the Gaussian encoder and the paired training data
used in this work. The reconstruction term is estimated by Monte Carlo
sampling from the encoder posterior, while the KL and Wasserstein terms
admit closed-form expressions for Gaussian distributions.

For each training pair $(q^{(i)},u_{\mathrm{obs}}^{(i)})$, the encoder
produces the Gaussian posterior
\begin{align}
	Q_\phi^{(i)}
	&=
	\mathcal N\!\left(
	\mu_{\mathrm{post}}^{(i)},
	\Sigma_{\mathrm{post}}^{(i)}
	\right).
	\label{eq:encoder_posterior}
\end{align}
A reparameterized draw from this posterior is given by
\begin{align}
	q_{\mathrm{draw}}^{(i)}
	&=
	\mu_{\mathrm{post}}^{(i)}
	+
	L^{(i)}\epsilon^{(i)},
	\qquad
	\epsilon^{(i)}\sim\mathcal N(0,I),
	\qquad
	L^{(i)}(L^{(i)})^\top
	=
	\Sigma_{\mathrm{post}}^{(i)}.
	\label{eq:reparameterization}
\end{align}

All three formulations share the same reconstruction term, evaluated through the frozen surrogate decoder: 
\begin{align}
	\ell_{\mathrm{rec}}^{(i)}
	&=
	\frac{1}{2}
	\left\|
	u_{\mathrm{obs}}^{(i)}
	-
	\Psi_d(q_{\mathrm{draw}}^{(i)})
	-
	\mu_\eta
	\right\|_{\Sigma_\eta^{-1}}^2 .
	\label{eq:loss_rec}
\end{align}
The formulations differ in how the encoder posterior is regularized
toward the prior and, for Formulations~II and~III, how it is supervised
using the paired parameter values. The reverse-KL and Wasserstein
regularizers are evaluated directly from the Gaussian posterior
parameters. The expectation in the VAE reconstruction term is therefore approximated
by Monte Carlo sampling of $q_{\mathrm{draw}}^{(i)}$; a single
reparameterized draw is used here for each training sample.

The formulations differ in how the encoder posterior is regularized
toward the prior and, for Formulations~II and~III, how it is supervised
using the paired parameter values. The reverse-KL and Wasserstein
regularizers are evaluated directly from the Gaussian posterior
parameters. For Gaussian distributions, the reverse-KL regularizer is
\begin{align}
	\mathcal{L}_{\mathrm{RKL}}^{(i)}
	&=
	\frac{1}{2}
	\Bigg[
	\operatorname{tr}\!\left(
	\Sigma_{\mathrm{pr}}^{-1}
	\Sigma_{\mathrm{post}}^{(i)}
	\right)
	+
	\left\|
	\mu_{\mathrm{post}}^{(i)}
	-\mu_{\mathrm{pr}}
	\right\|_{\Sigma_{\mathrm{pr}}^{-1}}^2
	-n
	+\ln
	\frac{|\Sigma_{\mathrm{pr}}|}
	{|\Sigma_{\mathrm{post}}^{(i)}|}
	\Bigg].
	\label{eq:loss_kl_reg}
\end{align}

For paired data $(q^{(i)},u_{\mathrm{obs}}^{(i)})$, the forward-KL
supervision term is implemented using the known parameter value
$q^{(i)}$. Up to terms independent of the encoder parameters, it is
equivalent to the negative log-density of $q^{(i)}$ under the encoder
posterior:
\begin{align}
	\mathcal{L}_{\mathrm{FKL}}^{(i)}
	&=
	\frac{1}{2}
	\left[
	\ln |\Sigma_{\mathrm{post}}^{(i)}|
	+
	\left\|
	\mu_{\mathrm{post}}^{(i)}-q^{(i)}
	\right\|_{(\Sigma_{\mathrm{post}}^{(i)})^{-1}}^2
	\right],
	\label{eq:loss_fwdkl}
\end{align}

For Formulation~III, the reverse-KL prior regularizer is replaced by the
squared $2$-Wasserstein distance. For Gaussian distributions,
\begin{align}
	\mathcal{L}_{W_2}^{(i)}
	= \|\mu_{\mathrm{post}}^{(i)} - \mu_{\mathrm{pr}}\|_2^2
	+ \mathrm{tr}(\Sigma_{\mathrm{post}}^{(i)}) + \mathrm{tr}(\Sigma_{\mathrm{pr}})
	- 2\,\mathrm{tr}\!\left(\bigl(
	(\Sigma_{\mathrm{post}}^{(i)})^{1/2}\Sigma_{\mathrm{pr}}(\Sigma_{\mathrm{post}}^{(i)})^{1/2}
	\bigr)^{1/2}\right)
	= W_2^2\bigl(Q_\phi^{(i)}, P_{\mathrm{pr}}\bigr).
	\label{eq:loss_w2reg}
\end{align}

The implementable objectives are obtained by combining \cref{eq:loss_rec,eq:loss_kl_reg,eq:loss_fwdkl,eq:loss_w2reg} terms and
averaging over the $N$ training samples: 
\begin{align}
	\mathcal{L}_{\mathrm{KL}}
	&\approx
	\frac{1}{N}\sum_{i=1}^{N}
	\left[
	\mathcal{L}_{\mathrm{rec}}^{(i)}
	+
	\beta_1
	\mathcal{L}_{\mathrm{RKL}}^{(i)}
	\right],
	\label{eq:loss_kl}
	\\
	\mathcal{L}_{\mathrm{JS}}
	&\approx
	\frac{1}{N}\sum_{i=1}^{N}
	\left[
	\beta_1
	\mathcal{L}_{\mathrm{FKL}}^{(i)}
	+
	\mathcal{L}_{\mathrm{rec}}^{(i)}
	+
	\beta_2
	\mathcal{L}_{\mathrm{RKL}}^{(i)}
	\right],
	\label{eq:loss_js}
	\\
	\mathcal{L}_{\mathrm{JSWA}}
	&\approx
	\frac{1}{N}\sum_{i=1}^{N}
	\left[
	\beta_1
	\mathcal{L}_{\mathrm{FKL}}^{(i)}
	+
	\mathcal{L}_{\mathrm{rec}}^{(i)}
	+
	\beta_2
	\mathcal{L}_{W_2}^{(i)}
	\right].
	\label{eq:loss_jswa}
\end{align}
Formulation~I corresponds to the standard negative-ELBO structure, with
$\beta_1=1$ recovering the conventional VAE weighting and
$\beta_1\neq1$ controlling the strength of the prior regularization
\citep{higgins_betavae_2017}. Following Goh et al.~\cite{goh_solving_2022},
Formulation~II augments this objective with forward-KL supervision based
on the paired parameter values while retaining the reverse-KL prior
regularizer. Formulation~III retains the same reconstruction and
forward-KL terms but replaces the reverse-KL prior regularizer with the
squared Wasserstein distance.

Table~\ref{tab:loss_summary} summarises the three formulations. Moving
from I to II adds supervised forward-KL; moving from II to III replaces
the reverse-KL regulariser with the Wasserstein term.
\begin{table}[htbp]
	\centering
	\caption{Summary and comparison of the three VAE formulations.
		\emph{Supervised}: requires paired data
		$\{(q^{(i)},u_{\mathrm{obs}}^{(i)})\}$.
		\emph{Finite under support mismatch}: the regularisation term remains
		finite when the two distributions have non-overlapping support.
		\emph{Geom.}: the regularisation term accounts for geometric separation
		between distributions.
		\emph{Tightens}: the corresponding posterior error bound decreases
		with increasing data informativeness.}
	\label{tab:loss_summary}
	\setlength{\tabcolsep}{5.5pt}
	\renewcommand{\arraystretch}{1.3}
	\begin{tabular}{@{} l c c c c c c c @{}}
		\toprule
		\textbf{Formulation} &
		\textbf{Supervised} &
		\textbf{Regularisation} &
		$\beta_1$ &
		$\beta_2$ &
		\makecell{\textbf{Finite under}\\\textbf{support mismatch}} &
		\textbf{Geom.} &
		\makecell{\textbf{Tightens}\\\textbf{with data}} \\
		\midrule
		I.\;\; VAE-KL
		& \xmark
		& $D_{\mathrm{KL}}(Q_\phi \| P_{\mathrm{pr}})$
		& tuned
		& ---
		& \xmark
		& \xmark
		& \xmark \\
		II.\; VAE-JS
		& \cmark
		& $D_{\mathrm{KL}}(Q_\phi \| P_{\mathrm{pr}})$
		& tuned
		& tuned
		& \xmark
		& \xmark
		& \xmark \\
		III.\; VAE-JSWA
		& \cmark
		& $W_2^2(Q_\phi,P_{\mathrm{pr}})$
		& tuned
		& tuned
		& \cmark
		& \cmark
		& \cmark \\
		\bottomrule
	\end{tabular}
\end{table}
\subsection{Training procedure}
\label{sec:training_procedure}

Training proceeds in two stages. First, the decoder is pre-trained as a surrogate for the forward model and then frozen. Second, the encoder is trained by minimizing one of \cref{eq:loss_kl}--\cref{eq:loss_jswa}.

\subsubsection{Stage 1: Surrogate pre-training}
\label{sec:decoder_pretraining}

The decoder $\Psi_d$ is trained to minimise the mean squared error over
forward model evaluations,
\begin{align}
    \mathcal{L}_{\mathrm{surr}}
    = \frac{1}{N_{\mathrm{train}}}\sum_{i=1}^{N_{\mathrm{train}}}
      \|\mathcal{F}(\chi, q^{(i)}) - \Psi_d(q^{(i)})\|_2^2,
    \label{eq:decoder_loss}
\end{align}
on data $q^{(i)}$ sampled from the prior. We monitor the relative $L^2$
validation error
\begin{align}
    e_{\mathrm{surr}}
    = \frac{1}{N_{\mathrm{val}}}\sum_{i=1}^{N_{\mathrm{val}}}
      \frac{\|\mathcal{F}(\chi, q^{(i)}) - \Psi_d(q^{(i)})\|_2}
           {\|\mathcal{F}(\chi, q^{(i)})\|_2},
    \label{eq:surrogate_rel_error}
\end{align}
and require $e_{\mathrm{surr}} < 1\%$ before proceeding to Stage~2. Once this threshold is met, the decoder
weights are frozen for the remainder of training.
\subsubsection{Stage 2: Encoder training}
\label{sec:encoder_training}

With the surrogate decoder $\Psi_d$ frozen, the encoder $\Psi_e$ is
trained independently for each formulation by minimizing the corresponding
objective in \cref{eq:loss_kl}--\cref{eq:loss_jswa} using the
Adam optimizer~\cite{kingma_adam_2015}. The encoder predicts
the posterior mean and the Cholesky factor of the full posterior
covariance, and posterior samples are generated using the
reparameterization trick. These samples are passed through the frozen
decoder to evaluate the reconstruction loss. Although the decoder
parameters are fixed, gradients are propagated through its input to
update the encoder parameters. The complete mini-batch training procedure
is summarized in \cref{alg:training}.

\begin{algorithm}[H]
\caption{Two-Stage Training (Common to All Formulations)}
\label{alg:training}
\begin{algorithmic}[1]
\REQUIRE Forward model $\mathcal{F}$, training data $\{(q^{(i)}, u_{\mathrm{obs}}^{(i)})\}_{i=1}^N$,
         prior $P_{\mathrm{pr}}$, noise $\mathcal{N}(\mu_\eta, \Sigma_\eta)$,
         formulation $\in\{\text{I, II, III}\}$, $(\beta_1, \beta_2)$
\ENSURE Trained encoder $\Psi_e$, frozen decoder $\Psi_d$
\STATE \textbf{--- Stage 1: Surrogate pre-training ---}
\REPEAT
    \STATE Sample mini-batch; compute $\mathcal{L}_{\mathrm{surr}}$~\eqref{eq:decoder_loss};
           update $\Psi_d$ via Adam
\UNTIL{$e_{\mathrm{surr}} < 1\%$ on validation set~\eqref{eq:surrogate_rel_error}}
\STATE Freeze decoder weights $\omega_d$
\STATE \textbf{--- Stage 2: Encoder training ---}
\FOR{each epoch}
    \FOR{each mini-batch $\mathcal{B}$}
        \STATE $(\mu_{\mathrm{post}}^{(i)}, L^{(i)}) \leftarrow \Psi_e(u_{\mathrm{obs}}^{(i)})$;
               \quad $\Sigma_{\mathrm{post}}^{(i)} \leftarrow L^{(i)}{L^{(i)}}^\top$
        \STATE $q_{\mathrm{draw}}^{(i)} \leftarrow \mu_{\mathrm{post}}^{(i)} + L^{(i)}\epsilon^{(i)}$,
               \quad $\epsilon^{(i)} \sim \mathcal{N}(0,I)$
        \STATE $\hat{u}^{(i)} \leftarrow \Psi_d(q_{\mathrm{draw}}^{(i)})$
               \hfill\COMMENT{decoder weights frozen}
        \STATE Compute $\mathcal{L}^{(i)}$ from \eqref{eq:loss_kl}, \eqref{eq:loss_js},
               or \eqref{eq:loss_jswa} per formulation
        \STATE Update $\omega_e$ via Adam on $\frac{1}{|\mathcal{B}|}\sum_{i\in\mathcal{B}}\mathcal{L}^{(i)}$
    \ENDFOR
    \STATE Apply early stopping on validation loss
\ENDFOR
\end{algorithmic}
\end{algorithm}
The regularization weights are selected separately for each formulation.
For Formulation~I, $\beta_1$ is selected from
$\{0.01,0.1,1.0\}$, whereas for Formulations~II and III,
$(\beta_1,\beta_2)$ is selected from $\{0.01,0.1,1.0\}^2$. The
hyperparameters are selected using the $W_2$ distance between the VAE
posterior and a reference MCMC posterior on a held-out validation set.

	\section{Comparative theoretical analysis of posterior approximation}
	\label{sec:theory}
	This section provides a theoretical comparison of the three VAE formulations introduced in \cref{sec:distances}. We first analyze posterior approximation errors and then examine the
	modal gradient behavior of the formulations in the generalized Fisher
	basis.
	\subsection{Posterior approximation error bounds} \label{subsec:posterior_error}
	
	\paragraph{Surrogate induced posterior gap and error decomposition.}
	The posterior approximation error has two distinct contributions: the
	error introduced by replacing the exact forward model $\mathcal{F}$ with
	the surrogate decoder $\Psi_d$, and the error arising from approximating
	the resulting posterior within the variational family. Let
	$P_{\mathrm{post}}^{\mathcal{F}}$ and
	$P_{\mathrm{post}}^{\Psi_d}$ denote the posteriors induced by the exact
	and surrogate forward models, respectively. Then, for any learned
	posterior $Q_\phi$, the triangle inequality for the Wasserstein distance
	gives
	\begin{equation}
		W_2\!\left(
		Q_\phi,P_{\mathrm{post}}^{\mathcal{F}}
		\right)
		\leq
		W_2\!\left(
		Q_\phi,P_{\mathrm{post}}^{\Psi_d}
		\right)
		+
		W_2\!\left(
		P_{\mathrm{post}}^{\Psi_d},
		P_{\mathrm{post}}^{\mathcal{F}}
		\right).
		\label{eq:posterior_error_decomp}
	\end{equation}
	The first term represents the variational approximation error associated
	with encoder training, while the second represents the posterior error
	induced by the forward-model surrogate. Thus, improving the surrogate
	accuracy reduces one component of the total posterior error, while the
	choice of variational objective determines how accurately the
	surrogate-induced posterior is approximated.
	
	Under appropriate regularity and stability assumptions on the posterior,
	the surrogate-induced term can be bounded in terms of the forward-model
	approximation error. In particular, if
	\begin{equation}
		\sup_q
		\left\|
		\mathcal{F}(\chi,q)-\Psi_d(q)
		\right\|_2
		\leq \varepsilon_{\mathrm{surr}},
		\label{eq:surr_uniform_error}
	\end{equation}
	then the posterior perturbation can be controlled by a stability bound
	of the form
	\begin{equation}
		W_2\!\left(
		P_{\mathrm{post}}^{\Psi_d},
		P_{\mathrm{post}}^{\mathcal{F}}
		\right)
		\leq
		C_{\mathrm{surr}}\,
		\varepsilon_{\mathrm{surr}},
		\label{eq:surrogate_posterior_bound}
	\end{equation}
	where $C_{\mathrm{surr}}$ depends on the noise model and the regularity
	of the posterior. Consequently, the surrogate should be sufficiently
	accurate that its contribution to the posterior error is small relative
	to the variational approximation error. This provides the theoretical
	motivation for the surrogate-accuracy criterion adopted in
	Stage~1.
	
	Since $W_2$ is a metric, the total posterior approximation error can be
	decomposed as
	\begin{align}
		W_2\!\left(
		Q_\phi,P_{\mathrm{post}}^{\mathcal F}
		\right)
		&\leq
		\underbrace{
			W_2\!\left(
			Q_\phi,P_{\mathrm{post}}^{\Psi_d}
			\right)
		}_{\text{encoder approximation error}}
		+
		\underbrace{
			W_2\!\left(
			P_{\mathrm{post}}^{\Psi_d},
			P_{\mathrm{post}}^{\mathcal F}
			\right)
		}_{\text{surrogate-induced error}}.
		\label{eq:error_decomp}
	\end{align}
	The first term is associated with the approximation achieved by encoder
	training, whereas the second is controlled by the accuracy of the
	surrogate decoder. This decomposition therefore separates the two
	sources of posterior error addressed by the two-stage training
	procedure: surrogate pre-training controls the forward-model
	approximation error, while encoder training controls the variational
	approximation error.
	\paragraph{Gaussian posterior approximation.}
	For the Gaussian posterior family used in this work, the variational
	approximation error admits a particularly transparent decomposition.
	Let
	\begin{equation}
		Q_\phi
		=
		\mathcal{N}(\mu_\phi,\Sigma_\phi),
		\qquad
		P_{\mathrm{post}}
		=
		\mathcal{N}(\mu_*,\Sigma_*).
	\end{equation}
	Their squared Wasserstein distance is
	\begin{equation}
		W_2^2(Q_\phi,P_{\mathrm{post}})
		=
		\|\mu_\phi-\mu_*\|_2^2
		+
		d_B^2(\Sigma_\phi,\Sigma_*),
		\label{eq:gaussian_w2_error}
	\end{equation}
	where
	\begin{equation}
		d_B^2(\Sigma_\phi,\Sigma_*)
		=
		\operatorname{tr}\!\left(
		\Sigma_\phi+\Sigma_*
		-
		2
		\left(
		\Sigma_*^{1/2}
		\Sigma_\phi
		\Sigma_*^{1/2}
		\right)^{1/2}
		\right)
	\end{equation}
	is the squared Bures distance between the covariance matrices~\citep{bhatia_bures_2019a}. Hence,
	the Wasserstein posterior error simultaneously accounts for errors in
	the posterior mean and covariance. In particular,
	\begin{equation}
		\|\mu_\phi-\mu_*\|_2
		\leq
		W_2(Q_\phi,P_{\mathrm{post}}),
		\qquad
		d_B(\Sigma_\phi,\Sigma_*)
		\leq
		W_2(Q_\phi,P_{\mathrm{post}}).
		\label{eq:w2_component_bounds}
	\end{equation}
	Thus, for the full-covariance Gaussian encoder considered here, $W_2$ provides upper bounds on discrepancies in both the posterior mean and covariance.
	
\subsection{Comparative modal gradient analysis}\label{subsec:modal_gradient}
The generalized Fisher basis provides a natural coordinate system for
examining how the different terms in the VAE objectives influence
posterior uncertainty along directions with different information
content. We use the local
linear-Gaussian approximation described in \cref{sec:fisher} to
obtain closed-form expressions for the modal gradient contributions.
The analysis therefore characterizes the local
behavior of the nonlinear inverse problem through its generalized Fisher modes. We focus on the variance component of the encoder posterior,
since the generalized Fisher eigenvalues characterize the information
content of each mode and therefore its posterior variance under the
local linear-Gaussian approximation. This provides a natural setting
for comparing the variance response of the different regularizers. Detailed derivations of the modal gradient contributions are provided in
\Cref{app:recon_gradient,app:fkl_gradient}.

\paragraph{Modal gradient equilibrium.}
Let $v_k=\tilde{\sigma}_k^2$ denote the encoder variance in generalized
Fisher mode $k$. Here, $v_k$ represents the current posterior variance
assigned by the encoder to mode $k$, and is therefore the quantity whose
sensitivity to the objective is examined. At a stationary point, the variance satisfies
\begin{align}
	\frac{\partial \mathcal{L}_{\mathrm{rec}}}{\partial v_k}
	+
	\beta_1
	\frac{\partial \mathcal{L}_{\mathrm{FKL}}}{\partial v_k}
	+
	\beta_2
	\frac{\partial \mathcal{L}_{\mathrm{reg}}}{\partial v_k}
	=0,
	\label{eq:modal_gradient_balance}
\end{align}
where the three terms represent the reconstruction, forward-KL, and
prior-regularization contributions, respectively. The individual
contributions to this balance are derived below.

\paragraph{Reconstruction gradient in the generalized Fisher basis.}
The reconstruction term in the VAE objective is the expectation of the
sample-wise reconstruction loss introduced in
\cref{eq:loss_rec} with respect to the encoder posterior,
\begin{align}
	\mathcal{L}_{\mathrm{rec}}
	&=
	\mathbb{E}_{q\sim Q_\phi(q\mid u_{\mathrm{obs}})}
	\left[
	\ell_{\mathrm{rec}}(q;u_{\mathrm{obs}})
	\right].
	\label{eq:expected_reconstruction_loss}
\end{align}
Under the local linear-Gaussian approximation with isotropic observation
noise and a calibrated encoder mean, the variance-dependent part of
this expected reconstruction loss can be expressed in the generalized
Fisher basis as
\begin{align}
	\mathcal{L}_{\mathrm{rec}}
	&=
	\frac{1}{2}
	\sum_{k=1}^{d}
	\lambda_k^{\mathrm{Fisher}}v_k
	+\mathrm{const.},
	\label{eq:recon_variance_fisher}
\end{align}
where terms independent of $v_k$ are absorbed into the constant.
Consequently,
\begin{align}
	\frac{\partial\mathcal{L}_{\mathrm{rec}}}{\partial v_k}
	=
	\frac{1}{2}\lambda_k^{\mathrm{Fisher}},
	\label{eq:recon_gradient_fisher}
\end{align}
Thus, within the local linear-Gaussian approximation, the reconstruction
sensitivity to the modal variance is independent of the current
variance and is proportional to the generalized Fisher eigenvalue. The
detailed derivation is given in \cref{app:recon_gradient}.

\paragraph{Forward-KL gradient and curvature.}
For the local Gaussian posterior, the variance-dependent forward-KL
contribution in mode $k$ is
\begin{align}
	\mathcal{L}_{\mathrm{FKL},k}(v_k)
	&=
	D_{\mathrm{KL}}
	\left(
	P_{\mathrm{post}}^{(k)}
	\,\|\,Q_\phi^{(k)}
	\right)
	\nonumber\\
	&=
	\frac{1}{2}
	\left[
	\frac{v_k^*}{v_k}
	+\ln v_k
	-\ln v_k^*
	-1
	\right],
	\label{eq:fwdkl_closed_form}
\end{align}
where $v_k^*$ denotes the corresponding local posterior variance. From
the posterior covariance relation established in \cref{sec:fisher},
\begin{align}
	v_k^*
	=
	\frac{1}{1+\lambda_k^{\mathrm{Fisher}}}.
	\label{eq:posterior_variance_fisher}
\end{align}
Differentiating \cref{eq:fwdkl_closed_form} gives
\begin{align}
	\frac{\partial\mathcal{L}_{\mathrm{FKL},k}}
	{\partial v_k}
	&=
	\frac{v_k-v_k^*}{2v_k^2}.
	\label{eq:fwdkl_gradient}
\end{align}
The gradient vanishes at $v_k=v_k^*$, is negative for $v_k<v_k^*$,
and positive for $v_k>v_k^*$. Thus, the forward-KL term provides a
variance-dependent correction toward the local posterior variance.


\paragraph{Comparison of the prior regularizers.}
The three formulations differ in their prior regularization: VAE-KL and
VAE-JS use the reverse-KL divergence
$D_{\mathrm{KL}}(Q_\phi\|P_{\mathrm{pr}})$, whereas VAE-JSWA uses
$W_2^2(Q_\phi,P_{\mathrm{pr}})$. In the whitened generalized Fisher basis,
the prior is standard normal. For the diagonal modal covariance
considered here, the variance-dependent parts of the two regularizers
are
\begin{align}
	\mathcal{L}_{\mathrm{RKL},k}(v_k)
	&=
	\frac{1}{2}
	\left(v_k-\ln v_k-1\right),
	&
	\mathcal{L}_{W_2,k}(v_k)
	&=
	(\sqrt{v_k}-1)^2.
	\label{eq:modal_regularizers}
\end{align}
Their gradients with respect to $v_k$ are therefore
\begin{align}
	\frac{\partial\mathcal{L}_{\mathrm{RKL},k}}{\partial v_k}
	&=
	\frac{1}{2}
	\left(1-\frac{1}{v_k}\right),
	&
	\frac{\partial\mathcal{L}_{W_2,k}}{\partial v_k}
	&=
	1-\frac{1}{\sqrt{v_k}}.
	\label{eq:regularizer_gradients_variance}
\end{align}
Both gradients vanish at the prior variance $v_k=1$ and are negative
for $v_k<1$, thereby opposing the positive reconstruction gradient when
the posterior is narrower than the prior.

The two regularizers differ in their sensitivity to small posterior
variances:
\begin{align}
	\frac{\partial\mathcal{L}_{\mathrm{RKL},k}}{\partial v_k}
	&=
	\mathcal{O}(v_k^{-1}),
	&
	\frac{\partial\mathcal{L}_{W_2,k}}{\partial v_k}
	&=
	\mathcal{O}(v_k^{-1/2}),
	\qquad
	v_k\rightarrow0.
	\label{eq:regularizer_asymptotics}
\end{align}
Thus, the Wasserstein gradient grows more slowly than the reverse-KL
gradient as the posterior variance approaches zero.

Evaluating the gradients at the local posterior variance
\cref{eq:posterior_variance_fisher} gives
\begin{align}
	\left.
	\frac{\partial\mathcal{L}_{\mathrm{RKL},k}}{\partial v_k}
	\right|_{v_k=v_k^*}
	&=
	-\frac{\lambda_k^{\mathrm{Fisher}}}{2},
	&
	\left.
	\frac{\partial\mathcal{L}_{W_2,k}}{\partial v_k}
	\right|_{v_k=v_k^*}
	&=
	1-\sqrt{1+\lambda_k^{\mathrm{Fisher}}}.
	\label{eq:regularizer_at_posterior}
\end{align}
For strongly informative modes,
$\lambda_k^{\mathrm{Fisher}}\gg1$, the corresponding magnitudes scale as
\begin{align}
	\left|
	\frac{\partial\mathcal{L}_{\mathrm{RKL},k}}{\partial v_k}
	\right|
	&\sim
	\frac{\lambda_k^{\mathrm{Fisher}}}{2},
	&
	\left|
	\frac{\partial\mathcal{L}_{W_2,k}}{\partial v_k}
	\right|
	&\sim
	\sqrt{\lambda_k^{\mathrm{Fisher}}}.
	\label{eq:regularizer_scaling}
\end{align}
Hence, the reverse-KL regularizer has a stronger variance sensitivity
than the Wasserstein regularizer in highly concentrated posterior
directions.
\paragraph{Combined modal gradient.}
The behavior of the encoder variance is determined by the combined
contribution of the reconstruction, forward-KL, and prior-regularization
terms. In a Fisher mode $k$, the corresponding variance gradient is
\begin{align}
	\frac{\partial \mathcal L}{\partial v_k}
	&=
	\frac{1}{2}\lambda_k^{\mathrm{Fisher}}
	+
	\beta_1
	\frac{v_k-v_k^*}{2v_k^2}
	+
	\beta_2
	\frac{\partial L_{\mathrm{reg},k}}{\partial v_k},
	\label{eq:combined_modal_gradient}
\end{align}
where $L_{\mathrm{reg},k}$ is either the reverse-KL or squared
$W_2$ regularizer. Thus, the reconstruction term provides a
Fisher-information-dependent contribution, the forward-KL term provides
a restoring contribution toward the Bayesian posterior variance
$v_k^*$, and the prior regularizer provides an additional constraint
toward the prior variance. The modal equilibrium is therefore determined
by the balance of these three contributions through
$\partial\mathcal L/\partial v_k=0$.

\paragraph{Implications across identifiability regimes.}
The preceding results characterize how the different gradient
contributions vary with the generalized Fisher eigenvalue. The
reconstruction gradient scales linearly with
$\lambda_k^{\mathrm{Fisher}}$, while the local posterior variance
decreases as $(1+\lambda_k^{\mathrm{Fisher}})^{-1}$. The forward-KL
contribution provides a variance-dependent correction toward this local
posterior variance, with local curvature increasing with the information
content of the mode. The two prior regularizers exhibit different
responses to posterior concentration: the reverse-KL gradient is more
singular as $v_k\rightarrow0$, whereas the Wasserstein gradient grows
more slowly. These differences provide the basis for comparing the
three formulations across modes with different levels of local
identifiability.


\section{Numerical results}\label{sec:numerical_exp}

Four test problems are considered to investigate how the performance of
KL-, JS-, and the proposed JSWA-regularized VAE formulations depends on
the information geometry of the underlying Bayesian inverse problem.
The experiments test the central hypothesis developed in
\Cref{sec:fisher,sec:theory}: rather than being universally
optimal, the effectiveness of a variational objective depends on the
distribution of information across parameter directions. In the Fisher
basis, the encoder-variance dynamics are governed by the balance between
an information-driven reconstruction contribution, determined by the
Fisher eigenvalue, a forward-KL posterior-correcting contribution
weighted by $\beta_1$, and a prior-restoring contribution weighted by
$\beta_2$. This framework suggests that regularization is most important in weakly informed directions, where small Fisher eigenvalues indicate that the
likelihood provides only a weak constraint on the posterior uncertainty.
In strongly informed directions, larger Fisher eigenvalues indicate that
the likelihood provides a stronger constraint on the posterior variance.

To assess the robustness of the proposed formulations beyond the
controlled linear-Gaussian setting, the numerical study progressively
considers nonlinear and field-valued inverse problems of increasing
complexity. The linear--Gaussian benchmark (\Cref{sec:prob0}) has an exact closed-form posterior and is used to verify the gradient-level mechanism in the generalized Fisher basis. The one-dimensional heat-conduction inverse problem (\Cref{sec:tc1}) extends the analysis to a nonlinear inverse problem governed by a linear ODE, while the two-dimensional conductivity-inversion problem (\Cref{sec:tc2}) considers a nonlinear inverse problem governed by a linear PDE. Together with the linear--Gaussian benchmark, these problems provide a progression in posterior conditioning that allows the relative
performance of the three objectives to be examined beyond the analytically
controlled setting. Finally, the two-dimensional temperature-dependent
conductivity problem (\Cref{sec:tc3})introduces a nonlinear PDE through
the parameter-dependent conductivity and joint parameter inference, and is
used as a qualitative test of whether the proposed interpretation extends
to stronger forward-model nonlinearity. The governing equations and associated parameters for the two-dimensional heat-conduction problems considered in \Cref{sec:tc2,sec:tc3} are expressed in dimensionless form.

The differential equation-based test cases follow a common
training and evaluation protocol. The KL, JS, and JSWA models use identical
neural network architectures and dataset splits, with 80\%, 10\%, and 10\%
of the data assigned to training, validation, and test sets, respectively.
VAE posteriors are benchmarked against adaptive Metropolis MCMC reference
solutions (\Cref{app:am}).

The two-dimensional PDE-based forward problems are discretized using the
finite-element formulation described in \Cref{app:fem_kl} and implemented
using FEniCSx~1.0.3~\citep{ScroggsEtal2022}. MCMC convergence is verified
using $\hat{R}<1.01$ and $\mathrm{ESS}\geq400$
(\Cref{app:convergence}). Posterior approximation is evaluated using the
metrics defined in \Cref{app:metrics}, including relative errors in
posterior mean and standard deviation and the Wasserstein distance $W_2$.
The results are interpreted in relation to the FIM spectra and posterior
condition numbers. All experiments
are implemented in TensorFlow~2.13.0 and executed on an NVIDIA RTX~A5000
GPU (24\,GB VRAM, CUDA~12.8) using fixed random seeds.


For nonlinear inverse problems, ensemble-level results are complemented by a detailed analysis of one representative test realization. The representative realization is selected independently of VAE performance as the test case whose posterior condition number is closest to the ensemble median, and the same realization is used for all three formulations within each test case.

\subsection{Linear Gaussian benchmark with exact analytical posterior}
\label{sec:prob0}
We first consider a linear Gaussian inverse problem for which the exact
posterior is available in closed form. This benchmark is used to examine
how the relative performance of KL-, JS- and JSWA-regularized VAEs
depends on the information geometry of the inverse problem and to
verify the Fisher-mode gradient mechanism developed in
\cref{sec:fisher}. Two cases with markedly different
posterior conditioning are considered: an ill-conditioned case with
$\kappa(\Sigma_{\mathrm{post}}^{*})\approx 8000$ and a well-conditioned
case with $\kappa(\Sigma_{\mathrm{post}}^{*})\approx 8$. Since the exact
posterior is known, this benchmark eliminates MCMC approximation error
and provides a controlled setting for isolating the effects of Fisher
information and the variational objective.The learned modal variances
are interpreted through the gradient equilibrium given by
\cref{eq:modal_gradient_balance}. In particular, the different variance
dependence of the reverse-KL and Wasserstein regularizers,
\cref{eq:regularizer_gradients_variance}, determines how the equilibrium
changes across the Fisher spectrum. This provides a common framework for
interpreting the contrasting behavior of the three objectives in the
ill-conditioned and well-conditioned regimes.

\paragraph{Ill-conditioned case
($\kappa(\Sigma_{\mathrm{post}}^{*})\approx8000$).}
\label{kappa_8000}
For the linear--Gaussian benchmark, the general forward model in
\Cref{eq:forward_model} is specified as
\begin{equation}
\mathcal{F}(q)=Aq,
\qquad
A\in\mathbb{R}^{20\times5},
\end{equation}
where the parameter prior model follow
\Cref{eq:obs_model}, with $Q\sim\mathcal{N}(0,I_5)$ and $\eta\sim\mathcal{N}(0,\sigma_\eta^2I)$ with $\sigma_\eta=0.05$.

The singular values of $A$ are
$s=[10,\,3,\,1,\,0.3,\,0.1]$, which give the Fisher eigenvalues
\[
\lambda_k^{\mathrm{Fisher}}
=
\{40000,\;3600,\;400,\;36,\;4\}.
\]
The corresponding posterior variances range from
$2.5\times10^{-5}$ to $0.2$, producing a four-order-of-magnitude
separation between the strongest- and weakest-information directions. The posterior is obtained from the linear--Gaussian posterior expression in \Cref{eq:posterior_variance}.

\begin{table}[ht]
\centering
\caption{Fisher information spectrum, exact posterior variances, and
their contributions to the total posterior variance for the
ill-conditioned linear Gaussian benchmark.}
\label{tab:prob0_kappa10000}
\begin{tabular}{cccc}
\toprule
Mode $k$ & $\lambda_k^{\mathrm{Fisher}}$ &
True $\tilde{\sigma}_k^{*2}$ & Variance contribution (\%) \\
\midrule
1 & 40000   & $2.50\times10^{-5}$ & 0.01 \\
2 & 3600    & $2.77\times10^{-4}$ & 0.12 \\
3 & 400     & $2.49\times10^{-3}$ & 1.08 \\
4 & 36      & $2.71\times10^{-2}$ & 11.79 \\
5 & 4       & $2.00\times10^{-1}$ & 87.05\\
\bottomrule
\end{tabular}
\end{table}
\Cref{tab:prob0_aggregate} shows the aggregate posterior errors
over 500 test observations. JSWA gives the lowest posterior-mean,
covariance, and $W_2$ errors, with values of $0.0732$, $0.0128$, and
$0.079$, respectively. Overall, JSWA provides the most favourable
combined performance, reducing $W_2$ by approximately $21.7\%$ relative
to JS and $35.2\%$ relative to KL.

\begin{table}[ht]
\centering
\caption{Aggregate posterior accuracy for the ill-conditioned linear
Gaussian benchmark.}
\label{tab:prob0_aggregate}
\begin{tabular}{lccc}
\toprule
Method & $e_\mu$ & $e_{\mathrm{cov}}$ & $W_2$ \\
\midrule
VAE-KL   & 0.1116 & 0.0335 & 0.122 \\
VAE-JS   & 0.0922 & 0.0156 & 0.101 \\
VAE-JSWA & \textbf{0.0732} & \textbf{0.0128} & \textbf{0.079} \\
\bottomrule
\end{tabular}
\end{table}

\Cref{fig:cdf} confirms the aggregate result at the distributional
level. The JSWA CDF is shifted furthest toward smaller $W_2$ values,
with its median reduced by $0.0215$ relative to JS and by $0.0407$
relative to KL. Thus, the improvement is distributed across the test
samples rather than arising from a small subset of observations.

\begin{figure}[ht]
\centering
\includegraphics[width=0.65\textwidth]{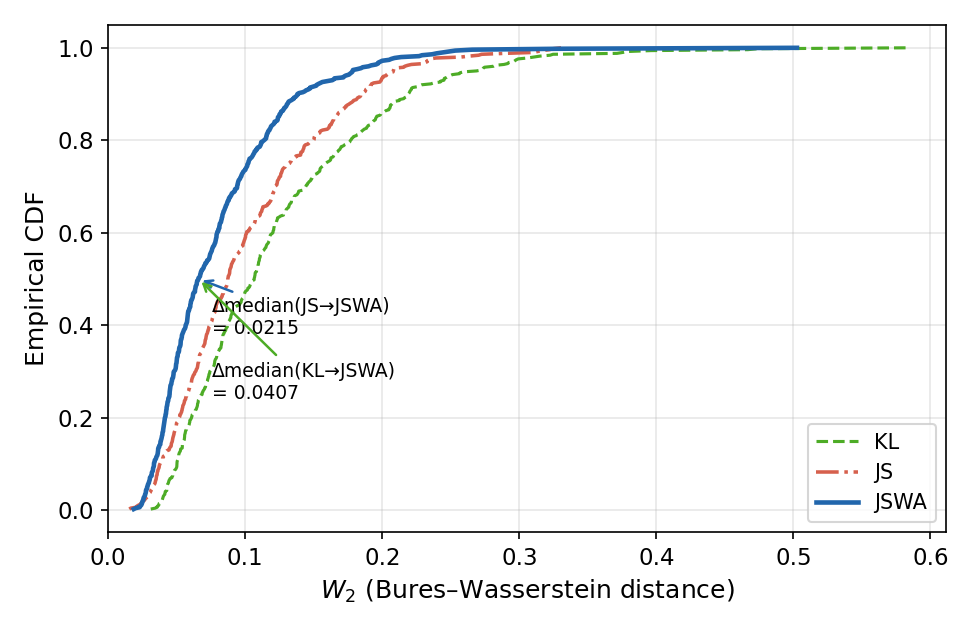}
\caption{Empirical CDF of the $2$-Wasserstein distance
$W_2(Q_\phi,P_{\mathrm{post}}^*)$ for the ill-conditioned benchmark.}
\label{fig:cdf}
\end{figure}

The exact Fisher-basis variances ($\tilde{\sigma}_k^{*2}$) are reported in \Cref{tab:prob0_kappa10000}.
The weakest mode alone accounts for approximately $87\%$ of the total
posterior variance, while the last two modes together account for
approximately $99\%$. Consequently, variance recovery in the
weak-information part of the spectrum has a dominant influence on the
global covariance error.

\begin{table}[ht]
\centering
\caption{Per-eigenmode posterior variance and relative variance error
(in parentheses) for the ill-conditioned benchmark.}
\label{tab:prob0_eigenmode}
\begin{tabular}{ccrrrr}
\toprule
Mode $k$ & FIM $\lambda_k$ &
True $\tilde{\sigma}_k^{*2}$ &
VAE-JSWA & VAE-JS & VAE-KL \\
\midrule
1 & 40000 &
$2.50\times10^{-5}$ &
$1.95\times10^{-5}$ ($22.0\%$) &
$2.30\times10^{-5}$ ($8.0\%$) &
$2.95\times10^{-5}$ ($18.0\%$) \\

2 & 3600 &
$2.70\times10^{-4}$ &
$2.22\times10^{-4}$ ($17.8\%$) &
$3.15\times10^{-4}$ ($16.7\%$) &
$3.30\times10^{-4}$ ($22.2\%$) \\

3 & 400 &
$2.49\times10^{-3}$ &
$2.00\times10^{-3}$ ($19.7\%$) &
$2.80\times10^{-3}$ ($12.4\%$) &
$3.02\times10^{-3}$ ($21.3\%$) \\

4 & 36 &
$2.70\times10^{-2}$ &
$2.30\times10^{-2}$ ($14.8\%$) &
$3.10\times10^{-2}$ ($14.8\%$) &
$3.20\times10^{-2}$ ($18.5\%$) \\

5 & 4 &
$\mathbf{2.00\times10^{-1}}$ &
$\mathbf{1.94\times10^{-1}}$ ($\mathbf{3.0\%}$) &
$\mathbf{2.25\times10^{-1}}$ ($\mathbf{12.5\%}$) &
$\mathbf{2.31\times10^{-1}}$ ($\mathbf{15.5\%}$) \\

\bottomrule
\end{tabular}
\end{table}

As shown in \Cref{tab:prob0_eigenmode}, the dominant weak-information mode ($\lambda_k=4$) provides the clearest distinction among the three formulations: JSWA recovers its variance with only $3.0\%$ error, compared with $12.5\%$ for JS and $15.5\%$ for KL. Since this mode contains most of the posterior uncertainty, its improved recovery has a disproportionate effect on the global metrics. Errors in the strongly informed modes,
although sometimes larger in relative terms, have much less influence because their absolute variances are very small.

\begin{figure*}[!t]
    \centering

    \begin{subfigure}[h]{0.45\textwidth}
        \centering
        \includegraphics[width=\linewidth]{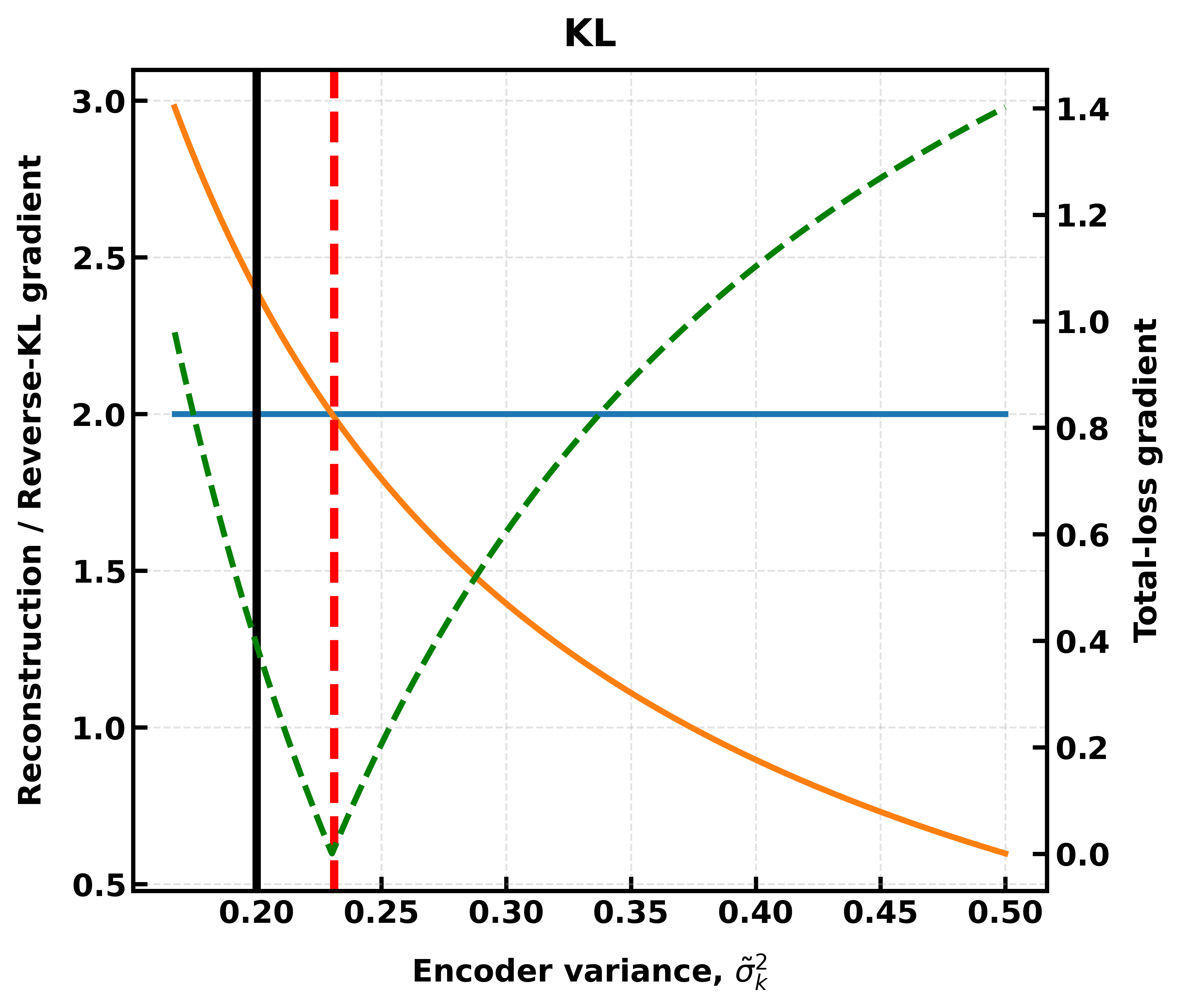}
        \caption{KL}
        \label{fig:gradient_equil_low_kl}
    \end{subfigure}
    \hfill
    \begin{subfigure}[h]{0.45\textwidth}
        \centering
        \includegraphics[width=\linewidth]{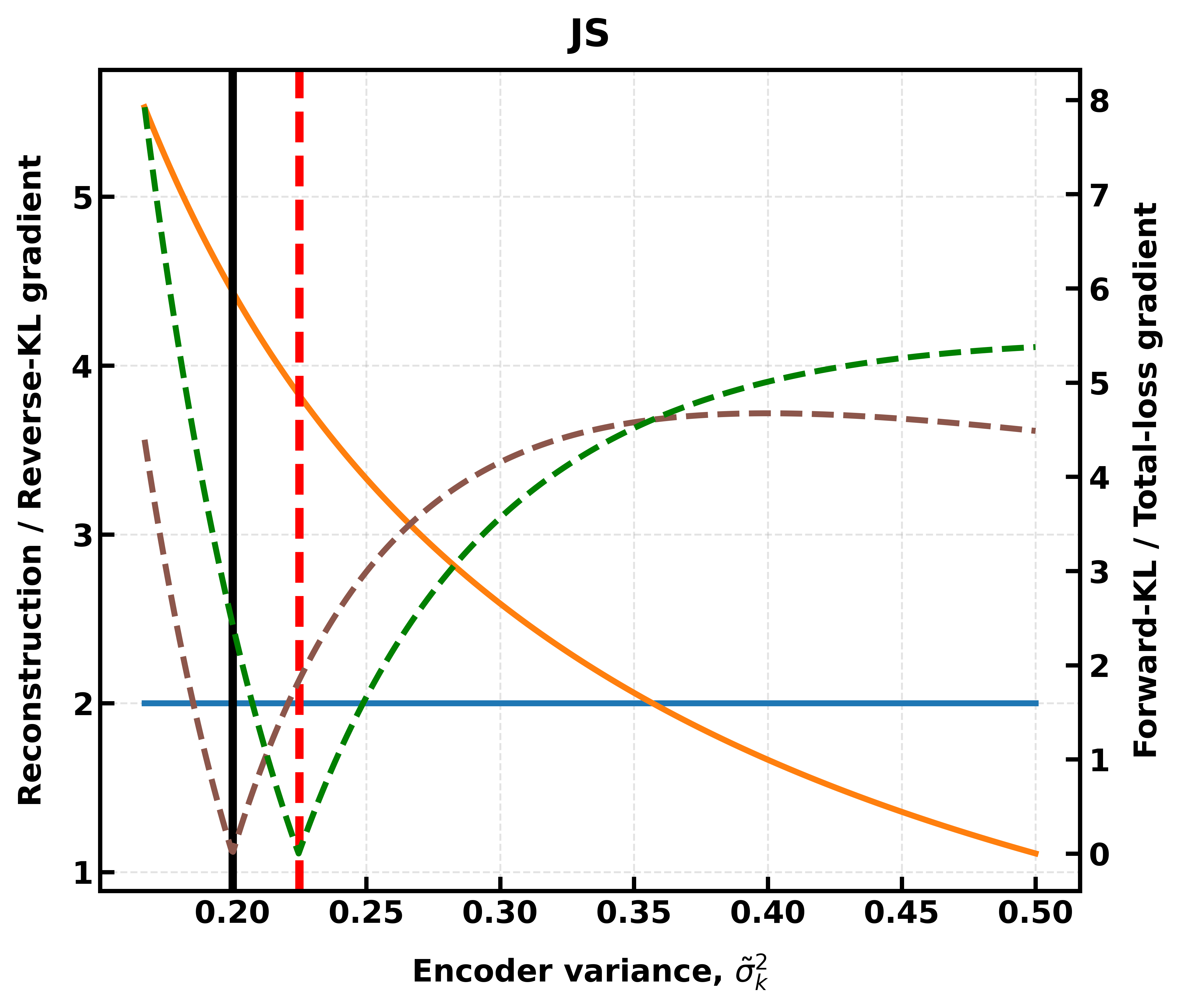}
        \caption{JS}
        \label{fig:gradient_equil_low_js}
    \end{subfigure}

    \vspace{0.03cm}

    \begin{subfigure}[h]{0.45\textwidth}
        \centering
        \includegraphics[width=\linewidth]{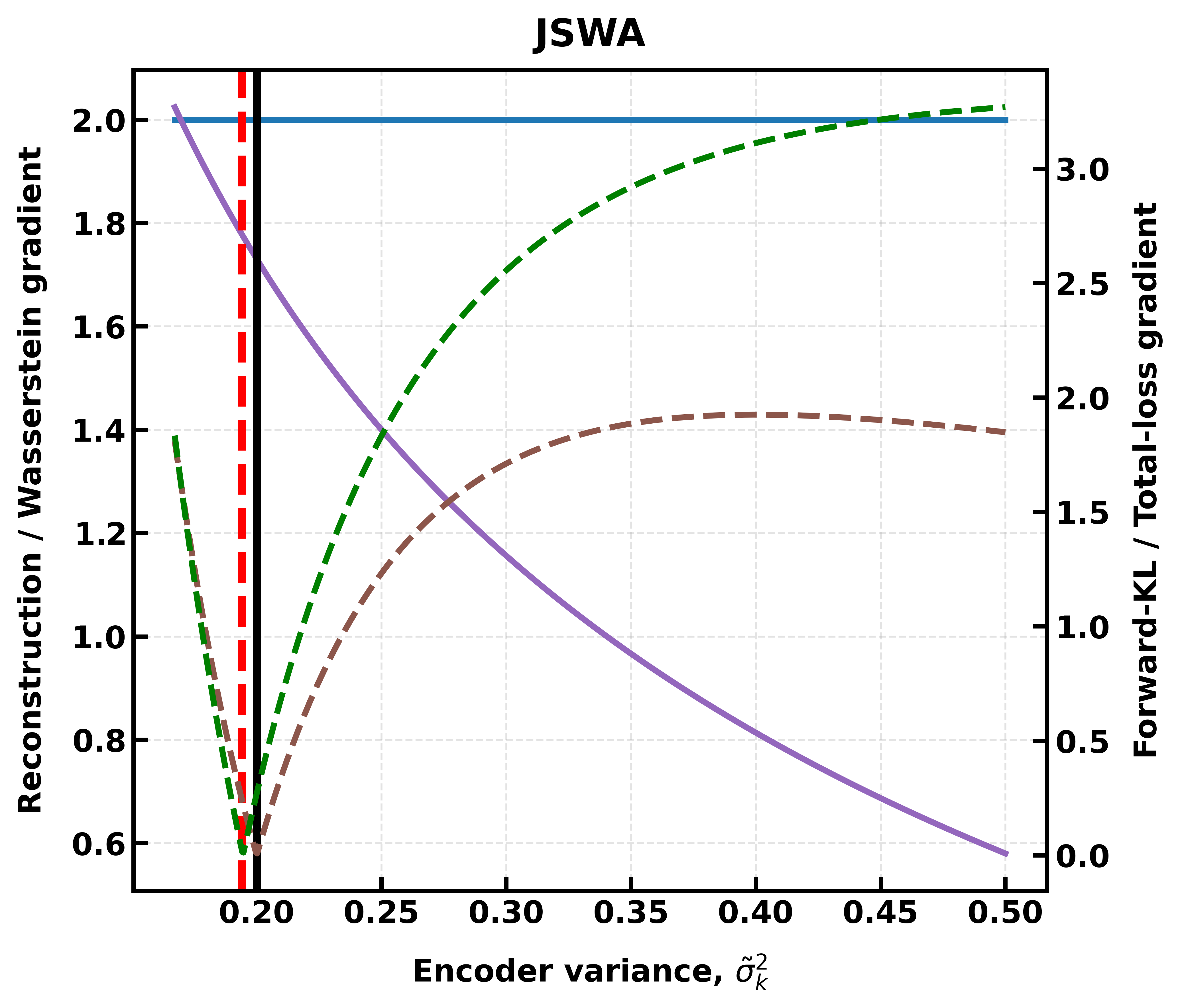}
        \caption{JSWA}
        \label{fig:gradient_equil_low_jswa}
    \end{subfigure}

    \vspace{0.05cm}

    \begin{subfigure}[t]{\textwidth}
        \centering
        \includegraphics[width=\linewidth]{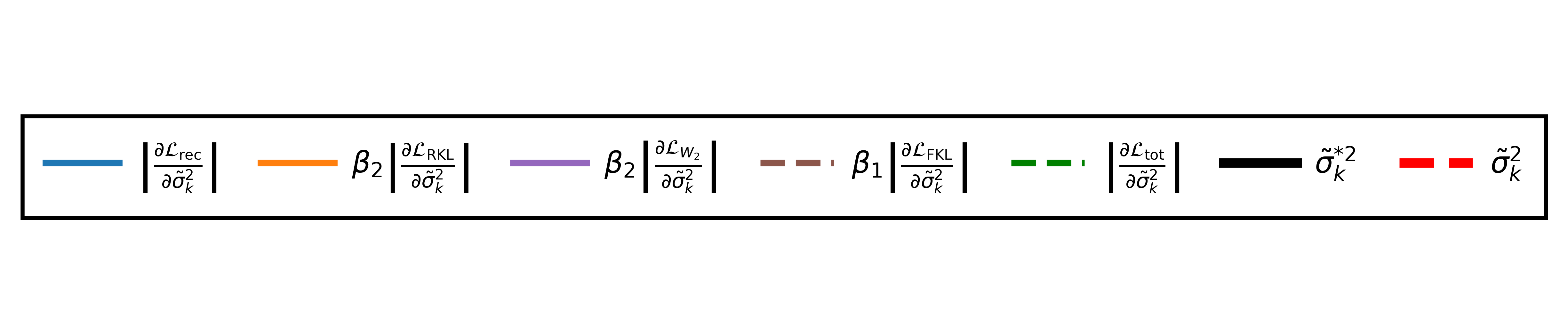}
        \caption*{}
    \end{subfigure}

    \caption{Gradient equilibrium for the weakest-information Fisher mode
    ($\lambda_k=4$) of the ill-conditioned benchmark. The common legend
    identifies the reconstruction, reverse Kullback--Leibler,
    Wasserstein, forward Kullback--Leibler, and total-loss gradient
    contributions, together with the true and encoder variance locations.}

    \label{fig:gradient_equil_low}

\end{figure*}

The gradient equilibrium is examined at the two extremes of the Fisher spectrum, corresponding to the weakest and strongest information modes. For the weakest mode, $\lambda_k=4$, the reconstruction
contribution is comparatively small (\Cref{fig:gradient_equil_low}).The KL reconstruction--reverse-KL gradient
balance occurs at $\tilde{\sigma}_k^2\approx0.231$, above the exact value $0.20$.
For JS, the forward-KL contribution shifts this overestimated
equilibrium toward $0.225$. For JSWA, the weaker Wasserstein
dependence allows the reconstruction--Wasserstein balance to favor
an underestimated variance; forward KL shifts the equilibrium upward
to approximately $0.194$. Thus, forward KL acts in opposite directions
for JS and JSWA, counteracting the bias produced by their respective regularizer balances.
\begin{figure*}[!t]
    \centering

    \begin{subfigure}[h]{0.45\textwidth}
        \centering
        \includegraphics[width=\linewidth]{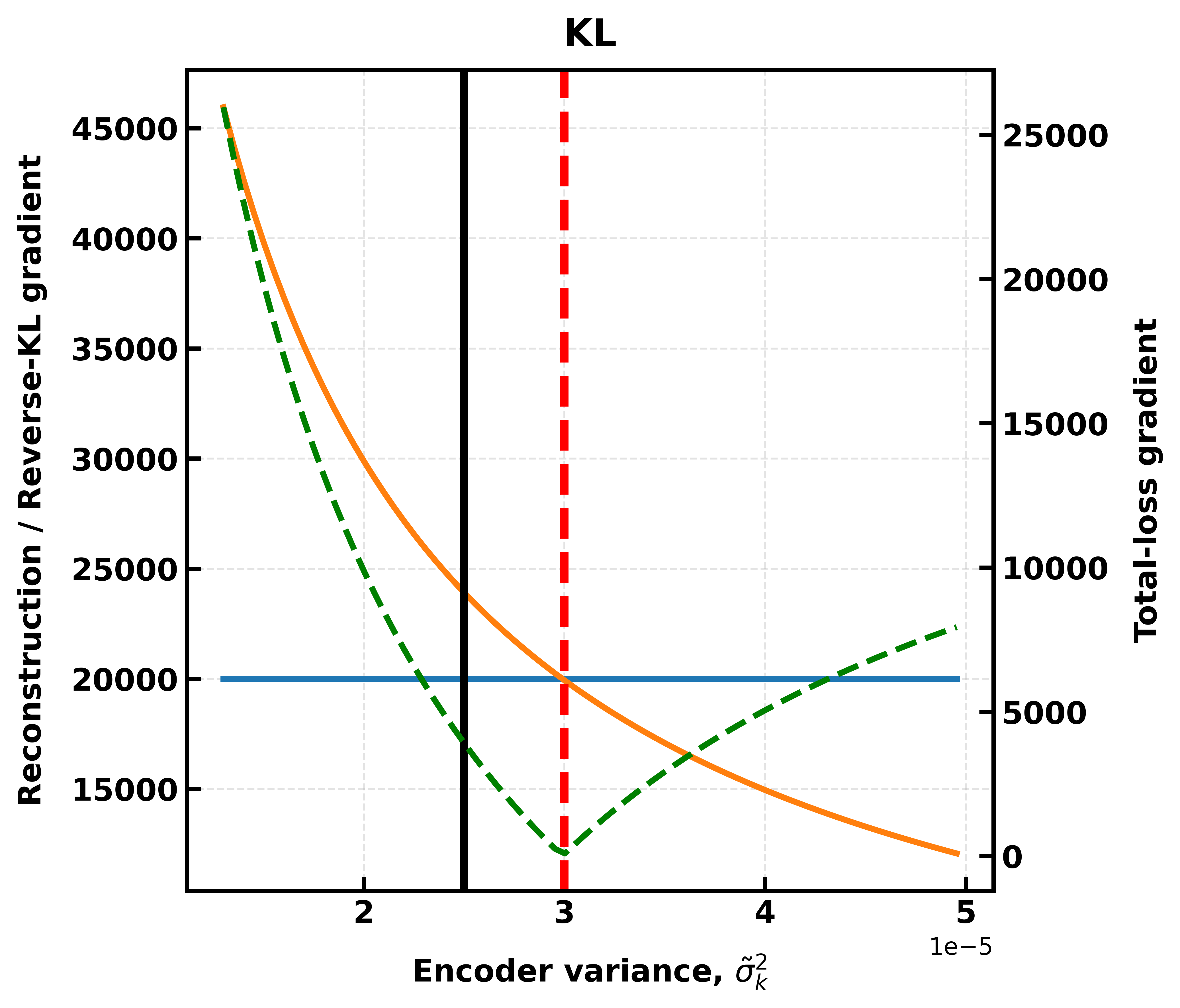}
        \caption{KL}
        \label{fig:gradient_equil_high_kl}
    \end{subfigure}
    \hfill
    \begin{subfigure}[h]{0.45\textwidth}
        \centering
        \includegraphics[width=\linewidth]{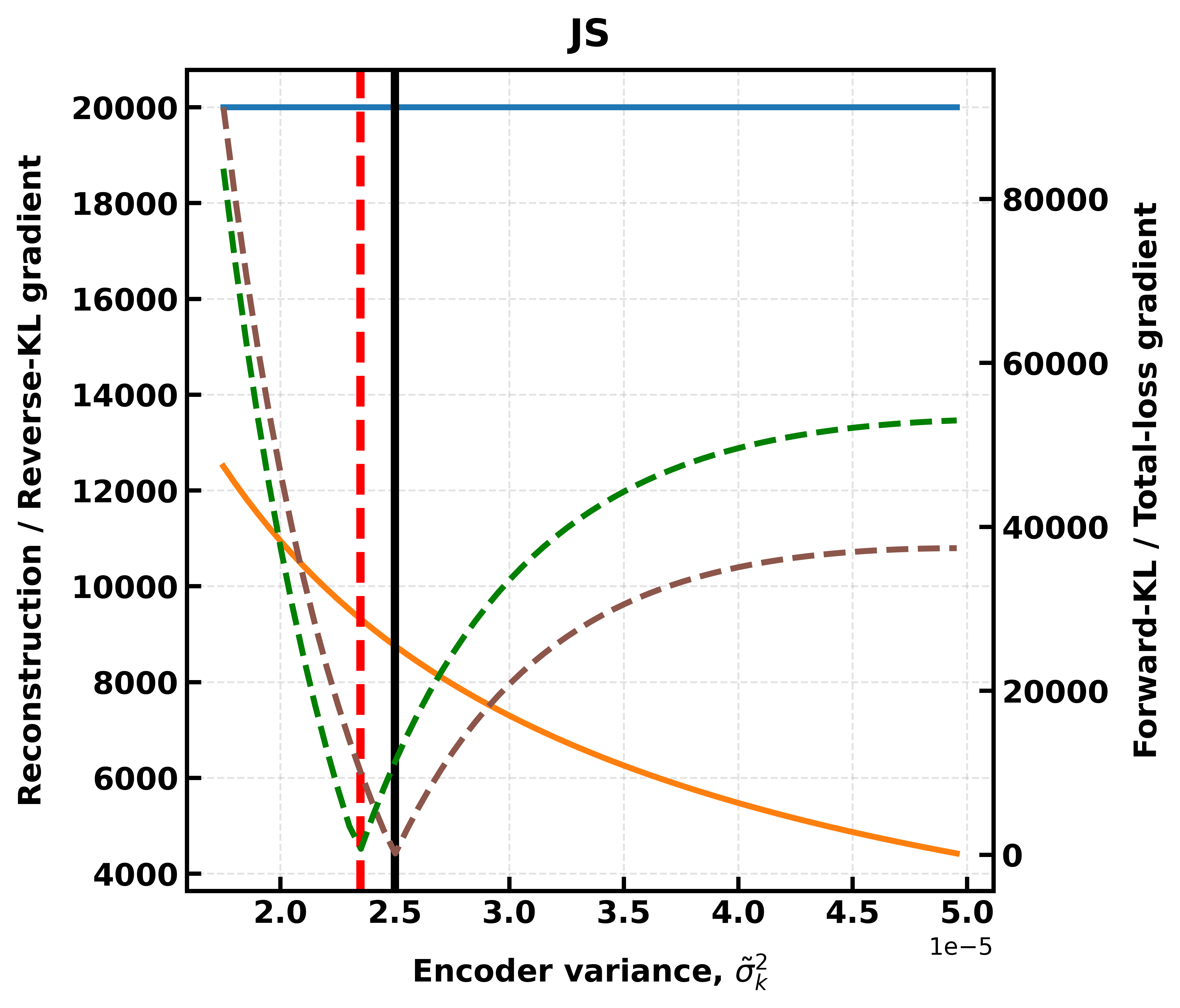}
        \caption{JS}
        \label{fig:gradient_equil_high_js}
    \end{subfigure}

    \vspace{0.03cm}

    \begin{subfigure}[h]{0.45\textwidth}
        \centering
        \includegraphics[width=\linewidth]{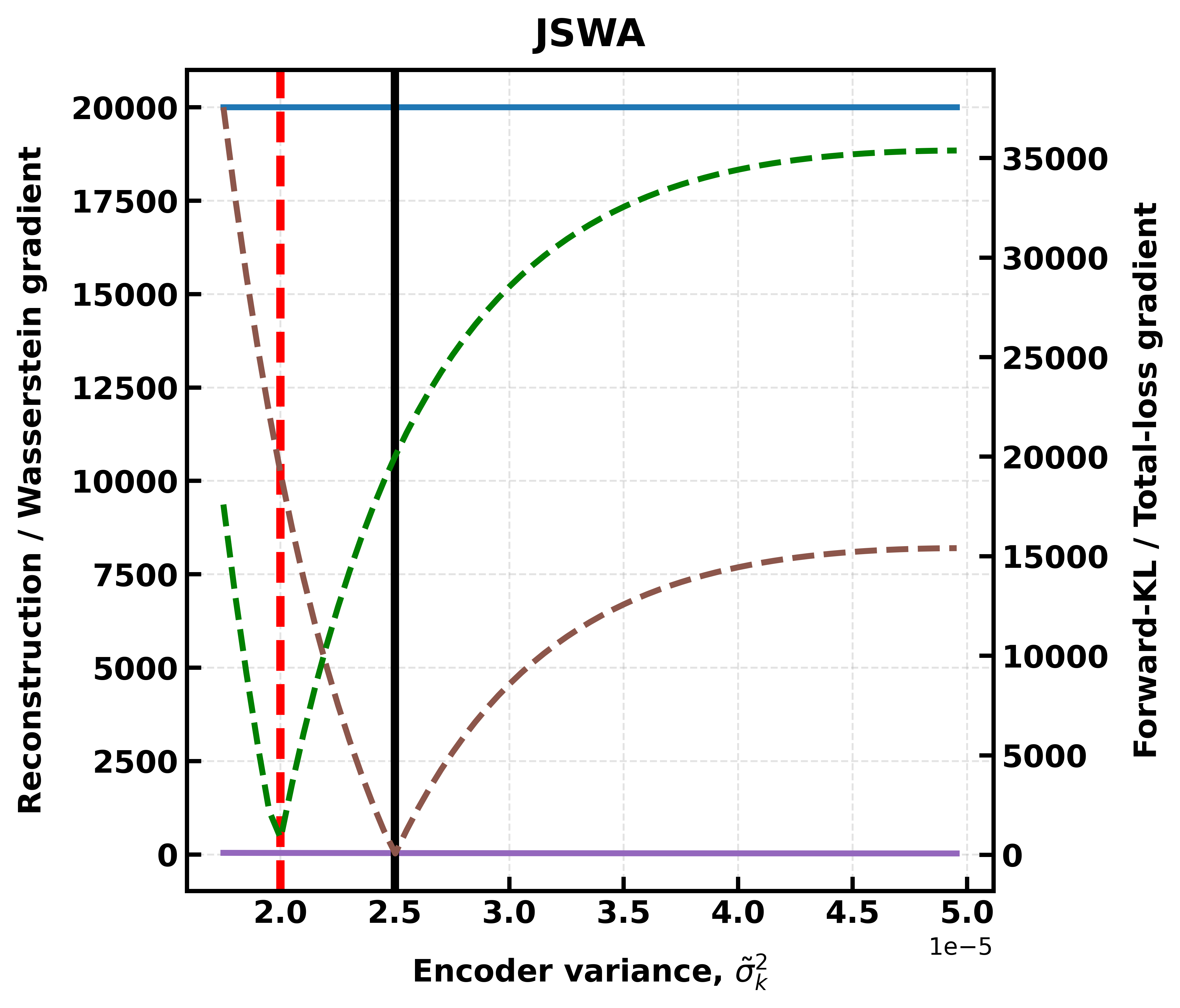}
        \caption{JSWA}
        \label{fig:gradient_equil_high_jswa}
    \end{subfigure}

    \vspace{0.05cm}

    \begin{subfigure}[t]{\textwidth}
        \centering
        \includegraphics[width=\linewidth]{figs/Test_case_0/gradient_legend.png}
        \caption*{}
    \end{subfigure}

    \caption{Gradient equilibrium for the strongest-information Fisher mode
    ($\lambda_k=40000$) of the ill-conditioned benchmark. The common legend
    identifies the reconstruction, reverse Kullback--Leibler,
    Wasserstein, forward Kullback--Leibler, and total-loss gradient
    contributions, together with the true and encoder variance locations.}

    \label{fig:gradient_equil_high}

\end{figure*}

At the opposite end of the Fisher spectrum, $\lambda_k=40000$, the reconstruction contribution is much larger and the exact posterior variance is only $2.5\times10^{-5}$ (\Cref{fig:gradient_equil_high}). The reverse-KL contribution remains sufficiently
steep to balance reconstruction for KL, while the forward-KL
contribution confines the JS and JSWA equilibria to a narrow posterior
neighbourhood. In JSWA, the Wasserstein contribution remains small
relative to reconstruction at this scale. Hence, differences in the learned variances of these strongly informed modes have a smaller influence on the global covariance error because their absolute posterior variances are negligible.

\paragraph{Well-conditioned case
($\kappa(\Sigma_{\mathrm{post}}^{*})\approx9$).}
\label{kappa_well}
The same linear Gaussian model is considered with the more compact with singular values $s = [1.5,1.25,1.0,0.75,0.5]$ giving 
Fisher spectrum as $\lambda_k^{\mathrm{Fisher}}
=\{900,\;625,\;400,\;225,\;100\}$. The corresponding exact posterior variances range from 0.0099 to 0.011 as described in Table~\ref{tab:prob0_well_spectrum}. Unlike the ill-conditioned case, posterior uncertainty is distributed
across several modes. The first mode contributes $51.3\%$ of the total
variance, while the first two and first three contribute $74.1\%$ and
$86.5\%$, respectively.

\begin{table}[ht]
\centering
\caption{Fisher information spectrum and posterior variance
distribution for the well-conditioned benchmark.}
\label{tab:prob0_well_spectrum}
\begin{tabular}{cccc}
\toprule
Mode $k$ & $\lambda_k^{\mathrm{Fisher}}$ &
$\tilde{\sigma}_k^{*2}$ & Variance contribution (\%) \\
\midrule
1 & 900 & $1.1\times10^{-3}$ & 5.70 \\
2 & 625 & $1.5\times10^{-3}$ & 7.77 \\
3 & 400 & $2.4\times10^{-3}$ & 12.44 \\
5 & 100 & $9.9\times10^{-3}$ & 51.30 \\
4 & 225 & $4.4\times10^{-3}$ & 22.80 \\
\bottomrule
\end{tabular}
\end{table}

In contrast to the ill-conditioned case, VAE-KL gives the lowest error
in all three aggregate metrics. The corresponding values are
$e_\mu=0.0221$, $e_{\mathrm{cov}}=0.092$, and $W_2=0.0263$, compared
with $(0.0249,0.117,0.0303)$ for JS and
$(0.0303,0.149,0.0351)$ for JSWA. Thus, when the posterior is well
conditioned, the standard KL objective provides the most accurate
global posterior approximation.

\begin{table}[ht]
\centering
\caption{Aggregate posterior accuracy for the well-conditioned linear
Gaussian benchmark.}
\label{tab:prob0_well_aggregate}
\begin{tabular}{lccc}
\toprule
Method & $e_\mu$ & $e_{\mathrm{cov}}$ & $W_2$ \\
\midrule
VAE-KL   & \textbf{0.0221} & \textbf{0.092} & \textbf{0.0263} \\
VAE-JS   & 0.0249 & 0.117 & 0.0303 \\
VAE-JSWA & 0.0303 & 0.149 & 0.0351 \\
\bottomrule
\end{tabular}
\end{table}

The CDF in \Cref{fig:cdf_well} confirms the global ranking, with
the KL distribution shifted toward smaller $W_2$ values. The median
separation is modest, with differences of approximately $0.0048$
between JSWA and JS and $0.0069$ between JSWA and KL. Hence, unlike
the pronounced separation observed for the ill-conditioned case, the
three objectives remain relatively close, although KL is consistently
more accurate.

\begin{figure}[ht]
\centering
\includegraphics[width=0.65\textwidth]{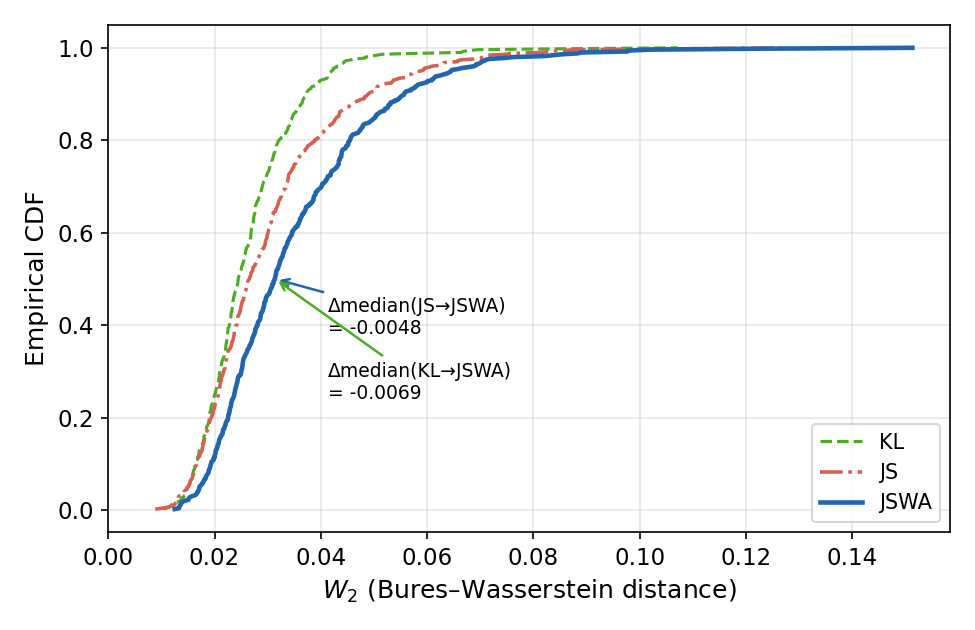}
\caption{Empirical CDF of the $2$-Wasserstein distance
$W_2(Q_\phi,P_{\mathrm{post}}^*)$ for the well-conditioned benchmark.}
\label{fig:cdf_well}
\end{figure}

The modal results in \Cref{tab:prob0_well_eigenmode} show that KL
remains close to the exact variance throughout the spectrum. Its errors
are only $0.40\%$ and $0.68\%$ in the first two modes, which together
contain $74.1\%$ of the total posterior variance, and the largest KL
error is $4.17\%$. JS and JSWA generally exhibit larger deviations.
Although JSWA achieves essentially zero error in the highest-information
mode, that mode contributes only $5.7\%$ of the total variance.
Therefore, accurate recovery of the uncertainty-dominant modes favours
KL and explains its lower global covariance error.

\begin{table}[ht]
\centering
\caption{Per-eigenmode posterior variance and relative variance error
(in parentheses) for the well-conditioned benchmark.}
\label{tab:prob0_well_eigenmode}
\begin{tabular}{cccccc}
\toprule
Mode $k$ & FIM $\lambda_k$ &
True $\tilde{\sigma}_k^{*2}$ &
VAE-JSWA & VAE-JS & VAE-KL \\
\midrule
1 & 900 &
$1.1\times10^{-3}$ &
$\mathbf{1.1\times10^{-3}}$ ($\mathbf{0.00\%}$) &
$1.3\times10^{-3}$ ($18.18\%$) &
$\mathbf{1.1\times10^{-3}}$ ($\mathbf{0.00\%}$) \\

2 & 625 &
$1.5\times10^{-3}$ &
$1.4\times10^{-3}$ ($6.67\%$) &
$1.7\times10^{-3}$ ($13.33\%$) &
$\mathbf{1.55\times10^{-3}}$ ($\mathbf{3.33\%}$) \\

3 & 400 &
$2.4\times10^{-3}$ &
$2.1\times10^{-3}$ ($12.50\%$) &
$2.65\times10^{-3}$ ($10.42\%$) &
$\mathbf{2.5\times10^{-3}}$ ($\mathbf{4.17\%}$) \\

4 & 225 &
$4.4\times10^{-3}$ &
$4.0\times10^{-3}$ ($9.09\%$) &
$4.8\times10^{-3}$ ($9.09\%$) &
$\mathbf{4.37\times10^{-3}}$ ($\mathbf{0.68\%}$) \\

5 & 100 &
$9.9\times10^{-3}$ &
$8.9\times10^{-3}$ ($10.10\%$) &
$1.08\times10^{-2}$ ($9.09\%$) &
$\mathbf{9.86\times10^{-3}}$ ($\mathbf{0.40\%}$) \\
\bottomrule
\end{tabular}
\end{table}

\begin{figure}[!htbp]
    \centering

    \begin{subfigure}[t]{0.45\textwidth}
        \centering
        \includegraphics[width=\linewidth]{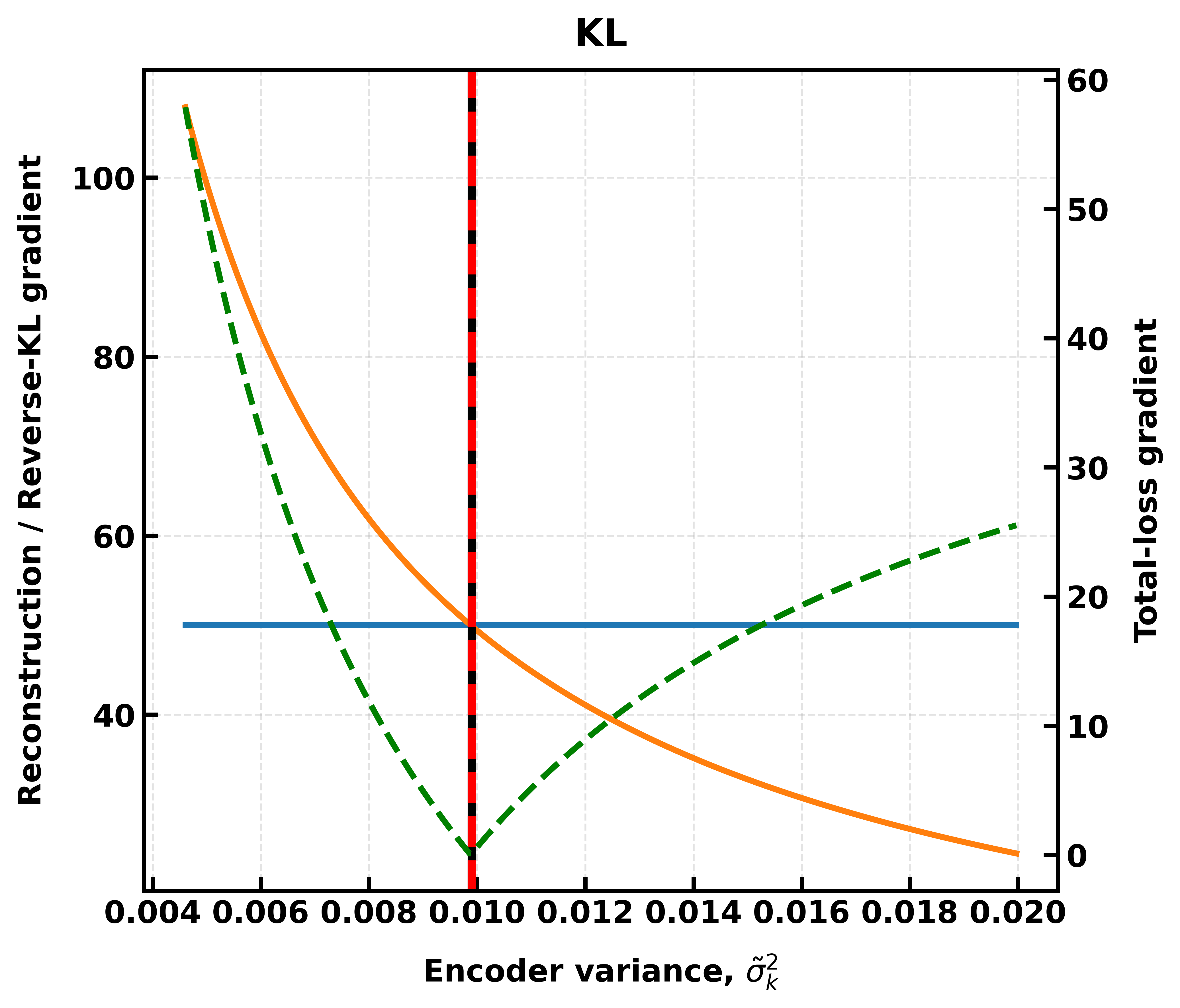}
        \caption{KL}
        \label{fig:gradient_equil_low_kl_well}
    \end{subfigure}
    \hfill
    \begin{subfigure}[t]{0.45\textwidth}
        \centering
        \includegraphics[width=\linewidth]{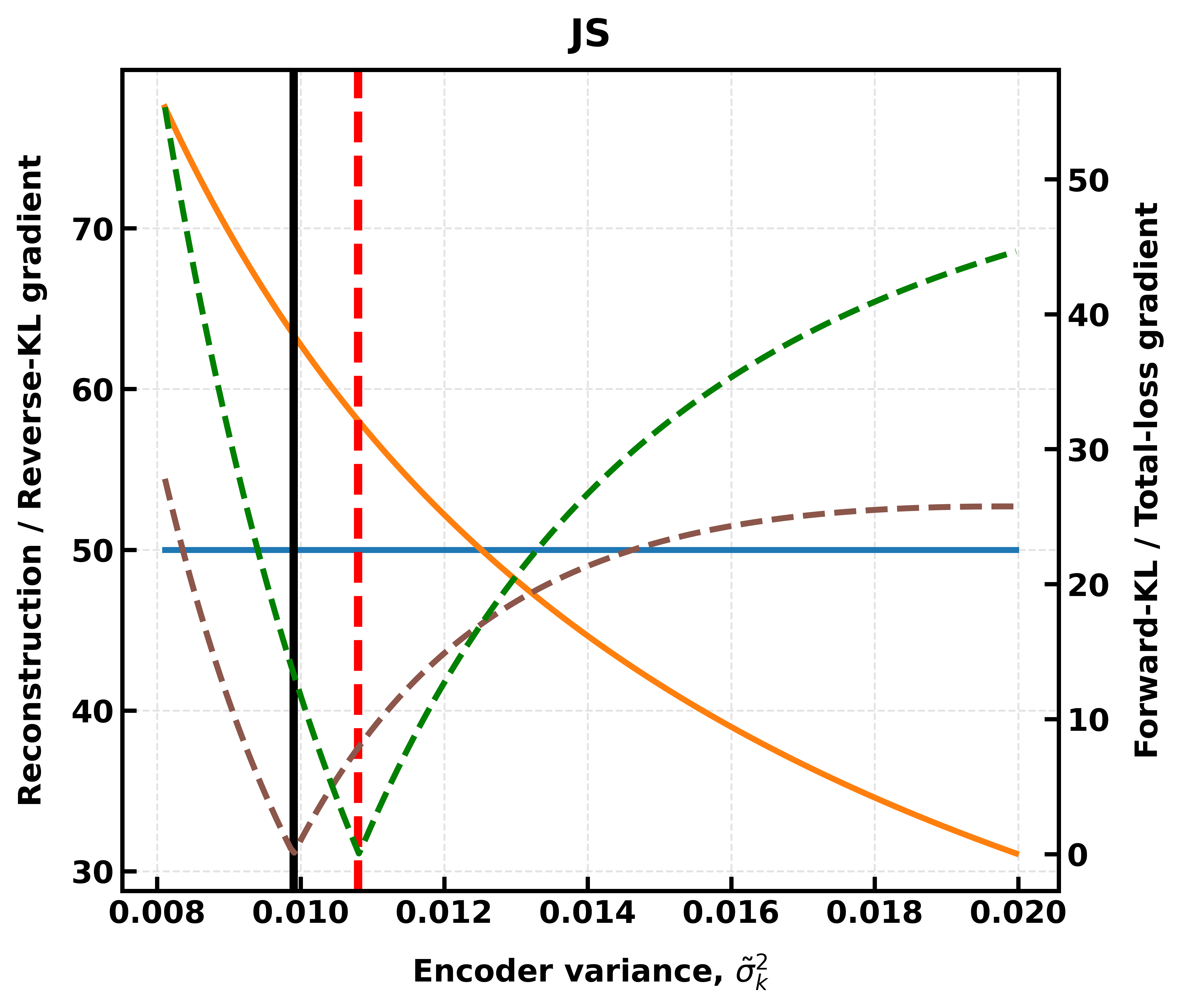}
        \caption{JS}
        \label{fig:gradient_equil_low_js_well}
    \end{subfigure}

    \vspace{0.005cm}

    \begin{subfigure}[t]{0.45\textwidth}
        \centering
        \includegraphics[width=\linewidth]{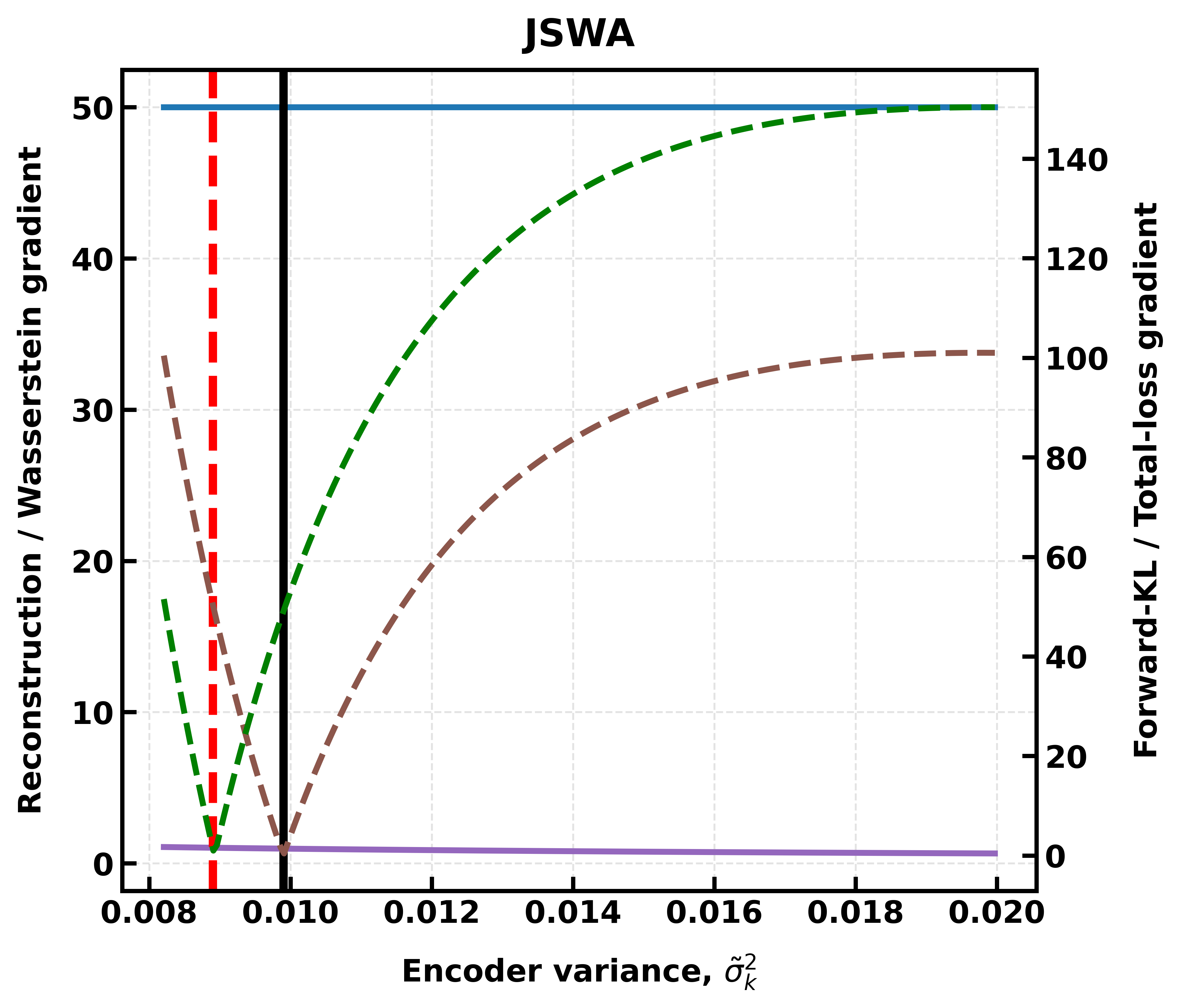}
        \caption{JSWA}
        \label{fig:gradient_equil_low_jswa_well}
    \end{subfigure}

    \vspace{0.001cm}

    \begin{subfigure}[t]{\textwidth}
        \centering
        \includegraphics[width=\linewidth]{figs/Test_case_0/gradient_legend.png}
        \caption*{}
    \end{subfigure}

    \caption{Gradient equilibrium for the lower-information Fisher mode
    ($\lambda_k=100$) of the well-conditioned benchmark.}

    \label{fig:gradient_equil_well_100}

\end{figure}

The gradient equilibrium is examined for $\lambda_k=100$ and
$\lambda_k=900$, representing the lower- and higher-information ends
of the well-conditioned spectrum. At $\lambda_k=100$,
$\tilde{\sigma}_k^{*2}\approx9.9\times10^{-3}$ and the KL
reconstruction--reverse-KL balance occurs essentially at the exact
variance, giving $\tilde{\sigma}_{\mathrm{KL}}^2\approx9.86\times10^{-3}$ (\Cref{fig:gradient_equil_well_100}).
In contrast, JS shifts the equilibrium above the target to
approximately $1.08\times10^{-2}$, whereas JSWA shifts it below the
target to approximately $8.9\times10^{-3}$. Thus, the additional terms
in JS and JSWA modify an equilibrium that is already accurately located
by KL.
\begin{figure*}[!t]
\centering
\begin{subfigure}[t]{0.45\textwidth}
\centering
\includegraphics[width=\linewidth]{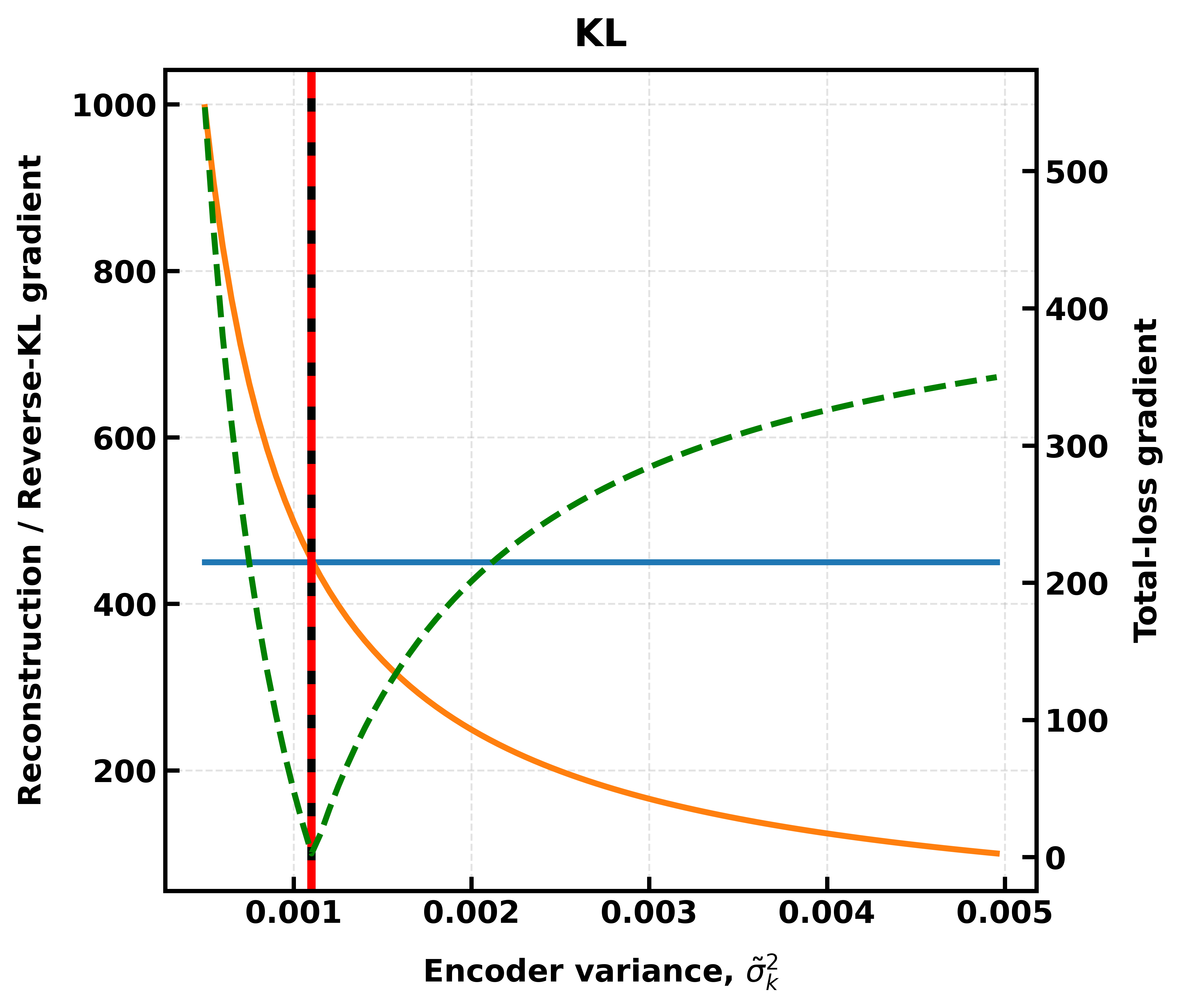}
\caption{KL}
\end{subfigure}
\hfill
\begin{subfigure}[t]{0.45\textwidth}
\centering
\includegraphics[width=\linewidth]{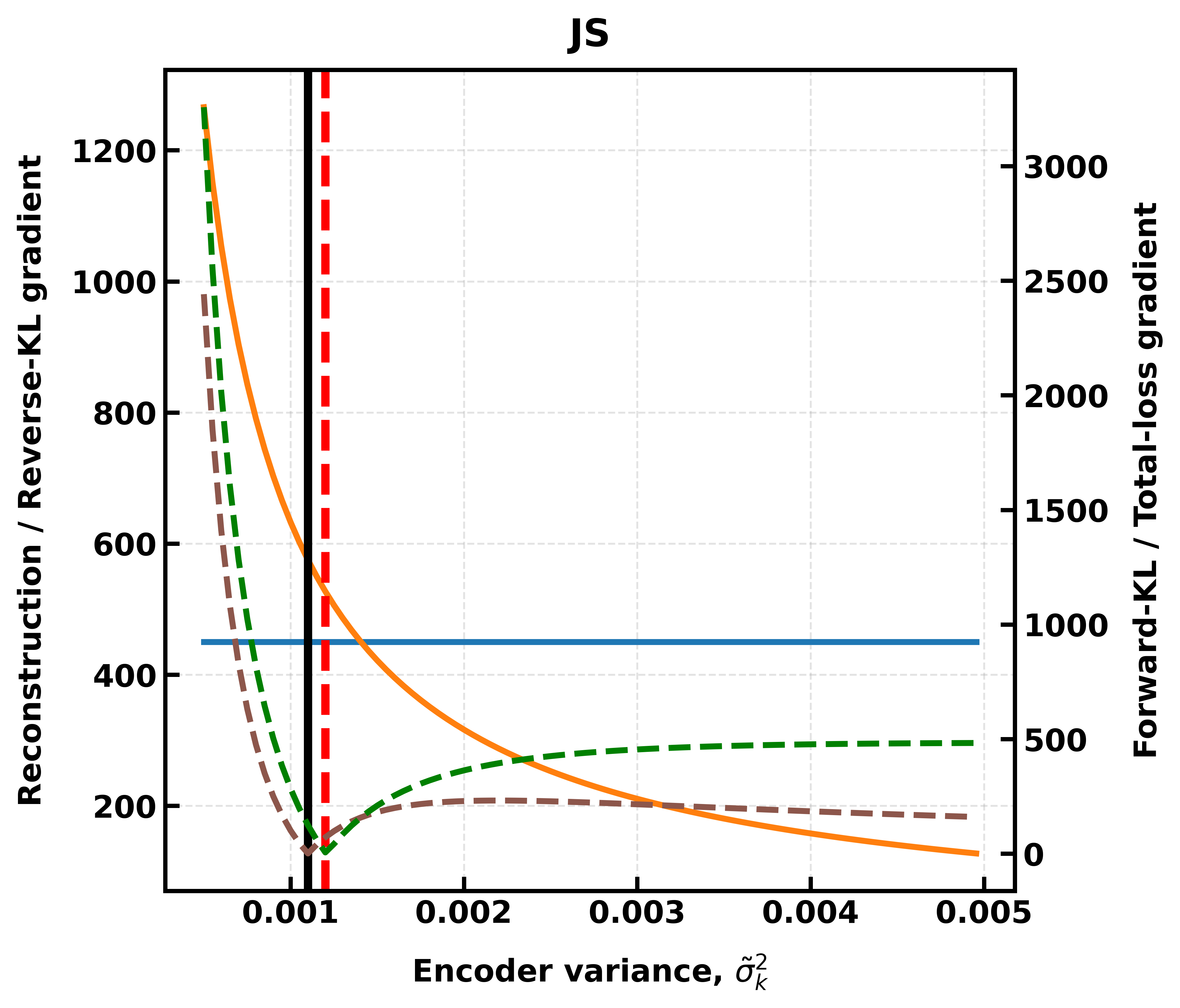}
\caption{JS}
\end{subfigure}

\vspace{0.05cm}

\begin{subfigure}[t]{0.45\textwidth}
\centering
\includegraphics[width=\linewidth]{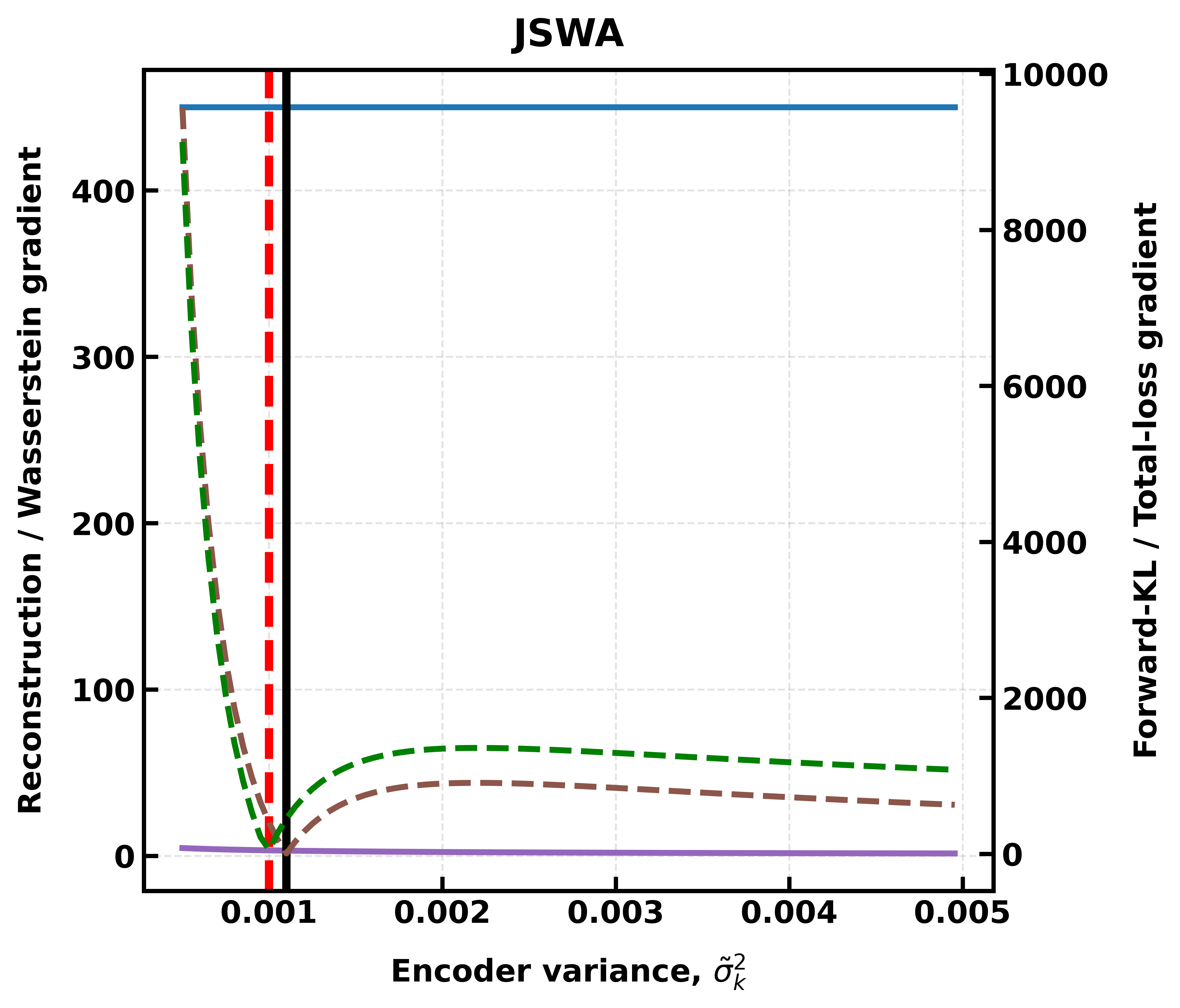}
\caption{JSWA}
\end{subfigure}

\vspace{0.05cm}

\begin{subfigure}[t]{\textwidth}
        \centering
        \includegraphics[width=\linewidth]{figs/Test_case_0/gradient_legend.png}
        \caption*{}
\end{subfigure}

\caption{Gradient equilibrium for the higher-information Fisher mode
($\lambda_k=900$) of the well-conditioned benchmark.}
\label{fig:gradient_equil_well_900}
\end{figure*}

At $\lambda_k=900$, the posterior variance decreases to approximately
$1.1\times10^{-3}$ and the reconstruction contribution becomes larger (\Cref{fig:gradient_equil_well_900}).
KL again equilibrates essentially at the exact variance. JS instead
overestimates the variance by approximately $18.18\%$, whereas JSWA
recovers the exact variance closely. The Wasserstein contribution
remains small compared with reconstruction, indicating that this
individual JSWA result arises from the combined gradient balance rather
than from a dominant Wasserstein contribution.

\paragraph{Overall interpretation.}

The two conditioning regimes reveal a clear dependence of objective
performance on the Fisher spectrum. In the ill-conditioned benchmark,
posterior uncertainty is concentrated in the weak-information modes,
where the reconstruction contribution is small and the KL equilibrium
exhibits a substantial variance bias. The JSWA formulation shifts this
equilibrium toward the exact variance in the dominant uncertainty
direction, leading to the most favourable global posterior accuracy.

In the well-conditioned benchmark, uncertainty is distributed across
several modes and the KL equilibrium is already close to the exact
posterior, particularly in the modes carrying most of the total
variance. The additional forward-KL and Wasserstein terms can improve
individual modes but do not provide a systematic global benefit.
Consequently, KL achieves the lowest mean, covariance, and
Wasserstein errors.

Taken together, these results show that the benefit of JSWA is
conditioning-dependent rather than universal: it becomes most
pronounced when weakly informed directions create a substantial
posterior-variance recovery challenge, whereas KL is already highly
accurate when the posterior is well conditioned.


\subsection{One-dimensional steady-state heat conduction}
\label{sec:tc1}
We next consider a one-dimensional steady-state heat-conduction inverse problem based on Smith~\cite{smith_uncertainty_2013} to investigate whether the conditioning-dependent behavior observed in the linear-Gaussian benchmark persists when the governing differential equation is linear, but the resulting parameter-to-observation map is nonlinear. The problem is deliberately kept low-dimensional and relatively well conditioned, providing a nonlinear counterpart to the well-conditioned case in \Cref{sec:prob0}.

\paragraph{Problem formulation.}

The one-dimensional steady-state heat conduction problem is defined on
$0<x<L$ and is governed by
\begin{equation}
\frac{d^2u}{dx^2}
=
\frac{2(a+b)}{ab}\frac{h}{k}
\left(u-u_\infty\right),
\qquad 0<x<L,
\label{eq:tc1:governing}
\end{equation}
subject to the boundary conditions
\begin{equation}
\frac{du}{dx}(0)
=
-\frac{\dot q}{k},
\qquad
\frac{du}{dx}(L)
=
-\frac{h}{k}
\left[u(L)-u_\infty\right],
\label{eq:tc1:bc}
\end{equation}
where $u$ denotes temperature, $k$ is the thermal conductivity,
$h$ is the convective heat-transfer coefficient, $u_\infty$ is the
ambient temperature, and $\dot q$ is the prescribed heat flux and $a$, $b$ are the cross-sectional dimensions of the fin.

The corresponding analytical solution can be written as
\begin{equation}
u(x)
=
c_1 e^{-\gamma x}
+
c_2 e^{\gamma x}
+
u_\infty
\label{eq:tc1:analytical}
\end{equation}

where $\gamma = \sqrt{2(a+b)h/(abk)}$,
$c_2 = -\dot{q}/(k\gamma) + c_1$, and
\begin{align}
    c_1 = \frac{\dot{q}}{k\gamma}
          \!\left(
            \frac{e^{\gamma L}(h + k\gamma)}
                 {e^{-\gamma L}(h - k\gamma)
                  + e^{\gamma L}(h + k\gamma)}
          \right)\!.
    \label{eq:tc1:c1}
\end{align}

Introducing $X = x/L$, $U = (u-u_\infty)/[q_0/(k\gamma_0)]$,
$q' = \dot{q}/q_0$, and the Biot number
$\mathrm{Bi} = hL/k$, and writing
$\gamma L = \sqrt{d_1\,\mathrm{Bi}}$ with $d_1 = 2(a+b)L/(ab)$,
the nondimensional temperature field is
\begin{align}
    \Aboxed{
    U(X) = C_1\,e^{-\sqrt{d_1\,\mathrm{Bi}}\,X}
         + C_2\,e^{+\sqrt{d_1\,\mathrm{Bi}}\,X}
    =: f(\mathrm{Bi},\,q'),}
    \label{eq:tc1:nondim}
\end{align}
where $C_1$ and $C_2$ are determined by the boundary conditions. The system response is fully determined by the Biot number and nondimensional heat flux, which constitute the parameters of interest. The parameter vector $\mathbf{q}$ in \Cref{eq:forward_model} is defined as

\begin{equation}
q=
\begin{bmatrix}
\mathrm{Bi} & q'
\end{bmatrix}^{T},
\end{equation}
where
\[
\mathrm{Bi}=\frac{hL}{k},
\qquad
q'=\frac{\dot q}{q_0},
\]

The physical parameters are fixed at $a=b=\SI{0.95}{cm}$, $L=\SI{70}{cm}$, $k=\SI{2.37}{W.cm^{-1}.K^{-1}}$, $h=\SI{1.91e-3}{W.cm^{-2}.K^{-1}}$, $u_\infty=\SI{21.29}{K}$, and $\dot{q}=\SI{18.41}{W.cm^{-2}}$. Observations are generated using the analytical forward model $\mathcal{F}$ at $n_{\mathrm{sensor}}=15$ equally spaced locations, with $x_0=\SI{10}{cm}$ and $\Delta x=\SI{4}{cm}$. These locations are selected to capture variations in both the temperature magnitude and curvature. Independent Gaussian noise, $\eta\sim\mathcal{N}(0,\sigma_\eta^2 I)$, with $\sigma_\eta=0.05$, is added to each observation. The prior distribution for the parameters of interest is specified as $P_{\mathrm{pr}}(\mathrm{q})=\mathcal{N}(\mathrm{\mu}_{\mathrm{pr}},\mathrm{\Sigma}_{\mathrm{pr}})$, where $\mathrm{\mu}_{\mathrm{pr}}=(0.04,0.65)^\top$ and $\mathrm{\Sigma}_{\mathrm{pr}}=\operatorname{diag}(0.2,0.05)$. The prior is chosen to encompass the physically plausible parameter range.

It must be noted that the governing heat-conduction
equation is linear in the temperature field, the resulting
parameter-to-observation map is nonlinear in $q$, making this a nonlinear
inverse problem.

\paragraph{Surrogate accuracy and reference posterior.}
The VAE formulations are trained following the procedure described previously
in \Cref{sec:training_procedure}. The dataset contains 1000 observations,
of which 800 are used for training. The decoder $\Psi_d$ is pre-trained using
the $N_{\mathrm{train}}=800$ parameter--solution pairs
\[
\left\{(\mathrm{Bi}^{(i)},q'^{(i)},u^{(i)})\right\}_{i=1}^{N_{\mathrm{train}}},
\]
sampled uniformly over the prescribed parameter ranges. Prior to training,
the parameter vector is transformed to log-space as
\[
q=
\begin{bmatrix}
\log\mathrm{Bi} & \log q'
\end{bmatrix}^{T}.
\]
The validation set is used to monitor surrogate accuracy through
\cref{eq:surrogate_rel_error}, and training is terminated once the
surrogate error satisfies $e_{\mathrm{surr}}<\SI{1}{\percent}$. The resulting
relative errors are $\SI{0.171}{\percent}$ on the validation set and
$\SI{0.161}{\percent}$ on the test set.

The resulting surrogate posterior is then evaluated against independently generated MCMC samples for each test observation, using the convergence criteria specified in Table~\ref{tab:app}. The converged MCMC posterior is taken as the reference distribution for all subsequent comparisons of posterior accuracy.

In contrast to the linear-Gaussian benchmark, the nonlinear forward
operator does not admit a single Fisher-information basis that remains
representative across the entire test ensemble. Because the local
sensitivity of the forward map depends on the parameter state, the
corresponding Fisher information and posterior geometry vary from one test
case to another. Consequently, a global Fisher-mode decomposition is not
appropriate here; posterior conditioning is therefore characterized
directly from the covariance of each MCMC reference posterior. Although the governing heat-conduction
equation is linear in the temperature field, the resulting
parameter-to-observation map is nonlinear in $q$, making this a nonlinear
inverse problem.

\paragraph{Posterior conditioning and distributional accuracy.}
The posterior conditioning of the one-dimensional heat-conduction
inverse problem is first characterized directly from the MCMC reference
posteriors. As discussed common Fisher-mode decomposition is not used for the
PDE test cases therefore the condition number of each reference posterior covariance is evaluated directly from the MCMC samples as
\begin{equation}
\kappa_i
=
\frac{
\lambda_{\max}
\left(\Sigma_{\mathrm{MCMC},i}\right)
}{
\lambda_{\min}
\left(\Sigma_{\mathrm{MCMC},i}\right)
},
\label{eq:tc1:condition}
\end{equation}
where $\Sigma_{\mathrm{MCMC},i}$ denotes the posterior covariance for
test case $i$.

For the heat-conduction problem, the posterior condition numbers range
from $48.3$ to $68.6$, with a median of $55.7$ as shown in \Cref{fig:tc1:condition_number}. These values are substantially smaller than those
of the ill-conditioned benchmarks considered later and indicate a
relatively moderate degree of posterior anisotropy across the test
ensemble.

\begin{figure}[h]
    \centering
    \begin{subfigure}[b]{0.48\textwidth}
        \centering
        \includegraphics[width=\textwidth]
        {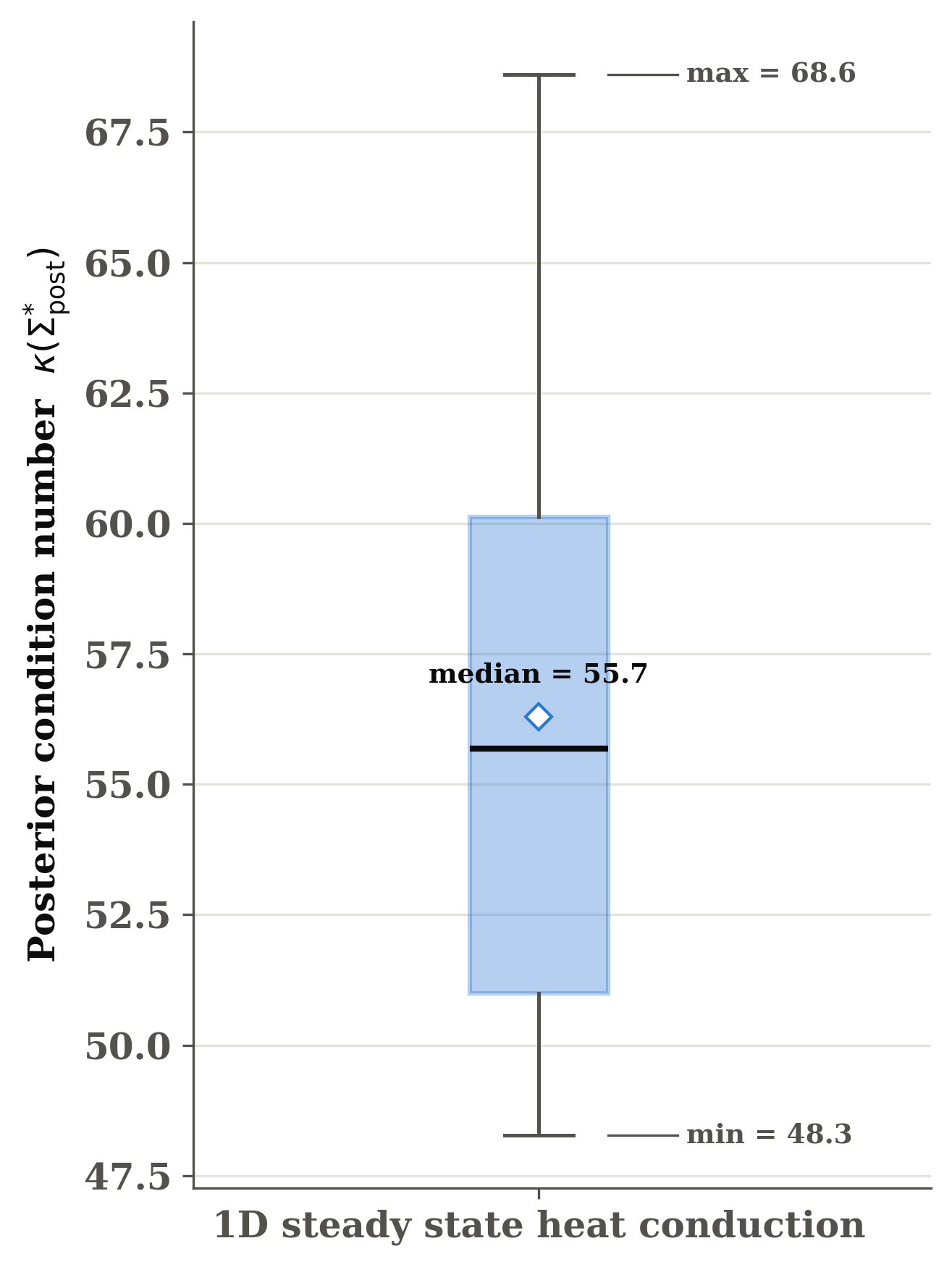}
        \caption{Posterior condition number.}
        \label{fig:tc1:condition_number}
    \end{subfigure}
    \hfill
    \begin{subfigure}[b]{0.48\textwidth}
        \centering
        \includegraphics[width=\textwidth]
        {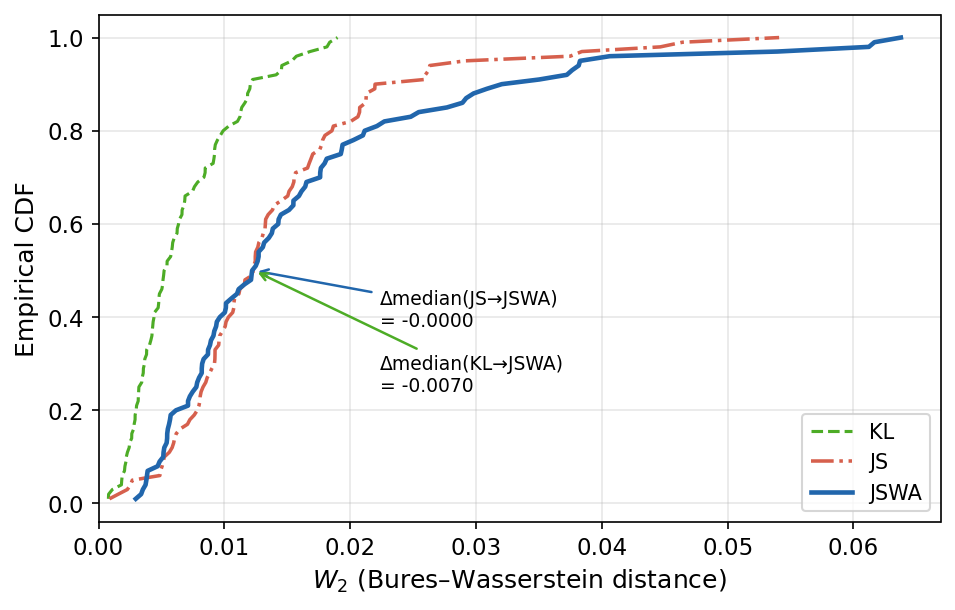}
        \caption{CDF of the $W_2$-Wasserstein distance.}
        \label{fig:tc1:cdf}
    \end{subfigure}
    
    \caption{Posterior conditioning and distributional posterior accuracy
    for the one-dimensional steady-state heat-conduction inverse problem.
    The posterior condition number is computed from the MCMC reference
    covariance for each test case, while the empirical CDF shows the
    distribution of the $2$-Wasserstein distance between the learned and
    MCMC reference posteriors.}
    \label{fig:tc1:conditioning_accuracy}
\end{figure}

The distributional posterior accuracy is quantified using the
$2$-Wasserstein distance
\begin{equation}
W_{2,i}
=
W_2
\left(
Q_{\phi,i},
P_{\mathrm{MCMC},i}
\right),
\label{eq:tc1:w2}
\end{equation}
where $Q_{\phi,i}$ denotes the VAE posterior and
$P_{\mathrm{MCMC},i}$ the corresponding reference posterior.

The empirical CDF (\Cref{fig:tc1:cdf}) reveals a clear distinction between the three objectives. VAE-KL exhibits the leftmost CDF, corresponding to the smallest $W_2$ values and hence the most accurate distributional posterior approximation for this test problem. In contrast, the CDFs of VAE-JS and VAE-JSWA are shifted toward larger $W_2$ values and are nearly coincident over most of the distribution, indicating very similar posterior accuracy for the two formulations. At the median, the difference between JS and JSWA is essentially negligible, $\Delta \operatorname{median} \left(W_2^{\mathrm{JS}} \rightarrow W_2^{\mathrm{JSWA}}\right) \approx 0$, whereas the median $W_2$ for VAE-KL is approximately $0.007$ smaller than that of VAE-JSWA across the test ensemble. These results are consistent with the well-conditioned linear-Gaussian benchmark with $\kappa(\Sigma_{\mathrm{post}}^{*})\approx9$, where VAE-KL likewise exhibited lower distributional posterior error than VAE-JS and VAE-JSWA. The consistency across the two problems indicates that, when the posterior is relatively well conditioned, the additional regularization terms in JS and JSWA do not provide a systematic distributional-accuracy advantage over the simpler KL formulation.

\paragraph{Representative posterior reconstruction.}

To complement the ensemble-level comparison, a representative test
case is selected to examine the posterior reconstruction at the
individual-sample level. \Cref{fig:tc1:representative1} compares
the marginal posterior distributions of $\mathrm{Bi}$ and $q_p$
obtained using VAE-KL, VAE-JS, and VAE-JSWA with the corresponding MCMC
reference, while the quantitative errors are summarized in \Cref{tab:tc1_representative_error}.

\begin{figure}[h]
\centering
\includegraphics[width=\textwidth]
{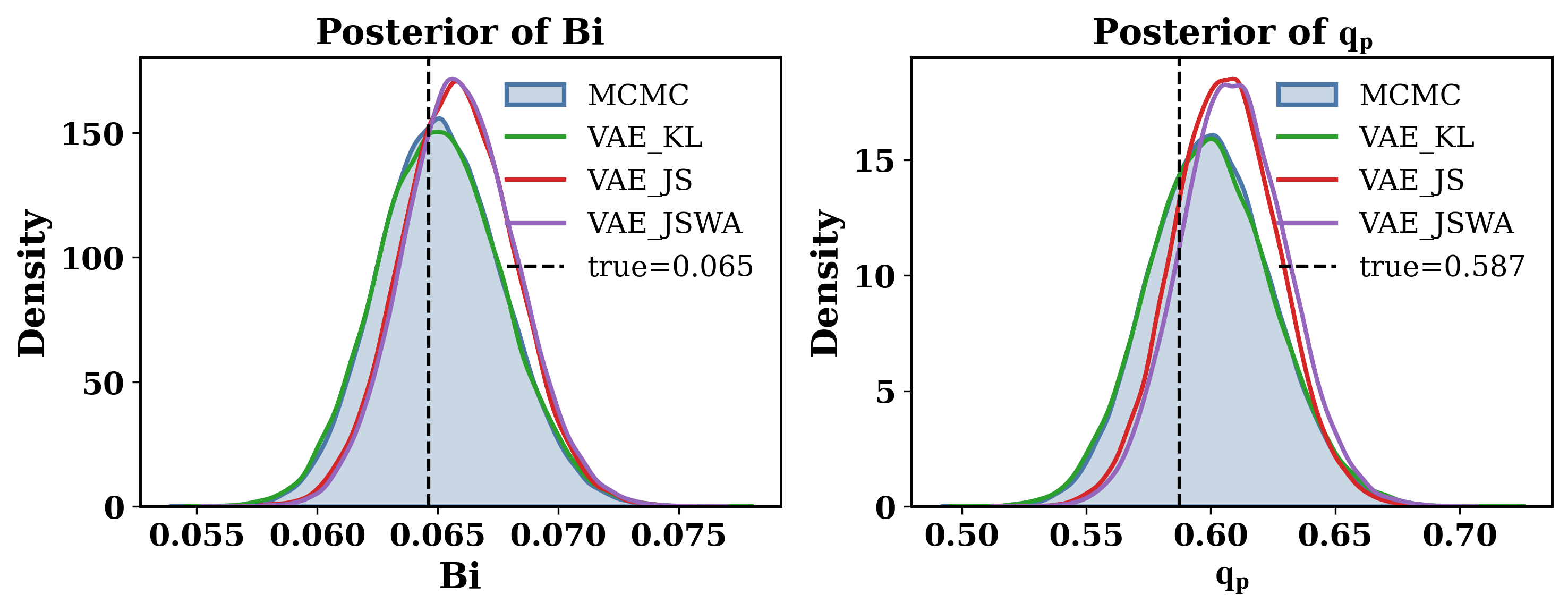}
\caption{Representative posterior reconstruction for the
one-dimensional heat-conduction inverse problem. The marginal
posteriors obtained using VAE-KL, VAE-JS, and VAE-JSWA are compared
with the corresponding MCMC reference.}
\label{fig:tc1:representative1}
\end{figure}

All three formulations reproduce the main location and shape of the
reference marginals reasonably well. VAE-KL gives the smallest relative
mean error, $e_\mu=0.309\%$, followed closely by VAE-JS with
$e_\mu=0.347\%$, while VAE-JSWA gives a larger error of $0.710\%$.
The covariance is recovered most accurately by VAE-KL, with a relative
error of only $4.486\%$, compared with $23.827\%$ for VAE-JS and
$9.595\%$ for VAE-JSWA.

The same trend is reflected in the distributional metric. VAE-KL
achieves the smallest $W_2$ distance, $0.00212$, whereas VAE-JS and
VAE-JSWA yield $0.00360$ and $0.00436$, respectively. Thus, although
the three methods provide visually similar posterior locations, their
differences become more apparent when the posterior covariance and
overall distributional discrepancy are quantified.

\begin{table}[h]
\centering
\caption{Representative-posterior errors relative to the MCMC
reference for the one-dimensional heat-conduction inverse problem.}
\label{tab:tc1_representative_error}
\begin{tabular}{lccc}
\toprule
Method & $e_\mu$ (\%) & $e_{\mathrm{cov}}$ (\%) & $W_2$ \\
\midrule
VAE-KL   & \textbf{0.309} & \textbf{4.486} & \textbf{0.00212} \\
VAE-JS   & 0.347 & 23.827 & 0.00360 \\
VAE-JSWA & 0.710 & 9.595 & 0.00436 \\
\bottomrule
\end{tabular}
\end{table}

Overall, the representative case is consistent with the ensemble-level
result shown in \Cref{fig:tc1:cdf}. VAE-KL provides the most accurate
posterior reconstruction in terms of covariance and $W_2$, while VAE-JS
achieves comparable mean accuracy. VAE-JSWA remains competitive in
covariance reconstruction but does not provide an advantage over KL for
this representative case. Together with the ensemble-level comparison,
these results reinforce the observation that, for this relatively
well-conditioned nonlinear inverse problem, the additional regularization
in VAE-JS and VAE-JSWA does not provide a systematic distributional
accuracy advantage over VAE-KL, consistent with the well-conditioned
linear-Gaussian benchmark.


\subsection{Two-dimensional steady-state heat conduction with spatially varying conductivity}
\label{sec:tc2}

We next consider a two-dimensional steady-state heat-conduction inverse
problem with spatially varying conductivity. This problem provides a
PDE-based validation of the ill-conditioned regime identified
in the linear-Gaussian benchmark. The conductivity field is represented
using a truncated Karhunen--Lo\`eve expansion with nine latent
parameters, and the forward problem is solved using the finite-element
formulation described in \Cref{app:fem_kl}. Temperature observations are
collected at 30 interior sensor locations. This nine-dimensional
field-inversion problem therefore provides a substantially more
challenging posterior reconstruction task than the one-dimensional
heat-conduction problem considered in \Cref{sec:tc1}.

In the general forward model of \Cref{eq:forward_model}, the unknown
parameter vector $q$ is given by
\begin{equation}
    q =
    [z_1,z_2,z_3,z_4,z_5,z_6,z_7,z_8,z_9]^\top,
    \qquad
    z_i\sim\mathcal{N}(0,1),
\end{equation}
where $z_i$ are the coefficients of the truncated Karhunen--Lo\`eve
expansion described in \Cref{app:fem_kl}. The forward operator $\mathcal{F}$ maps these parameters to the temperature observations. The temperature field is governed by
\begin{align}
    -\nabla\cdot\left(
        k(\mathbf{x};q)\nabla T(\mathbf{x})
    \right)
    &=0,
    \qquad \mathbf{x}\in\Omega,
    \qquad \Omega=[0,1]^2,
\label{eq:tc2:gov}
\end{align}
where $k(\mathbf{x};q)$ denotes the spatially varying conductivity represented
using the truncated KL expansion described in \Cref{app:fem_kl}. The left
and right boundaries are prescribed with $T=0$ and $T=1$, respectively, while
the top and bottom boundaries are insulated.

The finite-element forward model is discretized on a structured
$50\times50$ mesh, and the temperature field is sampled at 30 interior
sensor locations arranged on a $6\times5$ grid, with
$x\in\{0.10,0.26,0.42,0.58,0.74,0.90\}$ and
$y\in\{0.10,0.30,0.50,0.70,0.90\}$. The resulting temperature measurements
constitute the observation vector. A dataset containing 5000 parameter--solution pairs is generated using this forward model. Following the common training protocol described in \Cref{sec:training_procedure}, the decoder $\Psi_d$ is pre-trained on the 4000 training pairs, while the validation set is used to monitor surrogate accuracy
through \Cref{eq:surrogate_rel_error}. The resulting surrogate achieves
relative errors of approximately $\SI{0.43}{\percent}$ on both the
validation and test sets.

\paragraph{Posterior conditioning and distributional accuracy:}

Unlike the linear-Gaussian benchmark, the posterior covariance in this
nonlinear PDE-constrained inverse problem is evaluated directly from the
reference MCMC posterior for each test sample. Across the test ensemble,
the posterior condition number, evaluated according to
\cref{eq:tc1:condition}, exhibits substantial variability. The median
condition number is approximately $\kappa(\Sigma_{\mathrm{post}}^{*})=1929.3$,
with observed values ranging from approximately $1779.7$ to $2253.9$
(\Cref{fig:tc2:condition_number}). These large condition numbers indicate
substantial anisotropy in the posterior uncertainty across parameter
directions, providing a stringent test of the three variational
objectives.

\begin{figure}[h]
    \centering
    \begin{subfigure}[b]{0.48\textwidth}
        \centering
        \includegraphics[width=\textwidth]{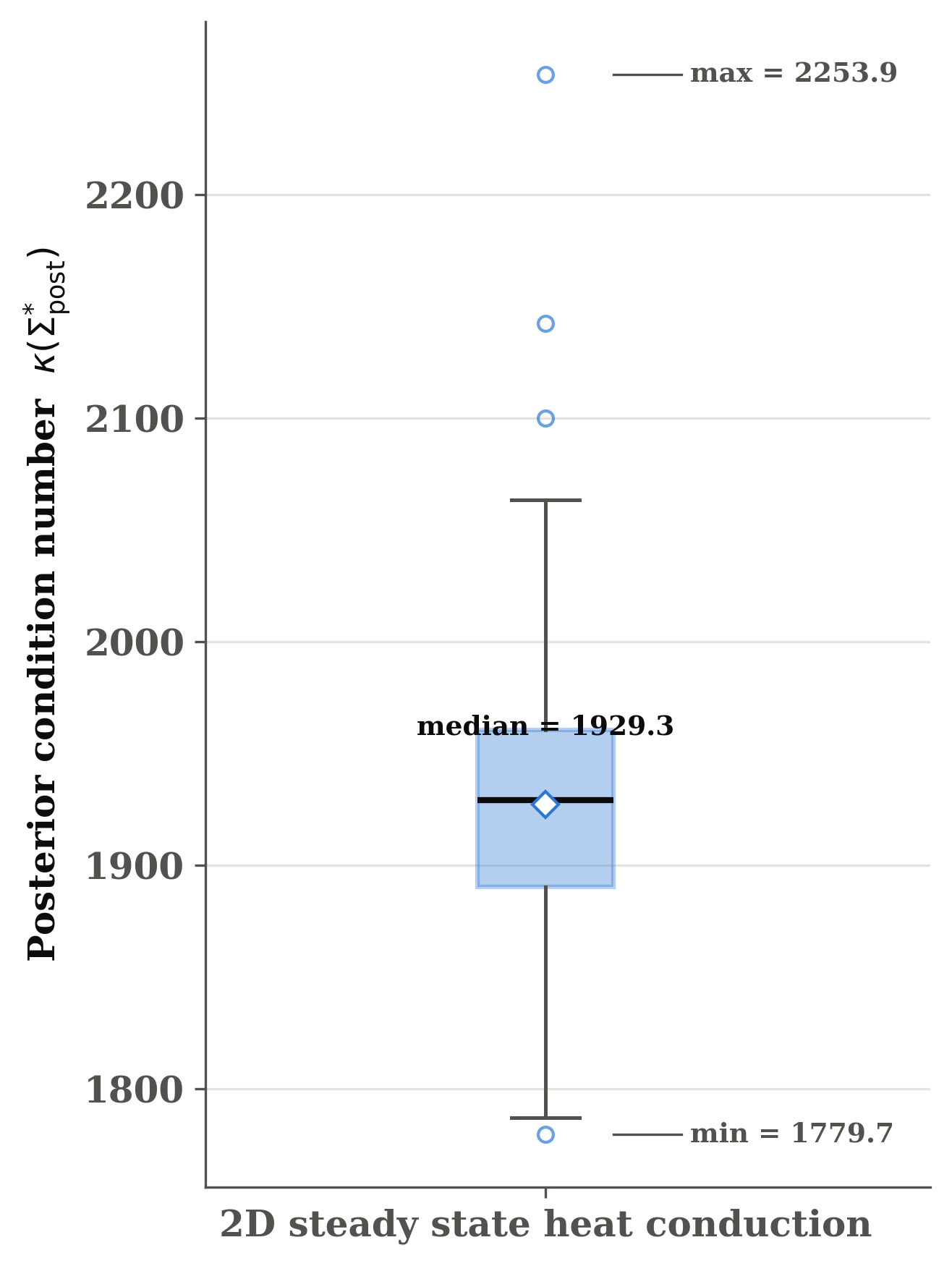}
        \caption{Posterior condition number.}
        \label{fig:tc2:condition_number}
    \end{subfigure}
    \hfill
    \begin{subfigure}[b]{0.48\textwidth}
        \centering
        \includegraphics[width=\textwidth]{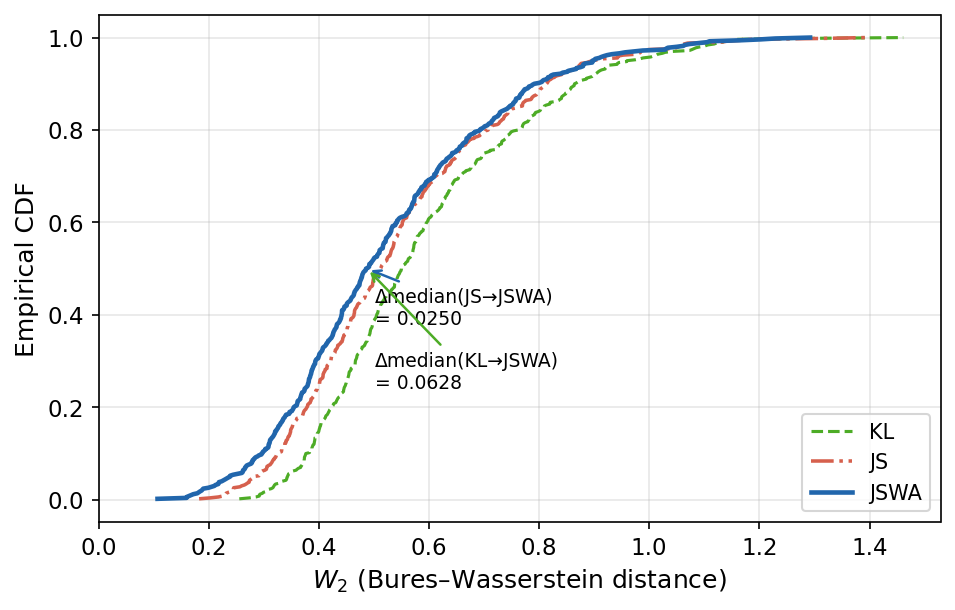}
        \caption{CDF of the $W_2$-Wasserstein distance.}
        \label{fig:tc2:cdf}
    \end{subfigure}
    
    \caption{Posterior conditioning and distributional posterior accuracy
    for the one-dimensional steady-state heat-conduction inverse problem.
    The posterior condition number is computed from the MCMC reference
    covariance for each test case, while the empirical CDF shows the
    distribution of the $2$-Wasserstein distance between the learned and
    MCMC reference posteriors.}
    \label{fig:tc2:conditioning_accuracy}
\end{figure}

The ensemble-level distributional accuracy is assessed using the empirical
CDF of the $2$-Wasserstein distance between the VAE posterior and the
corresponding MCMC reference posterior (\Cref{fig:tc2:cdf}).The Wasserstein distance is computed according to \Cref{eq:tc1:w2}. The CDFs show
a clear separation between the three formulations, with VAE-JSWA shifted
toward smaller $W_2$ values over most of the distribution. VAE-JS also
outperforms VAE-KL, although the difference between JS and JSWA is smaller
than that between KL and JSWA. At the median, VAE-JSWA reduces $W_2$ by
$0.0250$ relative to VAE-JS and by $0.0628$ relative to VAE-KL. Thus,
unlike the one-dimensional heat-conduction problem, where VAE-KL achieved
the smallest median $W_2$, the present two-dimensional problem exhibits
a clear advantage for VAE-JSWA in distributional posterior accuracy.

\paragraph{Representative posterior reconstruction:}

To complement the ensemble-level comparison, a representative test
case is selected to examine the reconstruction of the spatially varying
conductivity posterior at the individual-sample level. Figure~\ref{fig:k_sample415}
compares the posterior mean conductivity field and posterior
standard-deviation field obtained from VAE-KL, VAE-JS, and VAE-JSWA
with the corresponding MCMC reference. The middle row additionally
shows the pointwise absolute difference between each VAE posterior-mean
conductivity field and the MCMC posterior-mean conductivity field. The corresponding quantitative posterior
errors are summarized in \Cref{tab:tc2_representative_error}.

\begin{table}[h]
\centering
\caption{Posterior reconstruction errors for the representative test case
relative to the MCMC reference for the two-dimensional heat-conduction
inverse problem.}
\label{tab:tc2_representative_error}
\begin{tabular}{lccc}
\toprule
Method & $e_\mu$ (\%) & $e_{\mathrm{cov}}$ (\%) & $W_2$ \\
\midrule
VAE-KL   & 15.09 & 108.088 & 0.6577 \\
VAE-JS   & 11.11 & 45.762 & 0.4173 \\
VAE-JSWA & \textbf{3.68} & \textbf{24.114} & \textbf{0.2111} \\
\bottomrule
\end{tabular}
\end{table}

\begin{figure}[ht]
\centering
\includegraphics[width=\textwidth]
{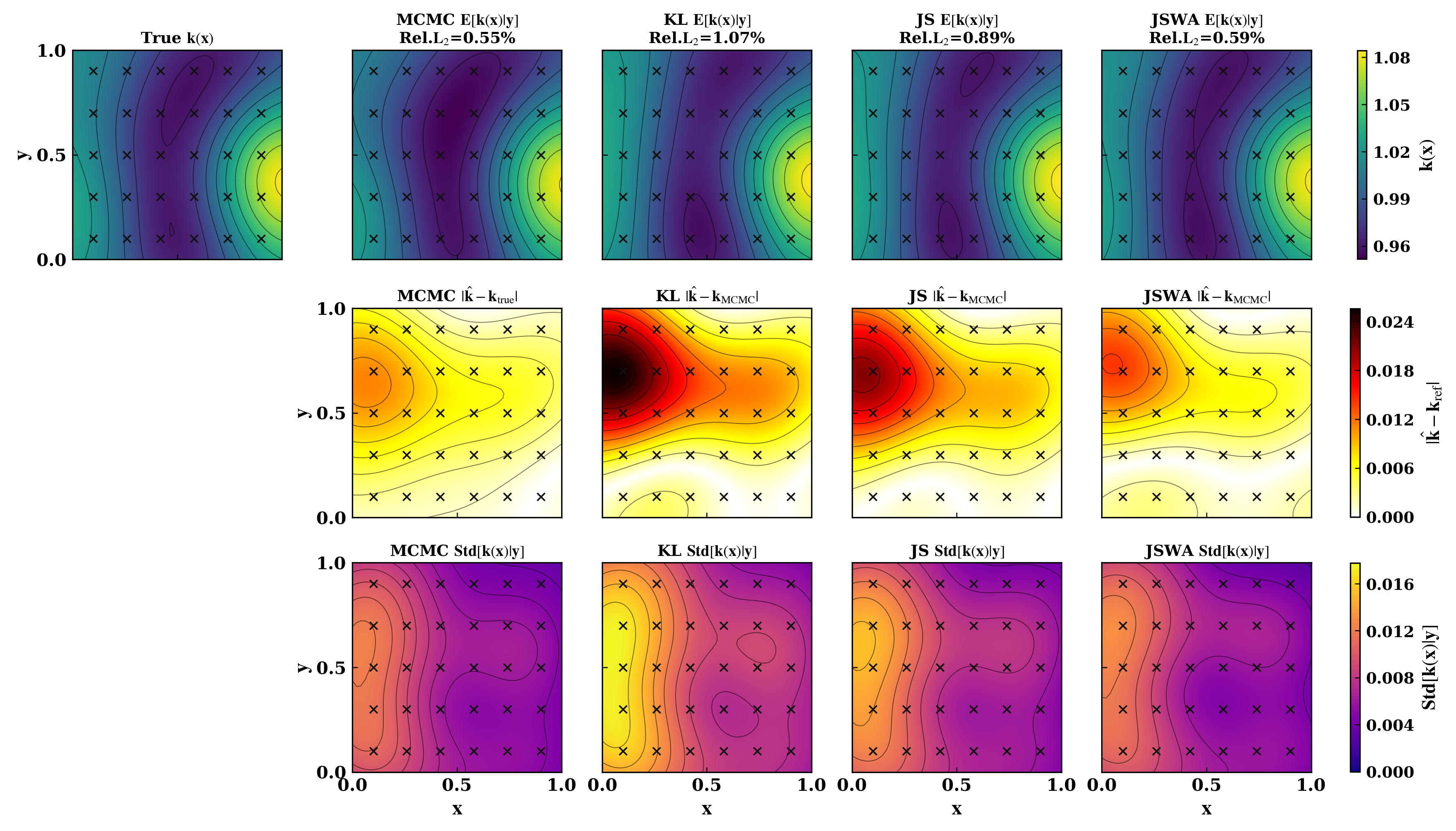}
\caption{Representative posterior field reconstruction for the
two-dimensional heat-conduction inverse problem. The posterior
conductivity fields reconstructed using VAE-KL, VAE-JS, and VAE-JSWA
are compared with the corresponding MCMC reference.}
\label{fig:k_sample415}
\end{figure}

All three formulations reproduce the dominant spatial structure of the
reference posterior mean, but their quantitative accuracy differs
substantially. As shown in \Cref{tab:tc2_representative_error},
VAE-JSWA provides the most accurate reconstruction across all three
posterior metrics, with a relative mean error of $3.68\%$, a covariance
error of $24.114\%$, and a $W_2$ distance of $0.2111$. VAE-JS provides
intermediate accuracy, with corresponding errors of $11.11\%$,
$45.762\%$, and $0.4173$, while VAE-KL gives the largest errors of
$15.09\%$, $108.088\%$, and $0.6577$, respectively.

The field-level reconstruction provides a consistent spatial
interpretation of these differences. The relative $L_2$ errors of the
posterior mean conductivity field with respect to the MCMC reference
are $1.07\%$, $0.89\%$, and $0.59\%$ for VAE-KL, VAE-JS, and VAE-JSWA,
respectively. Thus, VAE-JSWA gives the closest posterior-mean field to
the MCMC reference. The pointwise error fields further show reduced
deviations for JSWA, particularly in regions where the conductivity
varies more strongly. The posterior standard-deviation fields also
show that all three formulations recover the dominant spatial pattern
of posterior uncertainty, with JSWA showing the closest visual
agreement with the MCMC uncertainty structure.

Overall, the representative case is consistent with the
ensemble-level distributional results shown in \Cref{fig:tc2:cdf} and
supports the trend identified in the ill-conditioned linear-Gaussian
benchmark, where VAE-JSWA provided the most accurate posterior
reconstruction among the three formulations. VAE-JSWA achieves the
smallest errors in the posterior mean, covariance, and $W_2$ distance,
while VAE-JS gives intermediate performance and VAE-KL exhibits the
largest discrepancy. Together, the ensemble-level and
representative-case results indicate that the advantage of VAE-JSWA is
reflected not only in the overall posterior distributional accuracy but
also in the reconstructed spatial structure and uncertainty of the
conductivity field.

\subsection{Two-dimensional nonlinear steady-state heat conduction with
joint parameter inference}
\label{sec:tc3}
The fourth test case considers a two-dimensional nonlinear steady-state
heat-conduction inverse problem with temperature-dependent thermal
conductivity. In addition to the KL coefficients describing the spatially
varying baseline conductivity, the temperature-dependence parameter
$\alpha$ is inferred jointly. The governing equation is
\begin{align}
    -\nabla\cdot\left(k(T,\mathbf{x}; q)\nabla T(\mathbf{x})\right)
    &=\beta,
    \qquad \mathbf{x}\in\Omega,
    \qquad \Omega=[0,1]^2,
\end{align}
where $\beta=1$. The left and right boundaries are prescribed with
$T=0$ and $T=1$, respectively, while the top
and bottom boundaries are insulated. The temperature-dependent
conductivity is given by
\begin{align}
    k(T,\mathbf{x};q)
    &= k_0(\mathbf{x};\mathbf z)
    \left(1+\alpha T(\mathbf{x})\right),
\end{align}
where $k_0(\mathbf{x};\mathbf z)$ denotes the spatially varying baseline
conductivity represented using the truncated KL expansion described in
\Cref{app:fem_kl}. The unknown parameter vector is
\begin{align}
    q
    &=
    [z_1,z_2,z_3,\alpha]^\top,
\end{align}
with $z_i\sim\mathcal N(0,1)$ and
$\alpha\sim\mathcal N(0,0.09)$.  In the general formulation of
\Cref{eq:forward_model}, this parameter vector defines the forward
operator $\mathcal{F}$ through the solution of the governing
PDE followed by the observation operator. Thus, $\mathcal{F}$ represents
the parameter-to-observation map for this inverse problem. Three KL modes are retained, capturing $65\%$ of the cumulative variance. This truncation was adopted as a
practical compromise between representing the dominant spatial
variability of the conductivity field and controlling the computational
cost of the nonlinear forward solves and MCMC reference calculations.
The parameter range is chosen such that
$1+\alpha T(\mathbf{x})>0$ throughout the domain, ensuring positive
thermal conductivity. 

The finite-element forward model is discretized on a structured $50\times50$ mesh, and the temperature field is sampled at 30 sensor locations arranged on a $6\times5$ interior grid, with $x\in\{0.10,0.26,0.42,0.58,0.74,0.90\}$ and $y\in\{0.10,0.30,0.50,0.70,0.90\}$. The resulting temperature measurements constitute the observation vector. A dataset containing
5000 parameter--solution pairs is generated using this forward model.
The decoder $\Psi_d$ is pre-trained using the 4000 training
parameter--solution pairs, while the validation set is used to monitor
surrogate accuracy through \cref{eq:surrogate_rel_error}. The resulting
surrogate achieves relative errors of approximately
$\SI{0.59}{\percent}$ on both the validation and test sets.

\paragraph{Posterior conditioning and distributional accuracy.}
Since the forward operator in this problem is nonlinear with respect to the unknown parameter vector, as in the preceding nonlinear test cases, a single global Fisher basis does not provide a common coordinate system in which the posterior variances of all test samples can be directly compared.
The posterior conditioning is therefore characterized directly from the MCMC reference covariance for each test sample, using the condition number defined in \Cref{eq:tc1:condition}. \Cref{fig:tc3:condition_number} summarizes the resulting posterior conditioning across the test ensemble.
The posterior condition number exhibits substantial variability, with a
median value of approximately $11526.3$ and values ranging from
$1939.5$ to $328168.8$. This wide range indicates substantial
sample-to-sample variation in posterior anisotropy, with some
realizations exhibiting considerably stronger ill-conditioning than
others.

\begin{figure}[h]
    \centering
    \begin{subfigure}[b]{0.48\textwidth}
        \centering
        \includegraphics[width=\textwidth]{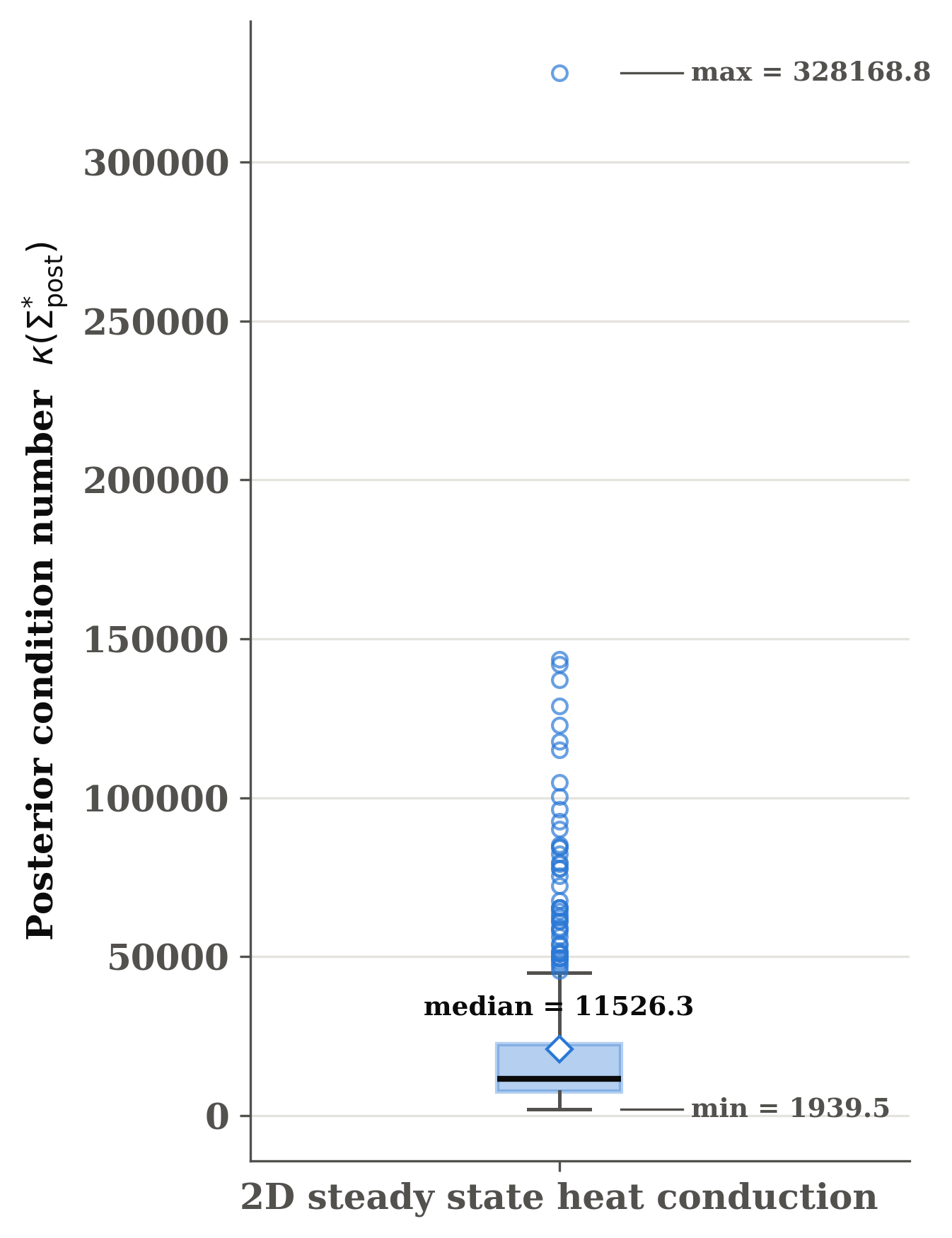}
        \caption{Posterior condition number.}
        \label{fig:tc3:condition_number}
    \end{subfigure}
    \hfill
    \begin{subfigure}[b]{0.48\textwidth}
        \centering
        \includegraphics[width=\textwidth]{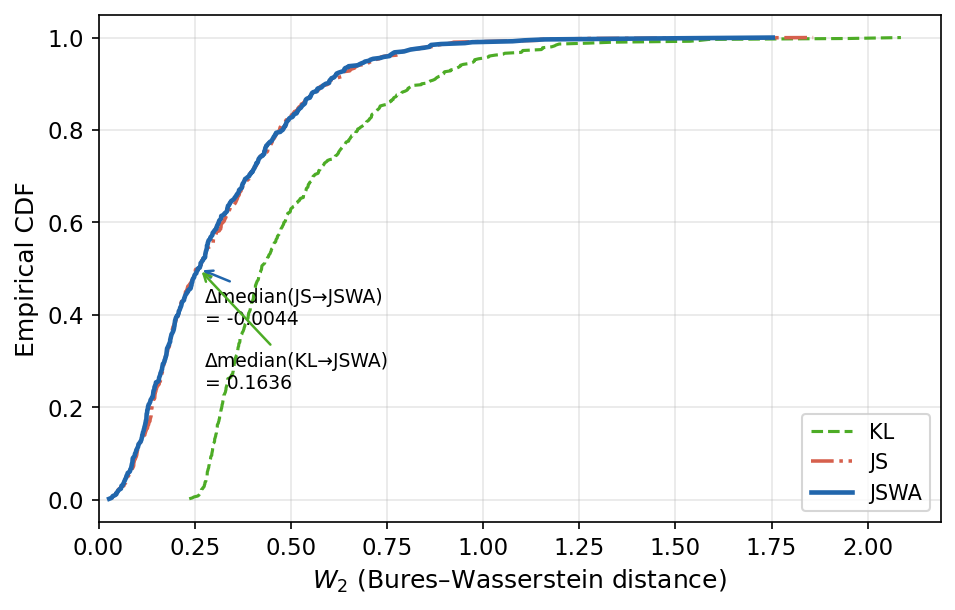}
        \caption{CDF of the $W_2$-Wasserstein distance.}
        \label{fig:tc3:cdf}
    \end{subfigure}
    
    \caption{Posterior conditioning and distributional posterior accuracy
    for the two-dimensional nonlinear steady-state heat-conduction
    inverse problem. The posterior condition number is computed from the
    MCMC reference covariance for each test case, while the empirical CDF
    shows the distribution of the $2$-Wasserstein distance between the
    learned and MCMC reference posteriors.}
    \label{fig:tc3:conditioning_accuracy}
\end{figure}

The distributional accuracy of the three VAE formulations is assessed
using the $W_2$ distance between the variational posterior and the
corresponding MCMC reference posterior. The empirical CDFs in \Cref{fig:tc3:cdf} show a clear advantage of the JS and JSWA objectives over VAE-KL. The CDF of VAE-KL is shifted substantially toward larger $W_2$ values, indicating systematically larger distributional discrepancies from the MCMC reference. In contrast, the CDFs of VAE-JS and VAE-JSWA are closely aligned over most of the distribution, with VAE-JSWA providing a small but consistent improvement. At the median, the reduction in $W_2$ from VAE-KL to VAE-JSWA is approximately $0.1636$, whereas the corresponding difference between VAE-JS and VAE-JSWA is only $0.0044$. At the ensemble level, the results provide clear evidence that reverse-KL regularization becomes increasingly inadequate for strongly ill-conditioned posteriors. Thus, the principal improvement over VAE-KL is associated with the forward-KL contribution, while the Wasserstein regularization provides a smaller
additional improvement at the ensemble level.

\paragraph{Representative posterior reconstruction.}

To complement the ensemble-level comparison, a
representative test case is selected to examine the reconstruction of
the spatially varying conductivity posterior at the individual-sample
level. \Cref{fig:k_sample120} compares the posterior mean
conductivity field and posterior standard-deviation field obtained from
VAE-KL, VAE-JS, and VAE-JSWA with the corresponding MCMC reference.
The middle row additionally shows the pointwise absolute difference between each VAE posterior-mean conductivity field and the MCMC posterior-mean
conductivity field. The corresponding quantitative posterior errors are
summarized in \Cref{tab:tc3_representative_error}.

\begin{table}[h]
\centering
\caption{Posterior reconstruction errors for the representative test
    case relative to the MCMC reference for the two-dimensional nonlinear
    heat-conduction inverse problem.}
\label{tab:tc3_representative_error}
\begin{tabular}{lccc}
\toprule
Method & $e_\mu$ (\%) & $e_{\mathrm{cov}}$ (\%) & $W_2$ \\
\midrule
VAE-KL   & 20.85 & 177.04 & 0.3307 \\
VAE-JS   & 7.84 & 6.43 & 0.0671 \\
VAE-JSWA & \textbf{3.52} & \textbf{3.91} & \textbf{0.0359} \\
\bottomrule
\end{tabular}
\end{table}

\begin{figure}[ht]
\centering
\includegraphics[width=\textwidth]{%
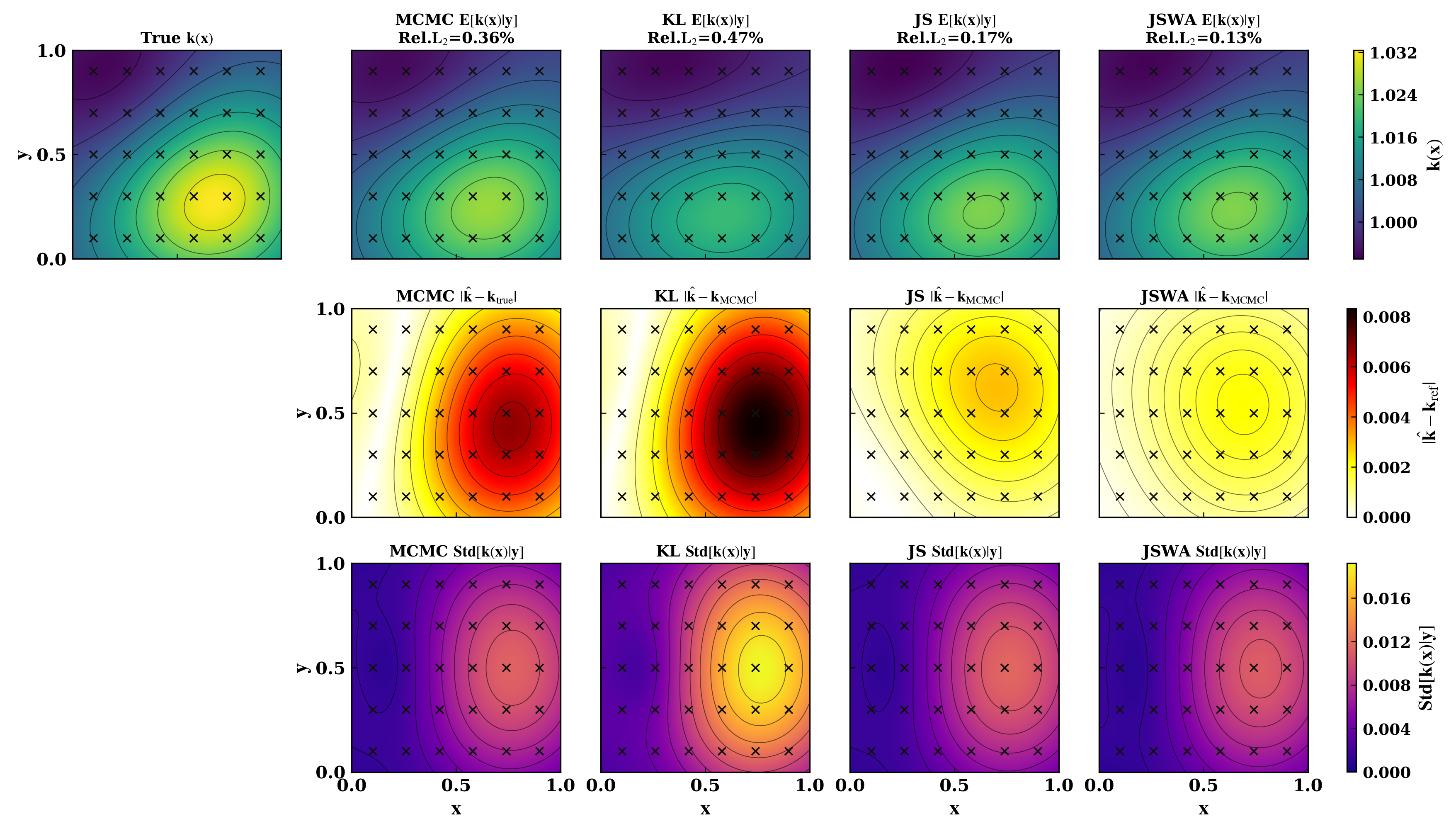}
\caption{Representative posterior conductivity-field reconstruction
for the two-dimensional nonlinear heat-conduction inverse problem.
The posterior mean and posterior standard-deviation fields obtained
using VAE-KL, VAE-JS, and VAE-JSWA are compared with the corresponding
MCMC reference.}
\label{fig:k_sample120}
\end{figure}

All three formulations reproduce the dominant spatial structure of the
MCMC posterior mean conductivity field, but VAE-JS and VAE-JSWA provide
substantially closer agreement than VAE-KL. The relative $L_2$ errors of the posterior-mean conductivity field are $0.47\%$, $0.17\%$, and $0.13\%$ for VAE-KL, VAE-JS, and
VAE-JSWA, respectively. Therefore, VAE-JSWA yields the lowest field-level error and reconstructs a posterior-mean conductivity field that most closely matches the MCMC result. The pointwise error fields further show reduced spatial discrepancies for
the JS-based formulations, with VAE-JSWA providing the closest visual
agreement with the MCMC reference. The posterior standard-deviation
fields also recover the dominant spatial pattern of uncertainty, with
VAE-JS and VAE-JSWA showing closer correspondence with the MCMC
reference.

The posterior-level metrics in \Cref{tab:tc3_representative_error}
reinforce the field-level comparison. VAE-KL exhibits the largest posterior mean, covariance, and $W_2$
errors, whereas VAE-JS substantially reduces all three. VAE-JSWA
provides a further reduction, achieving errors of $3.52\%$, $3.91\%$,
and $0.0359$ for the posterior mean, covariance, and $W_2$ distance,
respectively. For this representative realization, the improvement
from JS to JSWA is more pronounced than at the ensemble median, where
the corresponding difference in $W_2$ is only $0.0044$.

Overall, this fourth test case extends the numerical comparison to a
nonlinear PDE with both spatially varying and temperature-dependent
conductivity, together with joint inference of field and scalar
parameters. VAE-JSWA provides the most accurate posterior reconstruction
for the representative case, while the ensemble-level results show a
substantial improvement of both JS-based formulations over VAE-KL.

\section{Conclusions and future work}
\label{sec:conclusions}

This paper investigated how the choice of variational regularization affects
amortized Bayesian inference for inverse problems, with particular emphasis
on weakly identified parameter directions. We introduced a
Jensen--Shannon--Wasserstein VAE (JSWA-VAE) that retains the forward-KL
data-coverage term of the asymmetric Jensen--Shannon formulation while
replacing its reverse-KL prior regularizer with the squared $2$-Wasserstein
distance. A full-covariance Gaussian encoder is used to represent
parameter correlations and anisotropic posterior uncertainty. The results
show that the relative performance of the three objectives is not universal,
but depends strongly on the conditioning of the posterior.

The theoretical analysis provides a mechanistic explanation for this
dependence through the whitened generalized Fisher basis. Under the local
linear--Gaussian approximation, the encoder variance in each Fisher mode is
determined by the balance between the reconstruction gradient, the
forward-KL posterior-correcting term, and the prior regularization term.
The reverse-KL and squared Wasserstein regularizers exhibit different
variance-dependent gradient scalings: for modal variance $v_k$, the
reverse-KL gradient scales as $v_k^{-1}$ as $v_k\rightarrow0$, whereas the
Wasserstein gradient scales as $v_k^{-1/2}$. This distinction becomes most
important in weakly informed directions, where the reconstruction gradient
is small, and becomes less consequential in strongly informed directions,
where the data-driven contribution is larger. The analysis is further
supported by a surrogate-induced posterior error bound that separates
surrogate and variational approximation errors, and by Gaussian posterior
bounds relating $W_2$ to posterior mean and covariance discrepancies.
Together, these results provide a geometric interpretation of why the
choice of regularizer can affect posterior approximation differently across
posterior information spectra.

The numerical experiments support this conditioning-dependent behavior. The
linear--Gaussian benchmarks provide a controlled setting in which the
effect of posterior conditioning can be isolated. In the well-conditioned
case, with $\kappa(\Sigma_{\mathrm{post}}^*)\approx9$, VAE-KL achieves the
lowest posterior mean, covariance, and Wasserstein errors, indicating that
the additional regularization terms do not provide a systematic global
accuracy advantage when the posterior is already well conditioned. In the
strongly ill-conditioned case, with
$\kappa(\Sigma_{\mathrm{post}}^*)\approx8000$, posterior uncertainty is
strongly concentrated in weakly informed directions, with the weakest
Fisher mode accounting for approximately $87\%$ of the total posterior
variance. In this regime, VAE-JSWA achieves the most accurate posterior
reconstruction, with lower posterior mean, covariance, and $W_2$ errors
than both VAE-KL and VAE-JS. These results are consistent with the modal
gradient analysis, which predicts that differences between the
regularizers become more consequential when weakly informed directions
carry a substantial fraction of the posterior uncertainty.

The three physics-based inverse problems extend the comparison beyond the
controlled linear--Gaussian setting. The one-dimensional heat-conduction
problem considers a nonlinear inverse problem governed by a linear ODE and
has a relatively moderate posterior covariance condition number, with a
median of approximately $\kappa(\Sigma_{\mathrm{post}}^*)=55.7$. In this
regime, VAE-KL provides the most accurate distributional posterior
approximation at the ensemble level. The two-dimensional conductivity
inversion considers a nonlinear inverse problem governed by a linear PDE
and exhibits substantially stronger posterior anisotropy, with a median
posterior covariance condition number of approximately $1929.3$. In this
case, VAE-JSWA provides the lowest ensemble-level $W_2$ error. The final
two-dimensional problem further introduces nonlinear parameter dependence
through the temperature-dependent conductivity and joint parameter
inference, resulting in a nonlinear PDE-constrained inverse problem. In
this case, both JS-based formulations substantially outperform VAE-KL, with
VAE-JSWA providing the best overall performance. Taken together, these
results support the interpretation that the benefit of the additional
regularization becomes more pronounced as posterior uncertainty becomes
increasingly concentrated in weakly informed directions. They therefore
support posterior conditioning as an important factor governing the
relative effectiveness of the three variational objectives, rather than
suggesting that JSWA is universally preferable to KL.

The nonlinear ODE- and PDE-based problems should nevertheless be interpreted
with appropriate caution. Unlike the linear--Gaussian benchmarks, these
problems do not admit a common global Fisher basis in which posterior
conditioning can be systematically controlled and compared across all test
cases. The results therefore provide evidence that the
conditioning-dependent behavior identified in the controlled analysis can
persist in nonlinear physics-based inverse problems, but do not establish a
quantitative conditioning-scaling relationship for general nonlinear
forward models.

Several limitations remain. The quantitative conditioning-scaling evidence
is established only in the controlled linear--Gaussian setting, and the
predictive capability of the Fisher-based analysis for strongly nonlinear
posterior geometries has not yet been established systematically. The
regularization weights $\beta_1$ and $\beta_2$ were selected using
Optuna-based optimization against a held-out MCMC reference posterior.
Although effective, this procedure requires an MCMC reference during model
development, and a principled strategy for selecting these weights without
such a reference remains an open problem. The encoder posterior is also
restricted to the Gaussian family; while a full-covariance representation
captures parameter correlations and anisotropic uncertainty, it cannot
represent genuinely multimodal posteriors. Finally, the numerical
experiments are limited to moderate-dimensional inverse problems with
$d\leq9$.

Future work will first focus on establishing the conditioning dependence
quantitatively for nonlinear forward models through controlled families of
inverse problems with systematically varying posterior information spectra.
A natural extension is to develop an adaptive VAE in which the relative
contributions of the KL, forward-KL, and Wasserstein terms are adjusted
according to the inferred posterior geometry, rather than using fixed
regularization weights for all observations and problems. Such a framework
could use the estimated Fisher-information spectrum or related measures of
posterior anisotropy to determine when additional Wasserstein regularization
is warranted. Developing this strategy without requiring an MCMC reference
during model development is an important open problem. On the computational
side, developing more efficient and accurate physics surrogates would
reduce the cost of VAE training and facilitate application to more
expensive and higher-dimensional forward models. Finally, extending the
encoder to richer posterior families, including mixture models ~\citep{bishop_pattern_2006} and
normalizing flows ~\citep{rezende_variational_2015}, would enable genuinely multimodal inverse problems to be
considered.

Taken together, the results suggest that posterior conditioning can serve
as a useful diagnostic for selecting variational objectives in amortized
Bayesian inverse inference. In particular, a linearized estimate of the
posterior information spectrum can indicate whether uncertainty is
concentrated in weakly informed directions, where additional Wasserstein
regularization is more likely to provide a benefit. This observation
suggests a broader direction beyond selecting a single fixed objective:
the variational objective itself could be adapted to the inferred posterior
geometry. The Fisher-basis perspective therefore provides both a
mechanistic interpretation of the observed differences between KL, JS, and
JSWA and a potential foundation for geometry-adaptive variational
inference.

\section*{Acknowledgements}
Arijit Hazra would like to acknowledge funding support from the Ramanujan Fellowship (RJF/2022/000046) administered by Anusandhan National Research Foundation (ANRF), India.

\appendix

\section{Finite Element Formulation}
\label{app:fem}
The test problems consider the parametric elliptic heat-conduction equation
\begin{align}
    -\nabla\cdot
    \left(
        k(T,\mathbf{x};q)\nabla T(\mathbf{x})
    \right)
    &=
    f(\mathbf{x}),
    \qquad \mathbf{x}\in\Omega,
    \label{eq:fem_governing}
\end{align}
where $T(\mathbf{x})$ is the temperature, $k$ is the thermal
conductivity, $q$ is the unknown parameter vector, and
$f(\mathbf{x})$ is the prescribed volumetric heat source. The
temperature-dependent conductivity cases are nonlinear through the
dependence of $k$ on $T$.

The boundary is partitioned into Dirichlet and Neumann portions,
$\partial\Omega=\Gamma_D\cup\Gamma_N$, with
\begin{align}
    T &= T_D,
    \qquad \text{on }\Gamma_D,
    \label{eq:fem_dirichlet}\\
    -k(T,\mathbf{x};q)\nabla T\cdot\mathbf n
    &=g_N,
    \qquad \text{on }\Gamma_N.
    \label{eq:fem_neumann}
\end{align}
In the problems considered here, the left and right boundaries have
prescribed temperatures, while the top and bottom boundaries are
insulated, corresponding to $g_N=0$.

Using a test function $v$ and integrating \cref{eq:fem_governing} by
parts gives the weak formulation: find $T\in V$ such that
\begin{align}
    \int_{\Omega}
    k(T,\mathbf{x};q)
    \nabla T\cdot\nabla v\,\mathrm d\mathbf{x}
    &=
    \int_{\Omega} f(\mathbf{x})v\,\mathrm d\mathbf{x}
    -
    \int_{\Gamma_N}g_Nv\,\mathrm ds,
    \qquad \forall v\in V_0,
    \label{eq:fem_weak}
\end{align}
where
\begin{align}
    V &=
    \left\{
        T\in H^1(\Omega):T=T_D\text{ on }\Gamma_D
    \right\},
    &
    V_0 &=
    \left\{
        v\in H^1(\Omega):v=0\text{ on }\Gamma_D
    \right\}.
\end{align}
For the insulated boundaries, the Neumann contribution vanishes.

The temperature field is approximated using finite elements as
\begin{align}
    T_h(\mathbf{x})
    &=
    \sum_{i=1}^{N_h}T_iN_i(\mathbf{x}),
    \label{eq:fem_interpolation}
\end{align}
where $N_i$ are the finite-element basis functions and $T_i$ the
corresponding nodal temperatures. The resulting discrete systems are
\begin{align}
    K(q)\mathbf T &= \mathbf f,
    \qquad k=k(\mathbf{x};q),
    \label{eq:fem_linear}\\
    K(q,\mathbf T)\mathbf T &= \mathbf f,
    \qquad k=k(T,\mathbf{x}; q),
    \label{eq:fem_nonlinear}
\end{align}
after incorporation of the prescribed Dirichlet values, with element
stiffness matrices
\begin{align}
    K_{ij}^{(e)}
    &=
    \int_{\Omega_e}
    k(T_h,\mathbf{x}; q)
    \nabla N_i\cdot\nabla N_j\,\mathrm d\mathbf{x}.
    \label{eq:fem_stiffness}
\end{align}

For the nonlinear cases, the system in
\cref{eq:fem_nonlinear} is solved by Picard iteration. Given
$T_h^{(\ell)}$, the conductivity is evaluated as
\begin{align}
    k^{(\ell)}(\mathbf{x})
    &=
    k\left(T_h^{(\ell)}(\mathbf{x}),\mathbf{x}; q\right),
\end{align}
and the next iterate follows from
\begin{align}
    K\left(q,T_h^{(\ell)}\right)
    \mathbf T^{(\ell+1)}
    &=\mathbf f.
    \label{eq:fem_picard}
\end{align}
The iteration is terminated when
\begin{align}
    \frac{
        \left\|\mathbf T^{(\ell+1)}-\mathbf T^{(\ell)}\right\|_2
    }{
        \left\|\mathbf T^{(\ell)}\right\|_2
    }
    &<10^{-8}.
    \label{eq:fem_picard_tol}
\end{align}

\subsection{Karhunen--Lo\`eve Parameterization}
\label{app:fem_kl}

The spatially varying conductivity is represented through a truncated
Karhunen--Lo\`eve expansion of its logarithm,
\begin{align}
    k_0(\mathbf{x};\mathbf z)
    &=
    \exp\left(
        \mu_k+
        \sum_{i=1}^{d}
        z_i\sqrt{\lambda_i^{\mathrm{KL}}}\,
        \phi_i(\mathbf{x})
    \right),
    \label{eq:fem_kl}
\end{align}
where $\mu_k$ is the mean log-conductivity,
$\{\lambda_i^{\mathrm{KL}},\phi_i\}$ are the covariance
eigenvalue--eigenfunction pairs, and
$z_i\sim\mathcal N(0,1)$ are independent KL coefficients. The
eigenvalues are ordered as
\begin{align}
    \lambda_1^{\mathrm{KL}}
    \geq\lambda_2^{\mathrm{KL}}
    \geq\cdots.
\end{align}
The superscript ``KL'' distinguishes these eigenvalues from the
generalized Fisher eigenvalues used in the posterior analysis.

The truncation dimension is chosen to retain a prescribed fraction
$\tau$ of the total variance,
\begin{align}
    \frac{\displaystyle\sum_{i=1}^{d}\lambda_i^{\mathrm{KL}}}
         {\displaystyle\sum_{i=1}^{\infty}\lambda_i^{\mathrm{KL}}}
    &\geq\tau.
    \label{eq:fem_kl_truncation}
\end{align}
For the temperature-dependent conductivity problem,
\begin{align}
    k(T,\mathbf{x};q)
    &=
    k_0(\mathbf{x};\mathbf z)
    \left(1+\alpha T(\mathbf{x})\right),
    \label{eq:fem_nonlinear_conductivity}
\end{align}
where $\alpha$ is inferred jointly with the KL coefficients. The
admissible parameter range satisfies
\begin{align}
    1+\alpha T(\mathbf{x})&>0,
    \qquad \mathbf{x}\in\Omega,
    \label{eq:fem_positive_conductivity}
\end{align}
ensuring positive conductivity and ellipticity.

\paragraph{Parameter-to-Observation Map.} For a given $q$, the finite-element solution
$\mathbf T(\mathbf{x};q)$ is mapped to the $m$ temperature observations
through
\begin{align}
    u &=P\mathbf T(q),
    \label{eq:fem_sensor}
\end{align}
where $P\in\mathbb R^{m\times N_h}$ is the observation operator. Thus,
the parameter-to-observation map is
\begin{align}
    \mathcal F(\chi, q)
    &=P\,\mathbf T(\cdot;q),
    \label{eq:fem_forward_map}
\end{align}
with $\chi$ denoting the fixed problem configuration, including the
domain, boundary conditions, material model, and sensor locations.

The observation model is
\begin{align}
    u_{\mathrm{obs}}=
    \mathcal F(\chi, q)+\eta,
    \qquad
    \eta\sim \mathcal N(0,\sigma_\eta^2I_m),
    \label{eq:fem_observation}
\end{align}
where $\sigma_\eta$ is the observation-noise standard deviation. The
same forward map is used in both the VAE decoder and the MCMC reference
calculations.

\section{MCMC Reference Method and Comparison Metrics}
\label{app:mcmc}

\subsection{Adaptive Metropolis Algorithm}
\label{app:am}

MCMC samples from the unnormalised posterior
\begin{align}
    \pi(q) \propto P_{\mathrm{lklh}}(u_{\mathrm{obs}} \mid q)\, P_{\mathrm{pr}}(q)
    \label{eq:app:target}
\end{align}
and serves as the ground-truth reference in \Cref{sec:tc1,sec:tc2,sec:tc3}.  We
use the Adaptive Metropolis (AM) algorithm of~\cite{haario_adaptive_2001} as the
primary sampler, which adapts the proposal covariance online using the
chain history.  The standard Metropolis--Hastings
algorithm~\cite{metropolis_equation_1953,hastings_monte_1970} with a fixed Gaussian
random-walk proposal is used only during the initial non-adaptive
phase; the acceptance probability in both cases is
\begin{align}
    \alpha = \min\!\left(1,\;\frac{\pi(q')}{\pi(q^{(t-1)})}\right).
    \label{eq:app:accept}
\end{align}

\begin{algorithm}[H]
\caption{Adaptive Metropolis (AM)}
\label{alg:am}
\begin{algorithmic}[1]
\REQUIRE Initial state $\mathrm{q}^{(0)}$, total samples $N$, initial covariance
$C_0$, adaptation interval $K$, regularisation $\epsilon=10^{-6}$
\ENSURE Samples $\{\mathrm{q}^{(t)}\}_{t=1}^{N}$ and acceptance rate $r_{\mathrm{acc}}$

\STATE $d \leftarrow \operatorname{dim}(\mathrm{q}^{(0)})$
\STATE $s_d \leftarrow 2.38^2/d$
\STATE $\mathrm{q} \leftarrow \mathrm{q}^{(0)}$
\STATE $\ell \leftarrow \log \pi(\mathrm{q})$
\STATE $C \leftarrow C_0$
\STATE $n_{\mathrm{acc}} \leftarrow 0$

\FOR{$t=1$ \TO $N$}
    \STATE Draw $\mathrm{q}' \sim
    \mathcal{N}(\mathrm{q},C)$
    \IF{the covariance decomposition fails}
        \STATE $C \leftarrow 0.1I$
        \STATE Redraw $\mathrm{q}' \sim
        \mathcal{N}(\mathrm{q},C)$
    \ENDIF

    \STATE $\ell' \leftarrow \log \pi(\mathrm{q}')$
    \IF{$\ell$ and $\ell'$ are finite}
        \STATE $\alpha \leftarrow
        \min\left(1,\exp(\ell'-\ell)\right)$
    \ELSE
        \STATE $\alpha \leftarrow 0$
    \ENDIF

    \STATE Draw $u\sim\mathcal{U}(0,1)$
    \IF{$u<\alpha$}
        \STATE $\mathrm{q}\leftarrow\mathrm{q}'$
        \STATE $\ell\leftarrow\ell'$
        \STATE $n_{\mathrm{acc}}\leftarrow n_{\mathrm{acc}}+1$
    \ENDIF

    \STATE $\mathrm{q}^{(t)}\leftarrow\mathrm{q}$

    \IF{$t\bmod K=0$ \AND $t>d$}
        \STATE $C_{\mathrm{emp}}\leftarrow
        \operatorname{Cov}\left(
        \mathrm{q}^{(0)},\mathrm{q}^{(1)},\ldots,\mathrm{q}^{(t)}
        \right)$
        \STATE $\lambda_{\min}\leftarrow
        \lambda_{\min}(C_{\mathrm{emp}})$
        \IF{$\lambda_{\min}<10^{-8}$}
            \STATE $C_{\mathrm{emp}}\leftarrow
            C_{\mathrm{emp}}+
            (10^{-6}-\lambda_{\min})I$
        \ENDIF
        \STATE $C\leftarrow s_d C_{\mathrm{emp}}+\epsilon I$
    \ENDIF
\ENDFOR

\STATE $r_{\mathrm{acc}}\leftarrow n_{\mathrm{acc}}/N$
\RETURN $\{\mathrm{q}^{(t)}\}_{t=1}^{N},r_{\mathrm{acc}}$
\end{algorithmic}
\end{algorithm}

The scaling factor $s_d = 2.38^2/d$, where $d$ is the number of inferred parameters, is used for the adaptive proposal covariance~\cite{gelman_posterior_1996}. The proposal covariance is initialized with the user-supplied initial covariance $C_0$ and updated every $K$ iterations, with $K=100$ by default, provided that $t>d$. At each adaptation step, the empirical covariance is computed from all samples generated up to that iteration and rescaled by $s_d$. A regularization term $\epsilon I$, with $\epsilon=10^{-6}$, is added to maintain a positive-definite proposal covariance and improve numerical stability. Eigenvalue regularization is additionally applied when the empirical covariance is nearly singular.

\subsection{Convergence Criteria}
\label{app:convergence}
For each test problem, we run $M = 4$ independent chains, with initial
states drawn from the prior distribution. Convergence is assessed using
trace plots, the autocorrelation function (ACF), effective sample size
(ESS), and the rank-normalized split-$\hat{R}$ diagnostic~\cite{gelman_inference_1992,vehtari2021rank}; see~\cite{gelman_bayesian_2014} for background. The
diagnostics are computed using python based library ArviZ~\cite{kumar2019arviz}. \Cref{tab:app}
summarizes the convergence and stopping criteria applied to all runs in \Cref{sec:tc1,sec:tc2,sec:tc3}

\begin{table}[htbp]
\centering
\caption{MCMC convergence and stopping criteria applied to all runs in
\Cref{sec:tc1,sec:tc2,sec:tc3}. Both the ESS and $\hat{R}$ criteria must be satisfied for a run to be accepted as converged.}
\label{tab:app}
\setlength{\tabcolsep}{8pt}
\renewcommand{\arraystretch}{1.25}
\begin{tabular}{@{} l l @{}}
\toprule
\textbf{Diagnostic} & \textbf{Threshold / procedure} \\
\midrule
Chains & $M = 4$ independent chains, initialised from prior probability distribution
($P_{\mathrm{pr}}$) \\
Burn-in & The first 20\% of each chain is discarded as burn-in before computing the convergence diagnostics.\\
ESS & $\mathrm{ESS} > 400$ per parameter \\
$\hat{R}$ & Rank-normalized split-$\hat{R} < 1.01$ per parameter \\
\bottomrule
\end{tabular}
\end{table}

The rank-normalized split-$\hat{R}$ diagnostic is computed using ArviZ
after discarding the burn-in samples. The diagnostic compares the
between-chain and within-chain variability after rank normalization and
chain splitting. Values of $\hat{R}$ close to unity indicate agreement
between the chains; the stopping threshold is $\hat{R} < 1.01$ for each
parameter.

The effective sample size is estimated using the autocorrelation structure
of the post-burn-in chains. For a scalar parameter $q_j$, the conventional
autocorrelation-based expression is
\begin{align}
\mathrm{ESS}_j =
\frac{N}{1 + 2\sum_{k=1}^{K^*}\rho_{j,k}},
\label{eq:app}
\end{align}
where $\rho_{j,k}$ denotes the lag-$k$ autocorrelation and $K^*$ is a
truncation point determined by the autocorrelation estimator. In this
 study, ESS is computed with ArviZ~\cite{kumar2019arviz}, and a minimum ESS of 400 is required for each parameter.

\subsection{Comparison Metrics}
\label{app:metrics}

The following metrics quantify agreement between the VAE posterior
$Q_\phi^{\mathrm{VAE}} = \mathcal{N}(\mu_{\mathrm{post}}^{\mathrm{VAE}},
\Sigma_{\mathrm{post}}^{\mathrm{VAE}})$ and the MCMC reference fitted
as $\hat{P}^{\mathrm{MCMC}} = \mathcal{N}(\hat{\mu}^{\mathrm{MCMC}},
\hat{\Sigma}^{\mathrm{MCMC}})$ using empirical mean and covariance of
the post-burn-in samples. For the nonlinear inverse problems, this defines a moment-matched Gaussian representation of the MCMC posterior; the $W_2$ metric is therefore computed between the Gaussian VAE posterior and this Gaussian representation rather
than directly against the empirical MCMC distribution.

\paragraph{Relative error in posterior mean.}
\begin{align}
    e_\mu
    = \frac{\|\mu_{\mathrm{post}}^{\mathrm{VAE}} -
             \hat{\mu}^{\mathrm{MCMC}}\|_2}
           {\|\hat{\mu}^{\mathrm{MCMC}}\|_2}.
    \label{eq:app:emean}
\end{align}

\paragraph{Relative error in posterior standard deviation.}
\begin{align}
    e_{\sigma,j}
    = \frac{|\sigma_{\mathrm{post},j}^{\mathrm{VAE}} -
             \hat{\sigma}_j^{\mathrm{MCMC}}|}
           {|\hat{\sigma}_j^{\mathrm{MCMC}}|},
    \qquad
    \sigma_{\mathrm{post},j} = \sqrt{(\Sigma_{\mathrm{post}})_{jj}}.
    \label{eq:app:estd}
\end{align}

\paragraph{Full covariance accuracy.}
For multi-parameter problems, marginal spread errors do not capture
errors in the off-diagonal correlation structure.  The normalised
Frobenius-norm error of the full posterior covariance matrix,
\begin{align}
    e_{\mathrm{cov}}
    = \frac{\|\Sigma_{\mathrm{post}}^{\mathrm{VAE}} -
             \hat{\Sigma}^{\mathrm{MCMC}}\|_F}
           {\|\hat{\Sigma}^{\mathrm{MCMC}}\|_F},
    \label{eq:app:ecov}
\end{align}
captures errors in both diagonal variances and off-diagonal
covariances simultaneously.  Here $\|A\|_F = \sqrt{\sum_{i,j}
A_{ij}^2}$ is the Frobenius norm.  A method that recovers
marginal variances accurately but misses inter-parameter
correlations will have small $e_{\sigma,j}$ but large
$e_{\mathrm{cov}}$; the combination of both metrics gives the
complete picture of posterior covariance recovery.

\paragraph{2-Wasserstein distance.}
\begin{align}
    W_2\!\bigl(Q_\phi^{\mathrm{VAE}},\, \hat{P}^{\mathrm{MCMC}}\bigr),
    \label{eq:app:w2}
\end{align}
using the closed-form expression for Gaussian distributions given in
\Cref{eq:W2_gaussian}. This provides a scalar measure of discrepancy between
the two Gaussian posterior representations, accounting for differences in
both mean and covariance.

\section{Mathematical Details}
\label{app:theory}
\subsection{Posterior Covariance in the Fisher Basis}
\label{app:fisher_derivation}

This appendix derives the posterior covariance and mean for a
linear-Gaussian inverse problem with a general Gaussian prior. After
prior whitening, the posterior covariance is diagonal in the generalized
Fisher basis, with modal variances $1/(1+\lambda_k^{\mathrm{Fisher}})$.
For nonlinear forward models, the same construction gives the
corresponding local Laplace approximation.

The following standard linear-Gaussian posterior result is recalled for
completeness~\citep{barber_bayesian_2012}.
\paragraph{Setup and Gaussian posterior.}
Consider the linear-Gaussian inverse problem
$u_{\mathrm{obs}}=A q+\eta$, where
$\eta\sim\mathcal N(0,\sigma_\eta^2I_m)$, with
Gaussian prior $q\sim\mathcal N(\mu_{\mathrm{pr}},
\Sigma_{\mathrm{pr}})$. The negative log-posterior is
\begin{align}
	-\ln P_{\mathrm{post}}(q\mid u_{\mathrm{obs}})
	&=
	\frac12
	\|u_{\mathrm{obs}}-A q\|_{\Sigma_\eta^{-1}}^2
	+
	\frac12
	\|q-\mu_{\mathrm{pr}}\|_{\Sigma_{\mathrm{pr}}^{-1}}^2
	+\mathrm{const.}
	\label{eq:app:neg_log_post_general}
\end{align}
and hence the posterior precision and covariance are
\begin{align}
	\Lambda_{\mathrm{post}}
	=
	\frac{A^\top A}{\sigma_\eta^2}
	+\Sigma_{\mathrm{pr}}^{-1},
	\qquad
	\Sigma_{\mathrm{post}}^*
	=
	\Lambda_{\mathrm{post}}^{-1}.
	\label{eq:app:precision_general}
\end{align}
The corresponding posterior mean is
\begin{align}
	\mu_{\mathrm{post}}^*
	&=
	\Sigma_{\mathrm{post}}^*
	\left(
	\frac{A^\top u_{\mathrm{obs}}}{\sigma_\eta^2}
	+\Sigma_{\mathrm{pr}}^{-1}\mu_{\mathrm{pr}}
	\right).
	\label{eq:app:mean_general}
\end{align}

\paragraph{Prior whitening and generalized Fisher basis.}
Let $\Sigma_{\mathrm{pr}}^{-1}=L_{\mathrm{pr}}L_{\mathrm{pr}}^\top$ and define
the whitened parameter by
$\tilde{q}=L_{\mathrm{pr}}^\top(q-\mu_{\mathrm{pr}})$.
Then $\tilde{q}\sim\mathcal N(0,I_d)$. Defining
$\tilde A=AL_{\mathrm{pr}}^{-\top}$ and
$\tilde u_{\mathrm{obs}}=u_{\mathrm{obs}}-A\mu_{\mathrm{pr}}$ gives
\begin{align}
	\tilde u_{\mathrm{obs}}
	&=
	\tilde A\tilde{q}
	+{\eta}.
	\label{eq:app:whitened_model}
\end{align}

For isotropic observation noise, the generalized Fisher information in
the whitened coordinates is
\begin{align}
	\tilde{\mathcal I}
	&=
	\frac{\tilde A^\top\tilde A}{\sigma_\eta^2}.
	\label{eq:app:fisher_whitened}
\end{align}
Using the SVD $\tilde A=\tilde U\tilde\Sigma\tilde V^\top$, we have
\begin{align}
	\tilde{\mathcal I}
	=
	\tilde V
	\operatorname{diag}
	\left(\lambda_k^{\mathrm{Fisher}}\right)
	\tilde V^\top,
	\qquad
	\lambda_k^{\mathrm{Fisher}}
	=
	\frac{\tilde\sigma_k^2}{\sigma_\eta^2}.
	\label{eq:app:fisher_eig_general}
\end{align}
The columns of $\tilde V$ define the generalized Fisher basis.

\paragraph{Posterior covariance.}
In the whitened coordinates, the posterior precision is
$\tilde\Lambda_{\mathrm{post}}=I_d+\tilde{\mathcal I}$. Therefore,
\begin{align}
	\widetilde{\Sigma}_{\mathrm{post}}^*
	&=
	\tilde V
	\operatorname{diag}
	\left(
	\frac{1}{1+\lambda_k^{\mathrm{Fisher}}}
	\right)
	\tilde V^\top,
	\label{eq:app:cov_diagonal_general}
\end{align}
or, equivalently, the covariance in the generalized Fisher basis is
\begin{align}
	\tilde V^\top
	\widetilde{\Sigma}_{\mathrm{post}}^*
	\tilde V
	&=
	\operatorname{diag}
	\left(
	\frac{1}{1+\lambda_k^{\mathrm{Fisher}}}
	\right)_{k=1}^d.
	\label{eq:app:cov_fisher_general}
\end{align}
Thus the posterior variance in mode $k$ is
\begin{align}
	v_k^*
	&=
	\frac{1}{1+\lambda_k^{\mathrm{Fisher}}}.
	\label{eq:app:modal_variance_general}
\end{align}
The covariance in the original coordinates is recovered as
$\Sigma_{\mathrm{post}}^*
=L_{\mathrm{pr}}^{-\top}\widetilde{\Sigma}_{\mathrm{post}}^*
L_{\mathrm{pr}}^{-1}$.

\paragraph{Posterior mean and modal shrinkage.}
Let $u_{\mathrm{obs}}=Aq^{\mathrm{true}}+{\eta}$ and
define $\tilde{q}^{\mathrm{true}}
=L_{\mathrm{pr}}^\top(q^{\mathrm{true}}-\mu_{\mathrm{pr}})$.
Projection of \cref{eq:app:mean_general} onto the generalized Fisher
basis gives
\begin{align}
	\tilde\mu_{\mathrm{post},k}^*
	&=
	\frac{\lambda_k^{\mathrm{Fisher}}}
	{1+\lambda_k^{\mathrm{Fisher}}}
	\tilde q_k^{\mathrm{true}}
	+
	\frac{1}{1+\lambda_k^{\mathrm{Fisher}}}
	\frac{\tilde\sigma_k}{\sigma_\eta^2}
	\tilde\eta_k.
	\label{eq:app:post_mean_general}
\end{align}
Hence, averaging over the observation noise,
\begin{align}
	\mathbb E_{{\eta}}
	\left[
	\tilde\mu_{\mathrm{post},k}^*
	\right]
	&=
	\frac{\lambda_k^{\mathrm{Fisher}}}
	{1+\lambda_k^{\mathrm{Fisher}}}
	\tilde q_k^{\mathrm{true}}.
	\label{eq:app:shrinkage_general}
\end{align}
Thus weakly identified modes remain close to the prior mean, whereas
strongly identified modes approach the true parameter component.

\paragraph{Special case.}
For $\mu_{\mathrm{pr}}= 0$ and $\Sigma_{\mathrm{pr}}=I_d$,
$L_{\mathrm{pr}}=I_d$ and $\tilde A=A$. Hence
$\lambda_k^{\mathrm{Fisher}}=\sigma_{A,k}^2/\sigma_\eta^2$, and
\cref{eq:app:modal_variance_general} reduces to
\begin{align}
	v_k^*
	&=
	\frac{1}{1+\lambda_k^{\mathrm{Fisher}}},
	\label{eq:app:modal_variance_special}
\end{align}
which is the Fisher-modal posterior covariance used in the main text.
Equation~\cref{eq:app:shrinkage_general} likewise reduces to the
standard modal shrinkage relation.

\paragraph{Nonlinear forward models.}
For a nonlinear forward model $\mathcal F(q)$, linearization
about a fixed reference point $\hat{q}$ gives
$\mathcal F(q)\approx\mathcal F(\hat{q})
+J(q-\hat{q})$, where
$J=\left.\partial\mathcal F/\partial q\right|_{\hat{q}}$.
Replacing $A$ by $J$ in the preceding construction yields the
corresponding local Fisher information and Laplace covariance. Thus
$v_k^*\approx1/(1+\lambda_k^{\mathrm{Fisher}})$ is a local approximation
for the nonlinear inverse problem rather than an exact posterior
variance.
\subsection{Derivation of Loss Functions and Modal Gradients}
\label{app:loss_derivations}

This subsection provides the mathematical derivations underlying the
loss functions and modal gradient expressions used in the main text.
We first derive the closed-form expressions for the Gaussian
divergence measures and the reconstruction loss used in the
computational formulation. We then derive the corresponding gradient
expressions in the whitened generalized Fisher basis, which are used
to analyze the variance response of the different VAE formulations.
\subsubsection{Derivation of Implementable Loss Functions}
\label{app:derivations}
	
This section provides the closed-form expressions underlying the
implementable loss functions used in \cref{eq:loss_kl,eq:loss_js,eq:loss_jswa}.
We first give the Gaussian KL divergence and its specialization to the
reverse- and forward-KL terms, followed by the Gaussian $W_2$ distance and
the reconstruction loss.

\paragraph{Gaussian KL divergence.}
For $Q = \mathcal{N}(\mu_Q,\Sigma_Q)$ and
$P = \mathcal{N}(\mu_P,\Sigma_P)$ in $\mathbb{R}^n$,
\begin{align}
D_{\mathrm{KL}}(Q\|P)
	=
	\frac{1}{2}\left[
	\operatorname{tr}\!\left(\Sigma_P^{-1}\Sigma_Q\right)
	+ \|\mu_Q-\mu_P\|_{\Sigma_P^{-1}}^2
		- n
	+ \ln\frac{|\Sigma_P|}{|\Sigma_Q|}
	\right].
	\label{eq:kl_gaussian_full}
		\end{align}
		
		\begin{proof}
	Using
	$D_{\mathrm{KL}}(Q\|P)
	=\mathbb{E}_Q[\ln Q]-\mathbb{E}_Q[\ln P]$,
		the Gaussian entropy gives
		$\mathbb{E}_Q[\ln Q]
		=-\frac{n}{2}\ln(2\pi)-\frac{1}{2}\ln|\Sigma_Q|-\frac{n}{2}$.
	For the cross-entropy term,
	\begin{align}
	\mathbb{E}_Q[\ln P]
		=
		-\frac{n}{2}\ln(2\pi)
		-\frac{1}{2}\ln|\Sigma_P|
		-\frac{1}{2}
		\left[
	\operatorname{tr}\!\left(\Sigma_P^{-1}\Sigma_Q\right)
	+(\mu_Q-\mu_P)^\top
	\Sigma_P^{-1}
		(\mu_Q-\mu_P)
		\right],
		\end{align}
		where we have used
		\begin{align}
		\mathbb{E}_Q
		\!\left[
	(q-\mu_P)^\top\Sigma_P^{-1}(q-\mu_P)
	\right]
	=
	\operatorname{tr}\!\left(\Sigma_P^{-1}\Sigma_Q\right)
		+(\mu_Q-\mu_P)^\top
		\Sigma_P^{-1}
		(\mu_Q-\mu_P).
		\end{align}
		Subtracting the two expressions yields \cref{eq:kl_gaussian_full}.
		\end{proof}
		
		Setting $Q=Q_\phi$ and $P=P_{\mathrm{pr}}$ in
		\cref{eq:kl_gaussian_full} gives the reverse-KL regularization term $\mathcal{L}_{\mathrm{RKL}}$ in the training objectives.
		Dropping the additive constant $-n/2$ gives the implementable expression
		used in \cref{eq:loss_kl}. For the forward-KL term, the expectation with
	respect to the target posterior is approximated using the paired parameter
	$q^{(i)}$. This gives the supervised approximation
	\begin{align}
		D_{\mathrm{KL}}\!\left(P(q\mid u)\,\|\,Q_\phi\right)
	\approx
	\frac{1}{2}
	\left[
		\ln|\Sigma_{\mathrm{post}}|
		+
		\|q^{(i)}-\mu_{\mathrm{post}}\|_{\Sigma_{\mathrm{post}}^{-1}}^2
	\right]
	+ \mathrm{const.}
	\label{eq:kl_forward_applied}
	\end{align}
		which is the forward-KL term $\mathcal{L}_{\mathrm{FKL}}$ used in the training objectives.
		
		\paragraph{Squared $2$-Wasserstein distance.}
		For Gaussian distributions
		$Q=\mathcal{N}(\mu_Q,\Sigma_Q)$ and
		$P=\mathcal{N}(\mu_P,\Sigma_P)$ in $\mathbb{R}^n$,
		the squared $2$-Wasserstein distance is given by
		\begin{align}
		W_2^2(Q,P)
		=
		\|\mu_Q-\mu_P\|_2^2
	+
	\operatorname{tr}\!\left(
	\Sigma_Q+\Sigma_P
		-2
			\bigl(
			\Sigma_Q^{1/2}
		\Sigma_P
		\Sigma_Q^{1/2}
		\bigr)^{1/2}
		\right).
		\label{eq:w2_gaussian_full}
		\end{align}
			\citep{dowson_frechet_1982,bhatia_bures_2019a}
			The covariance-dependent term is the squared Bures metric and vanishes
		if and only if $\Sigma_Q=\Sigma_P$. For diagonal covariances,
		\cref{eq:w2_gaussian_full} reduces to
		\begin{align}
		W_2^2(Q,P)
	=
	\|\mu_Q-\mu_P\|_2^2
+
\sum_{j=1}^{n}
\left(\sigma_{Q,j}-\sigma_{P,j}\right)^2.
\end{align}
	In the implementation, the general full-covariance expression in
	\cref{eq:w2_gaussian_full} is evaluated using the eigendecomposition of
	$L_Q^\top\Sigma_P L_Q$, where
	$\Sigma_Q=L_QL_Q^\top$. For the dimensions considered in this work
	($n\leq15$), the resulting $\mathcal{O}(n^3)$ computation is negligible
	relative to the neural-network forward pass.
		
		\paragraph{Gaussian reconstruction loss.}
		For Gaussian observation noise
		$\eta\sim\mathcal{N}(\mu_\eta,\Sigma_\eta)$, the negative log-likelihood is
		\begin{align}
		-\ln P(u_{\mathrm{obs}}\mid q)
		=
	\frac{1}{2}
	\left\|
	u_{\mathrm{obs}}
		-\mathcal{F}(\chi,q)
	-\mu_\eta
	\right\|_{\Sigma_\eta^{-1}}^2
	+\mathrm{const.},
	\label{eq:reconstruction_derivation}
		\end{align}
		where terms independent of the variational parameters have been omitted.
		Replacing the forward model $\mathcal{F}(\chi,q)$ with the frozen decoder
	$\Psi_d(q_{\mathrm{draw}})$ and estimating the expectation over
	$Q_\phi(q\mid u_{\mathrm{obs}})$ using one reparameterized sample per
	iteration yields the reconstruction term $\mathcal{L}_{\mathrm{rec}}$ used in
		\cref{eq:loss_kl,eq:loss_js,eq:loss_jswa}.

\subsubsection{Derivation of the Reconstruction Gradient in the
	Whitened Fisher Basis}
\label{app:recon_gradient}

This subsection derives the reconstruction-gradient contribution used in
the modal analysis. We consider the linear-Gaussian setting used in the
benchmark, with isotropic observation noise, diagonal encoder covariance
in the whitened generalized Fisher basis, and a calibrated encoder mean.

For the variance-sensitivity analysis, we consider a locally mean-matched
encoder, for which the encoder mean coincides with the true parameter in
the whitened coordinates,
\begin{align}
	\widetilde{\mu}_{\mathrm{post}}
	=
	\widetilde{q}_{\mathrm{true}}.
	\label{eq:app_mean_matching}
\end{align}
This assumption isolates the variance-dependent contribution to the
reconstruction loss and is not intended to imply that the Bayesian
posterior mean generally coincides with the realized true parameter.

\begin{proposition}[Reconstruction gradient in the whitened Fisher basis]
	\label{prop:reconstruction_gradient}
	Consider the linear forward model
	$u=Aq+\eta$, where
	$\eta\sim\mathcal N(\mathbf 0,\sigma_\eta^2I_m)$. Let
	$\widetilde A=AL_{\mathrm{pr}}^{-\top}$ denote the prior-whitened forward
	operator, with SVD
	$\widetilde A=\widetilde U\widetilde\Sigma_{\mathrm{sv}}\widetilde V^\top$.
	The corresponding Fisher eigenvalues are
	\begin{align}
		\lambda_k^{\mathrm{Fisher}}
		=
		\frac{\widetilde{\sigma}_{\mathrm{sv},k}^{\,2}}
		{\sigma_\eta^2}.
		\label{eq:app_fisher_eigenvalues}
	\end{align}
	Express the encoder posterior in the generalized Fisher basis as
	$Q_\phi=\mathcal N(\widetilde\mu_{\mathrm{post}},
	\widetilde\Sigma_{\mathrm{post}})$, with
	$\widetilde\Sigma_{\mathrm{post}}
	=\operatorname{diag}(v_1,\ldots,v_d)$ and
	$v_k=\widetilde\sigma_k^2$. If
	$\widetilde\mu_{\mathrm{post}}=\widetilde{q}_{\mathrm{true}}$,
	then, up to terms independent of $v_k$,
	\begin{align}
		\mathbb E_{Q_\phi}[\ell_{\mathrm{rec}}]
		=
		\frac{1}{2}\sum_{k=1}^{d}
		\lambda_k^{\mathrm{Fisher}}v_k+\mathrm{const.},
		\qquad
		\frac{\partial\mathbb E_{Q_\phi}[\ell_{\mathrm{rec}}]}
		{\partial v_k}
		=
		\frac{1}{2}\lambda_k^{\mathrm{Fisher}}
		\label{eq:app_reconstruction_gradient}
	\end{align}
\end{proposition}

\begin{proof}
	The encoder posterior in the whitened generalized Fisher basis is
	reparameterized as
	$\widetilde{q}
	=\widetilde{\mu}_{\mathrm{post}}
	+\widetilde{\sigma}\odot\epsilon$,
	where $\epsilon\sim\mathcal N(\mathbf0,I_d)$ and
	$\widetilde\sigma_k=\sqrt{v_k}$. Under the local mean-matching
	assumption
	$\widetilde{\mu}_{\mathrm{post}}
	=\widetilde{q}_{\mathrm{true}}$,
	\begin{align}
		\widetilde{q}
		-
		\widetilde{q}_{\mathrm{true}}
		=
		\widetilde{\sigma}\odot\epsilon,
		\qquad
		\widetilde q_k-\widetilde q_{\mathrm{true},k}
		=
		\widetilde\sigma_k\epsilon_k .
		\label{eq:app_parameter_deviation}
	\end{align}
	
	After prior whitening and centering,
	$\widetilde u_{\mathrm{obs}}
	=u_{\mathrm{obs}}-A\mu_{\mathrm{pr}}
	=\widetilde A\widetilde{q}_{\mathrm{true}}
	+\eta$, while a posterior draw gives
	$\widehat{\widetilde u}=\widetilde A\widetilde{q}$. Hence
	\begin{align}
		\widetilde u_{\mathrm{obs}}-\widehat{\widetilde u}
		=
		-\widetilde A
		(\widetilde{\sigma}\odot\epsilon)
		+\eta .
		\label{eq:app_reconstruction_residual}
	\end{align}
	Let $\widetilde{\mathbf v}_k$ denote the $k$th right-singular vector of
	$\widetilde A$ and define
	$\widetilde{\mathbf a}_k=\widetilde A\widetilde{\mathbf v}_k$.
	From
	$\widetilde A=\widetilde U\widetilde\Sigma_{\mathrm{sv}}\widetilde V^\top$,
	we have
	$\|\widetilde{\mathbf a}_k\|_2^2
	=\widetilde\sigma_{\mathrm{sv},k}^{\,2}$. Therefore,
	\begin{align}
		\widetilde u_{\mathrm{obs}}-\widehat{\widetilde u}
		=
		-\sum_{k=1}^{d}
		\widetilde\sigma_k\epsilon_k\widetilde{\mathbf a}_k
		+\eta .
		\label{eq:app_residual_svd}
	\end{align}
	
	Taking the squared Euclidean norm of
	\cref{eq:app_residual_svd} and averaging over
	$\epsilon$, with
	$\mathbb E[\epsilon_k]=0$ and $\mathbb E[\epsilon_k^2]=1$, eliminates the
	cross terms and gives
	\begin{align}
		\mathbb E_{\epsilon}
		\left[
		\|\widetilde u_{\mathrm{obs}}
		-\widehat{\widetilde u}\|_2^2
		\right]
		=
		\sum_{k=1}^{d}
		\widetilde\sigma_k^2
		\|\widetilde{\mathbf a}_k\|_2^2
		+\|\eta\|_2^2
		=
		\sum_{k=1}^{d}
		\widetilde\sigma_k^2
		\widetilde\sigma_{\mathrm{sv},k}^{\,2}
		+\|\eta\|_2^2 .
		\label{eq:app_expected_residual}
	\end{align}
	
	For isotropic observation noise,
	$\Sigma_\eta=\sigma_\eta^2I_m$, so that
	$\|\mathbf r\|_{\Sigma_\eta^{-1}}^2
	=\|\mathbf r\|_2^2/\sigma_\eta^2$. Consequently,
	\begin{align}
		\mathbb E_{\epsilon}
		\left[
		\|\widetilde u_{\mathrm{obs}}
		-\widehat{\widetilde u}\|_{\Sigma_\eta^{-1}}^2
		\right]
		=
		\sum_{k=1}^{d}
		\frac{\widetilde\sigma_{\mathrm{sv},k}^{\,2}}
		{\sigma_\eta^2}
		\widetilde\sigma_k^2
		+\frac{\|\eta\|_2^2}{\sigma_\eta^2}~=
		\sum_{k=1}^{d}
		\lambda_k^{\mathrm{Fisher}}v_k
		+\mathrm{const.},
		\label{eq:app_expected_weighted_residual}
	\end{align}
	where the final term is independent of the encoder variances. Since the
	Gaussian reconstruction loss contains the factor $1/2$,
	\begin{align}
		\mathbb E_{Q_\phi}[\ell_{\mathrm{rec}}]
		=
		\frac{1}{2}
		\sum_{k=1}^{d}
		\lambda_k^{\mathrm{Fisher}}v_k
		+\mathrm{const.}.
		\label{eq:app_reconstruction_expected_loss}
	\end{align}
Finally, because the variance-dependent contribution is separable across Fisher modes,
	differentiation with respect to $v_k$ gives
	\begin{align}
		\frac{\partial
			\mathbb E_{Q_\phi}[\ell_{\mathrm{rec}}]}
		{\partial v_k}
		=
		\frac{1}{2}\lambda_k^{\mathrm{Fisher}},
		\qquad
		\frac{\partial^2
			\mathbb E_{Q_\phi}[\ell_{\mathrm{rec}}]}
		{\partial v_k^2}
		=
		0 .
		\label{eq:app_reconstruction_derivatives}
	\end{align}
	This completes the derivation.
\end{proof}
The reconstruction contribution is therefore constant with respect to the encoder variance within each Fisher mode, while its magnitude is directly proportional to the corresponding Fisher eigenvalue. Thus, within the local linear-Gaussian approximation, reconstruction is more sensitive to variance in strongly informative directions and less sensitive in poorly informative directions. The additional variance-dependent behavior in the modal stationary-gradient balance therefore arises from the forward-KL and prior-regularization terms.

\subsubsection{Derivation of the Forward-KL Gradient}
\label{app:fkl_gradient}

This subsection derives the forward-KL gradient and its curvature with
respect to the encoder posterior variance in a single Fisher mode. We
consider the linear-Gaussian setting used in the benchmark and isolate the
variance dependence by taking the encoder and exact posterior means to be
identical.

Under the local linear-Gaussian approximation, the generalized Fisher
basis diagonalizes the local information matrix and, consequently, the
posterior covariance. With the encoder covariance taken to be diagonal
in this basis, the variance-dependent contributions to the objective
separate across Fisher modes. We can therefore analyze the contribution
of an individual mode $k$ through its scalar variance $v_k$, while holding
the corresponding Fisher eigenvalue $\lambda_k^{\mathrm{Fisher}}$ fixed.

\begin{proposition}[Forward-KL gradient and curvature in a Fisher mode]
	\label{prop:fwdkl-grad}
	Let the exact posterior marginal in Fisher mode $k$ be
	$P_{\mathrm{post}}^{(k)}=\mathcal{N}(0,v_k^*)$, where
	$v_k^*=\tilde{\sigma}_k^{*2}
	=(1+\lambda_k^{\mathrm{Fisher}})^{-1}$, and let the encoder marginal be
	$Q_\phi^{(k)}=\mathcal{N}(0,v_k)$ with
	$v_k=\tilde{\sigma}_k^2>0$. Then
	\begin{align}
		L_{\mathrm{FKL}}(v_k)
		&=
		D_{\mathrm{KL}}
		\left(
		P_{\mathrm{post}}^{(k)}
		\,\|\,Q_\phi^{(k)}
		\right)
		=
		\frac{1}{2}
		\left[
		\frac{v_k^*}{v_k}
		+\log v_k
		-\log v_k^*
		-1
		\right],
		\label{eq:app_fwdkl_closed}
	\end{align}
\end{proposition}

\begin{proof}
	For a zero-mean univariate Gaussian with variance $v$,
	\begin{align}
		\log \mathcal{N}(x;0,v)
		=
		-\frac{1}{2}\log(2\pi)
		-\frac{1}{2}\log v
		-\frac{x^2}{2v}.
		\label{eq:app_gaussian_log_density}
	\end{align}
	Hence,
	\begin{align}
		\log\frac{p_k(x)}{q_k(x)}
		&=
		\frac{1}{2}
		\left(
		\log v_k-\log v_k^*
		\right)
		+
		\frac{x^2}{2}
		\left(
		\frac{1}{v_k}
		-
		\frac{1}{v_k^*}
		\right).
		\label{eq:app_fwdkl_logratio}
	\end{align}
	Since $x\sim\mathcal{N}(0,v_k^*)$ under the exact posterior,
	$\mathbb{E}[x^2]=v_k^*$. Taking the expectation of
	\cref{eq:app_fwdkl_logratio} therefore gives
	\begin{align}
		D_{\mathrm{KL}}
		\left(
		P_{\mathrm{post}}^{(k)}
		\,\|\,Q_\phi^{(k)}
		\right)
		&=
		\frac{1}{2}
		\left[
		\frac{v_k^*}{v_k}
		+\log v_k
		-\log v_k^*
		-1
		\right].
	\end{align}
	
	Because $v_k^*$ is fixed with respect to the encoder parameters,
	differentiation with respect to $v_k$ gives
	\begin{align}
		\frac{\partial L_{\mathrm{FKL}}}{\partial v_k}
		&=
		\frac{1}{2}
		\left[
		-\frac{v_k^*}{v_k^2}
		+\frac{1}{v_k}
		\right]
		=
		\frac{v_k-v_k^*}{2v_k^2},
		\label{eq:app_fwdkl_gradient_derivation}
	\end{align}
	which establishes \Cref{eq:app_fwdkl_closed}. Thus,
	$\partial L_{\mathrm{FKL}}/\partial v_k<0$ for $v_k<v_k^*$ and
	$\partial L_{\mathrm{FKL}}/\partial v_k>0$ for $v_k>v_k^*$, so gradient
	descent drives the encoder variance toward $v_k^*$.
	
\end{proof}

\begin{remark}[Structural properties of the forward-KL term]
	\label{rem:fwdkl_properties}
	The forward-KL term satisfies
	$L_{\mathrm{FKL}}(v_k^*)=0$ and
	$L_{\mathrm{FKL}}(v_k)\geq0$ for $v_k>0$. Its gradient changes sign at
	$v_k=v_k^*$, so gradient descent increases the variance when
	$v_k<v_k^*$ and decreases it when $v_k>v_k^*$. Thus, in the modal
	variance dynamics, the forward-KL term acts as a posterior-correcting
	contribution.
\end{remark}



\addcontentsline{toc}{section}{References}
\begin{thebibliography}{70}
\providecommand{\natexlab}[1]{#1}
\providecommand{\url}[1]{\texttt{#1}}
\expandafter\ifx\csname urlstyle\endcsname\relax
  \providecommand{\doi}[1]{doi: #1}\else
  \providecommand{\doi}{doi: \begingroup \urlstyle{rm}\Url}\fi

\bibitem[Uecker et~al.(2008)Uecker, Hohage, Block, and
  Frahm]{uecker_image_2008}
Martin Uecker, Thorsten Hohage, Kai~Tobias Block, and Jens Frahm.
\newblock Image reconstruction by regularized nonlinear inversion-{{Joint}}
  estimation of coil sensitivities and image content.
\newblock \emph{Magnetic Resonance in Medicine}, 60\penalty0 (3):\penalty0
  674--682, 2008.
\newblock \doi{10.1002/mrm.21691}.

\bibitem[Lustig et~al.(2007)Lustig, Donoho, and Pauly]{lustig_sparse_2007}
Michael Lustig, David Donoho, and John~M. Pauly.
\newblock Sparse {{MRI}}: {{The}} application of compressed sensing for rapid
  {{MR}} imaging.
\newblock \emph{Magnetic Resonance in Medicine}, 58\penalty0 (6):\penalty0
  1182--1195, 2007.
\newblock \doi{10.1002/mrm.21391}.

\bibitem[Narnhofer et~al.(2022)Narnhofer, Effland, Kobler, Hammernik, Knoll,
  and Pock]{narnhofer_bayesian_2022}
Dominik Narnhofer, Alexander Effland, Erich Kobler, Kerstin Hammernik, Florian
  Knoll, and Thomas Pock.
\newblock Bayesian {{Uncertainty Estimation}} of {{Learned Variational MRI
  Reconstruction}}.
\newblock \emph{IEEE Transactions on Medical Imaging}, 41\penalty0
  (2):\penalty0 279--291, 2022.
\newblock \doi{10.1109/TMI.2021.3112040}.

\bibitem[Sidky and Pan(2008)]{sidky_image_2008}
Emil~Y Sidky and Xiaochuan Pan.
\newblock Image reconstruction in circular cone-beam computed tomography by
  constrained, total-variation minimization.
\newblock \emph{Physics in Medicine and Biology}, 53\penalty0 (17):\penalty0
  4777--4807, 2008.
\newblock \doi{10.1088/0031-9155/53/17/021}.

\bibitem[Virieux and Operto(2009)]{virieux_overview_2009}
J.~Virieux and S.~Operto.
\newblock An overview of full-waveform inversion in exploration geophysics.
\newblock \emph{Geophysics}, 74\penalty0 (6):\penalty0 WCC1--WCC26, 2009.
\newblock \doi{10.1190/1.3238367}.

\bibitem[Fichtner et~al.(2006)Fichtner, Bunge, and
  Igel]{fichtner_adjoint_2006b}
Andreas Fichtner, H-P Bunge, and Heiner Igel.
\newblock The adjoint method in seismology: {{I}}. {{Theory}}.
\newblock \emph{Physics of the Earth and Planetary Interiors}, 157\penalty0
  (1-2):\penalty0 86--104, 2006.

\bibitem[Hansen(2006)]{hansen_deblurring_2006}
Per~Christian Hansen.
\newblock \emph{Deblurring images: matrices, spectra, and filtering}.
\newblock Fundamentals of algorithms. {SIAM, Society for Industrial and Applied
  Mathematics}, Philadelphia, 2006.

\bibitem[Babacan et~al.(2009)Babacan, Molina, and
  Katsaggelos]{babacan_variational_2009}
S.D. Babacan, R.~Molina, and A.K. Katsaggelos.
\newblock Variational {{Bayesian Blind Deconvolution Using}} a {{Total
  Variation Prior}}.
\newblock \emph{IEEE Transactions on Image Processing}, 18\penalty0
  (1):\penalty0 12--26, 2009.
\newblock \doi{10.1109/TIP.2008.2007354}.

\bibitem[Shull(2002)]{shull_nondestructive_2002}
Peter~J. Shull, editor.
\newblock \emph{Nondestructive evaluation: theory, techniques, and
  applications}.
\newblock Number 142 in Mechanical engineering. Dekker, New York Basel, 2002.

\bibitem[Doyley(2012)]{doyley_modelbased_2012}
M~M Doyley.
\newblock Model-based elastography: a survey of approaches to the inverse
  elasticity problem.
\newblock \emph{Physics in Medicine and Biology}, 57\penalty0 (3):\penalty0
  R35--R73, 2012.
\newblock \doi{10.1088/0031-9155/57/3/R35}.

\bibitem[Colton and Kress(2019)]{colton_inverse_2019}
David Colton and Rainer Kress.
\newblock \emph{Inverse {{Acoustic}} and {{Electromagnetic Scattering
  Theory}}}, volume~93 of \emph{Applied {{Mathematical Sciences}}}.
\newblock Springer International Publishing, Cham, 2019.
\newblock \doi{10.1007/978-3-030-30351-8}.

\bibitem[Evensen(2009)]{evensen_data_2009}
Geir Evensen.
\newblock \emph{Data {{Assimilation}}: {{The Ensemble Kalman Filter}}}.
\newblock Springer Berlin Heidelberg, Berlin, Heidelberg, 2009.
\newblock \doi{10.1007/978-3-642-03711-5}.

\bibitem[Lorenc(1986)]{lorenc_analysis_1986}
A.~C. Lorenc.
\newblock Analysis methods for numerical weather prediction.
\newblock \emph{Quarterly Journal of the Royal Meteorological Society},
  112\penalty0 (474):\penalty0 1177--1194, 1986.
\newblock \doi{10.1002/qj.49711247414}.

\bibitem[Friswell(2007)]{friswell_damage_2007}
Michael~I Friswell.
\newblock Damage identification using inverse methods.
\newblock \emph{Philosophical Transactions of the Royal Society A:
  Mathematical, Physical and Engineering Sciences}, 365\penalty0
  (1851):\penalty0 393--410, 2007.
\newblock \doi{10.1098/rsta.2006.1930}.

\bibitem[Borcea(2002)]{borcea_electrical_2002a}
Liliana Borcea.
\newblock Electrical impedance tomography.
\newblock \emph{Inverse Problems}, 18\penalty0 (6):\penalty0 R99--R136, 2002.
\newblock \doi{10.1088/0266-5611/18/6/201}.

\bibitem[Kaipio et~al.(2000)Kaipio, Kolehmainen, Somersalo, and
  Vauhkonen]{kaipio_statistical_2000}
Jari~P Kaipio, Ville Kolehmainen, Erkki Somersalo, and Marko Vauhkonen.
\newblock Statistical inversion and {{Monte Carlo}} sampling methods in
  electrical impedance tomography.
\newblock \emph{Inverse Problems}, 16\penalty0 (5):\penalty0 1487--1522, 2000.
\newblock \doi{10.1088/0266-5611/16/5/321}.

\bibitem[Gokhale et~al.(2008)Gokhale, Barbone, and
  Oberai]{gokhale_solution_2008}
Nachiket~H Gokhale, Paul~E Barbone, and Assad~A Oberai.
\newblock Solution of the nonlinear elasticity imaging inverse problem: the
  compressible case.
\newblock \emph{Inverse Problems}, 24\penalty0 (4):\penalty0 045010, 2008.
\newblock \doi{10.1088/0266-5611/24/4/045010}.

\bibitem[{\"O}zi{\c s}ik and Orlande(2021)]{ozisik_inverse_2021}
M.~Necati {\"O}zi{\c s}ik and Helcio R.~B. Orlande.
\newblock \emph{Inverse heat transfer: fundamentals and applications}.
\newblock Heat transfer series. CRC Press, Taylor \& Francis Group, Boca Raton
  London New York, second edition edition, 2021.

\bibitem[Higdon et~al.(2008)Higdon, Nakhleh, Gattiker, and
  Williams]{higdon_bayesian_2008}
Dave Higdon, Charles Nakhleh, James Gattiker, and Brian Williams.
\newblock A bayesian calibration approach to the thermal problem.
\newblock \emph{Computer Methods in Applied Mechanics and Engineering},
  197\penalty0 (29-32):\penalty0 2431--2441, 2008.
\newblock \doi{10.1016/j.cma.2007.05.031}.

\bibitem[Engl et~al.(1996)Engl, Hanke, and Neubauer]{engl_regularization_1996}
Heinz~W. Engl, Martin Hanke, and Andreas Neubauer.
\newblock \emph{Regularization of inverse problems}.
\newblock Number v. 375 in Mathematics and its applications. Kluwer Academic
  Publishers, Dordrecht ; Boston, 1996.

\bibitem[Vogel(2002)]{vogel_computational_2002}
Curtis~R. Vogel.
\newblock \emph{Computational methods for inverse problems}.
\newblock Frontiers in applied mathematics. {Society for Industrial and Applied
  Mathematics}, Philadelphia, 2002.

\bibitem[Dashti and Stuart(2017)]{dashti_bayesian_2017}
Masoumeh Dashti and Andrew~M. Stuart.
\newblock The {{Bayesian Approach}} to {{Inverse Problems}}.
\newblock In Roger Ghanem, David Higdon, and Houman Owhadi, editors,
  \emph{Handbook of {{Uncertainty Quantification}}}, pages 311--428. Springer
  International Publishing, Cham, 2017.
\newblock \doi{10.1007/978-3-319-12385-1_7}.

\bibitem[Stuart(2010)]{stuart_inverse_2010}
A.~M. Stuart.
\newblock Inverse problems: {{A Bayesian}} perspective.
\newblock \emph{Acta Numerica}, 19:\penalty0 451--559, 2010.
\newblock \doi{10.1017/S0962492910000061}.

\bibitem[Kaipio and Somersalo(2005)]{kaipio_statistical_2005}
Jari Kaipio and Erkki Somersalo.
\newblock \emph{Statistical and computational inverse problems}.
\newblock Number v. 160 in Applied mathematical sciences. Springer, New York,
  2005.

\bibitem[Robert and Casella(2004)]{robert_monte_2004}
Christian~P. Robert and George Casella.
\newblock \emph{Monte {{Carlo}} statistical methods}.
\newblock Springer texts in statistics. Springer, New York, 2nd ed edition,
  2004.

\bibitem[Metropolis et~al.(1953)Metropolis, Rosenbluth, Rosenbluth, Teller, and
  Teller]{metropolis_equation_1953}
Nicholas Metropolis, Arianna~W. Rosenbluth, Marshall~N. Rosenbluth, Augusta~H.
  Teller, and Edward Teller.
\newblock Equation of {{State Calculations}} by {{Fast Computing Machines}}.
\newblock \emph{The Journal of Chemical Physics}, 21\penalty0 (6):\penalty0
  1087--1092, 1953.
\newblock \doi{10.1063/1.1699114}.

\bibitem[Hastings(1970)]{hastings_monte_1970}
W.~K. Hastings.
\newblock Monte {{Carlo Sampling Methods Using Markov Chains}} and {{Their
  Applications}}.
\newblock \emph{Biometrika}, 57\penalty0 (1):\penalty0 97--109, 1970.
\newblock \doi{10.1093/biomet/57.1.97}.

\bibitem[Neal(2011)]{neal_mcmc_2011}
Radford~M Neal.
\newblock {{MCMC}} using {{Hamiltonian}} dynamics.
\newblock \emph{Handbook of markov chain monte carlo}, pages 47--95, 2011.

\bibitem[Haario et~al.(2006)Haario, Laine, Mira, and Saksman]{haario_dram_2006}
Heikki Haario, Marko Laine, Antonietta Mira, and Eero Saksman.
\newblock {{DRAM}}: {{Efficient}} adaptive {{MCMC}}.
\newblock \emph{Statistics and Computing}, 16\penalty0 (4):\penalty0 339--354,
  2006.
\newblock \doi{10.1007/s11222-006-9438-0}.

\bibitem[Gelman and Rubin(1992)]{gelman_inference_1992}
Andrew Gelman and Donald~B Rubin.
\newblock Inference from iterative simulation using multiple sequences.
\newblock \emph{Statistical science}, 7\penalty0 (4):\penalty0 457--472, 1992.

\bibitem[Brooks and Gelman(1998)]{brooks_general_1998}
Stephen~P. Brooks and Andrew Gelman.
\newblock General {{Methods}} for {{Monitoring Convergence}} of {{Iterative
  Simulations}}.
\newblock \emph{Journal of Computational and Graphical Statistics}, 7\penalty0
  (4):\penalty0 434--455, 1998.
\newblock \doi{10.1080/10618600.1998.10474787}.

\bibitem[Blei et~al.(2017)Blei, Kucukelbir, and
  McAuliffe]{blei_variational_2017}
David~M. Blei, Alp Kucukelbir, and Jon~D. McAuliffe.
\newblock Variational {{Inference}}: {{A Review}} for {{Statisticians}}.
\newblock \emph{Journal of the American Statistical Association}, 112\penalty0
  (518):\penalty0 859--877, 2017.
\newblock \doi{10.1080/01621459.2017.1285773}.

\bibitem[Jordan et~al.(1998)Jordan, Ghahramani, Jaakkola, and
  Saul]{jordan_introduction_1998}
Michael~I. Jordan, Zoubin Ghahramani, Tommi~S. Jaakkola, and Lawrence~K. Saul.
\newblock An {{Introduction}} to {{Variational Methods}} for {{Graphical
  Models}}.
\newblock In Michael~I. Jordan, editor, \emph{Learning in {{Graphical
  Models}}}, pages 105--161. Springer Netherlands, Dordrecht, 1998.
\newblock \doi{10.1007/978-94-011-5014-9_5}.

\bibitem[Raissi et~al.(2019)Raissi, Perdikaris, and
  Karniadakis]{raissi_physicsinformed_2019}
M.~Raissi, P.~Perdikaris, and G.E. Karniadakis.
\newblock Physics-informed neural networks: {{A}} deep learning framework for
  solving forward and inverse problems involving nonlinear partial differential
  equations.
\newblock \emph{Journal of Computational Physics}, 378:\penalty0 686--707,
  2019.
\newblock \doi{10.1016/j.jcp.2018.10.045}.

\bibitem[Lu et~al.(2021)Lu, Jin, and Karniadakis]{lu_deeponet_2021}
Lu~Lu, Pengzhan Jin, and George~Em Karniadakis.
\newblock {{DeepONet}}: {{Learning}} nonlinear operators for identifying
  differential equations based on the universal approximation theorem of
  operators.
\newblock \emph{Nature Machine Intelligence}, 3\penalty0 (3):\penalty0
  218--229, 2021.
\newblock \doi{10.1038/s42256-021-00302-5}.

\bibitem[Li et~al.(2021)Li, Kovachki, Azizzadenesheli, Liu, Bhattacharya,
  Stuart, and Anandkumar]{li_fourier_2021a}
Zongyi Li, Nikola Kovachki, Kamyar Azizzadenesheli, Burigede Liu, Kaushik
  Bhattacharya, Andrew Stuart, and Anima Anandkumar.
\newblock Fourier {{Neural Operator}} for {{Parametric Partial Differential
  Equations}}.
\newblock In \emph{9th {{International Conference}} on {{Learning
  Representations}} ({{ICLR}} 2021)}, 2021.

\bibitem[Rezende and Mohamed(2015)]{rezende_variational_2015}
Danilo Rezende and Shakir Mohamed.
\newblock Variational inference with normalizing flows.
\newblock In \emph{International conference on machine learning}, pages
  1530--1538. PMLR, 2015.

\bibitem[Ardizzone et~al.(2019)Ardizzone, Kruse, Rother, and
  K{\"o}the]{ardizzone_analyzing_2019}
Lynton Ardizzone, Jakob Kruse, Carsten Rother, and Ullrich K{\"o}the.
\newblock Analyzing {{Inverse Problems}} with {{Invertible Neural Networks}}.
\newblock In \emph{International {{Conference}} on {{Learning
  Representations}}}, 2019.

\bibitem[Song et~al.(2021)Song, {Sohl-Dickstein}, Kingma, Kumar, Ermon, and
  Poole]{song_scorebased_2021}
Yang Song, Jascha {Sohl-Dickstein}, Diederik~P Kingma, Abhishek Kumar, Stefano
  Ermon, and Ben Poole.
\newblock Score-{{Based Generative Modeling}} through {{Stochastic Differential
  Equations}}.
\newblock In \emph{International {{Conference}} on {{Learning
  Representations}}}, 2021.

\bibitem[Lipman et~al.(2023)Lipman, Chen, {Ben-Hamu}, Nickel, and
  Le]{lipman_flow_2023}
Yaron Lipman, Ricky T.~Q. Chen, Heli {Ben-Hamu}, Maximilian Nickel, and Matthew
  Le.
\newblock Flow {{Matching}} for {{Generative Modeling}}.
\newblock In \emph{The {{Eleventh International Conference}} on {{Learning
  Representations}}}, 2023.

\bibitem[Song et~al.(2022)Song, Shen, Xing, and Ermon]{song_solving_2022}
Yang Song, Liyue Shen, Lei Xing, and Stefano Ermon.
\newblock Solving {{Inverse Problems}} in {{Medical Imaging}} with
  {{Score-Based Generative Models}}.
\newblock In \emph{International {{Conference}} on {{Learning
  Representations}}}, 2022.

\bibitem[Wildberger et~al.(2023)Wildberger, Dax, Buchholz, Green, Macke, and
  Sch{\"o}lkopf]{wildberger_flow_2023}
Jonas~Bernhard Wildberger, Maximilian Dax, Simon Buchholz, Stephen~R Green,
  Jakob~H. Macke, and Bernhard Sch{\"o}lkopf.
\newblock Flow {{Matching}} for {{Scalable Simulation-Based Inference}}.
\newblock In \emph{Thirty-seventh {{Conference}} on {{Neural Information
  Processing Systems}}}, 2023.

\bibitem[Adler and {\"O}ktem(2018)]{adler_deep_2018}
Jonas Adler and Ozan {\"O}ktem.
\newblock Deep {{Bayesian Inversion}}.
\newblock \emph{arXiv:1811.05910 [cs, math, stat]}, 2018.

\bibitem[Patel et~al.(2022)Patel, Ray, and Oberai]{patel_solution_2022a}
Dhruv~V. Patel, Deep Ray, and Assad~A. Oberai.
\newblock Solution of physics-based bayesian inverse problems with deep
  generative priors.
\newblock \emph{Computer Methods in Applied Mechanics and Engineering},
  400:\penalty0 115428, 2022.
\newblock \doi{10.1016/j.cma.2022.115428}.

\bibitem[Scholz et~al.(2025)Scholz, Zang, and Koutsourelakis]{scholz_weak_2025}
Vincent~C. Scholz, Yaohua Zang, and Phaedon-Stelios Koutsourelakis.
\newblock Weak neural variational inference for solving bayesian inverse
  problems without forward models: Applications in elastography.
\newblock \emph{Computer Methods in Applied Mechanics and Engineering},
  433:\penalty0 117493, 2025.
\newblock \doi{10.1016/j.cma.2024.117493}.

\bibitem[Cranmer et~al.(2020)Cranmer, Brehmer, and
  Louppe]{cranmer_frontier_2020}
Kyle Cranmer, Johann Brehmer, and Gilles Louppe.
\newblock The frontier of simulation-based inference.
\newblock \emph{Proceedings of the National Academy of Sciences}, 117\penalty0
  (48):\penalty0 30055--30062, 2020.
\newblock \doi{10.1073/pnas.1912789117}.

\bibitem[Kingma and Welling(2014)]{kingma_autoencoding_2014}
Diederik~P. Kingma and Max Welling.
\newblock Auto-{{Encoding Variational Bayes}}.
\newblock In \emph{2nd {{International Conference}} on {{Learning
  Representations}} ({{ICLR}} 2014)}, Banff, AB, Canada, 2014.

\bibitem[Rezende et~al.(2014)Rezende, Mohamed, and
  Wierstra]{rezende_stochastic_2014}
Danilo~Jimenez Rezende, Shakir Mohamed, and Daan Wierstra.
\newblock Stochastic backpropagation and approximate inference in deep
  generative models.
\newblock In \emph{International conference on machine learning}, pages
  1278--1286. PMLR, 2014.

\bibitem[Murphy(2023)]{murphy_probabilistic_2023}
Kevin~P. Murphy.
\newblock \emph{Probabilistic machine learning: advanced topics}.
\newblock Adaptive computation and machine learning series. The MIT Press,
  Cambridge, Massachusetts, 2023.

\bibitem[He et~al.(2019)He, Spokoyny, Neubig, and
  {Berg-Kirkpatrick}]{he_lagging_2019}
Junxian He, Daniel Spokoyny, Graham Neubig, and Taylor {Berg-Kirkpatrick}.
\newblock Lagging {{Inference Networks}} and {{Posterior Collapse}} in
  {{Variational Autoencoders}}.
\newblock In \emph{International {{Conference}} on {{Learning
  Representations}}}, 2019.

\bibitem[Goh et~al.(2022)Goh, Sheriffdeen, Wittmer, and
  {Bui-Thanh}]{goh_solving_2022}
Hwan Goh, Sheroze Sheriffdeen, Jonathan Wittmer, and Tan {Bui-Thanh}.
\newblock Solving {{Bayesian Inverse Problems}} via {{Variational
  Autoencoders}}.
\newblock In \emph{Proceedings of the 2nd {{Mathematical}} and {{Scientific
  Machine Learning Conference}}}, volume 145 of \emph{Proceedings of {{Machine
  Learning Research}}}, pages 386--425. PMLR, 2022.

\bibitem[Sutter et~al.(2020)Sutter, Daunhawer, and
  Vogt]{sutter_multimodal_2020a}
Thomas Sutter, Imant Daunhawer, and Julia Vogt.
\newblock Multimodal generative learning utilizing jensen-shannon-divergence.
\newblock \emph{Advances in neural information processing systems},
  33:\penalty0 6100--6110, 2020.

\bibitem[Kingma and Welling(2013)]{Kingma_2013_AutoEncodingVB}
Diederik~P. Kingma and Max Welling.
\newblock Auto-encoding variational bayes.
\newblock \emph{CoRR}, abs/1312.6114, 2013.
\newblock URL \url{https://api.semanticscholar.org/CorpusID:216078090}.

\bibitem[Arjovsky et~al.(2017)Arjovsky, Chintala, and
  Bottou]{arjovsky_wasserstein_2017}
Martin Arjovsky, Soumith Chintala, and L{\'e}on Bottou.
\newblock Wasserstein generative adversarial networks.
\newblock In \emph{International conference on machine learning}, pages
  214--223. PMLR, 2017.

\bibitem[Nielsen(2011)]{nielsen_family_2011}
Frank Nielsen.
\newblock A family of statistical symmetric divergences based on {{Jensen}}'s
  inequality, December 2011.

\bibitem[Dowson and Landau(1982)]{dowson_frechet_1982}
D~C Dowson and B~V Landau.
\newblock The fr\'echet distance between multivariate normal distributions.
\newblock \emph{Journal of multivariate analysis}, 12\penalty0 (3):\penalty0
  450--455, 1982.

\bibitem[Bhatia et~al.(2019)Bhatia, Jain, and Lim]{bhatia_bures_2019a}
Rajendra Bhatia, Tanvi Jain, and Yongdo Lim.
\newblock On the {{Bures}}--{{Wasserstein}} distance between positive definite
  matrices.
\newblock \emph{Expositiones Mathematicae}, 37\penalty0 (2):\penalty0 165--191,
  2019.
\newblock \doi{10.1016/j.exmath.2018.01.002}.

\bibitem[Villani(2008)]{villani_optimal_2008}
C{\'e}dric Villani.
\newblock \emph{Optimal {{Transport}}: {{Old}} and {{New}}}.
\newblock Springer Science \& Business Media, 2008.

\bibitem[Walter and Pronzato(1997)]{walter_identification_1997}
E.~Walter and Luc Pronzato.
\newblock \emph{Identification of Parametric Models from Experimental Data}.
\newblock Communications and control engineering. Springer ; Masson, Berlin ;
  New York : Paris, 1997.

\bibitem[Higgins et~al.(2017)Higgins, Matthey, Pal, Burgess, Glorot, Botvinick,
  Mohamed, and Lerchner]{higgins_betavae_2017}
Irina Higgins, Loic Matthey, Arka Pal, Christopher Burgess, Xavier Glorot,
  Matthew Botvinick, Shakir Mohamed, and Alexander Lerchner.
\newblock Beta-{{VAE}}: {{Learning Basic Visual Concepts}} with a {{Constrained
  Variational Framework}}.
\newblock In \emph{International {{Conference}} on {{Learning
  Representations}}}, 2017.

\bibitem[Kingma and Ba(2015)]{kingma_adam_2015}
Diederik~P. Kingma and Jimmy Ba.
\newblock Adam: A method for stochastic optimization.
\newblock In Yoshua Bengio and Yann LeCun, editors, \emph{3rd International
  Conference on Learning Representations, ICLR 2015, San Diego, CA, USA, May
  7-9, 2015, Conference Track Proceedings}, 2015.

\bibitem[Scroggs et~al.(2022)Scroggs, Dokken, Richardson, and
  Wells]{ScroggsEtal2022}
Matthew~W. Scroggs, J{\o}rgen~S. Dokken, Chris~N. Richardson, and Garth~N.
  Wells.
\newblock Construction of arbitrary order finite element degree-of-freedom maps
  on polygonal and polyhedral cell meshes.
\newblock \emph{ACM Transactions on Mathematical Software}, 48\penalty0
  (2):\penalty0 {18:1--18:23}, 2022.
\newblock \doi{10.1145/3524456}.

\bibitem[Smith(2013)]{smith_uncertainty_2013}
Ralph~C. Smith.
\newblock \emph{Uncertainty quantification: theory, implementation, and
  applications}.
\newblock Computational science and engineering series. {Society for Industrial
  and Applied Mathematics}, Philadelphia, 2013.

\bibitem[Bishop(2006)]{bishop_pattern_2006}
Christopher~M. Bishop.
\newblock \emph{Pattern Recognition and Machine Learning}.
\newblock Information science and statistics. Springer, New York, 2006.

\bibitem[Haario et~al.(2001)Haario, Saksman, and
  Tamminen]{haario_adaptive_2001}
Heikki Haario, Eero Saksman, and Johanna Tamminen.
\newblock An {{Adaptive Metropolis Algorithm}}.
\newblock \emph{Bernoulli}, 7\penalty0 (2):\penalty0 223, 2001.
\newblock \doi{10.2307/3318737}.

\bibitem[Gelman et~al.(1996)Gelman, Meng, and Stern]{gelman_posterior_1996}
Andrew Gelman, Xiao-Li Meng, and Hal Stern.
\newblock Posterior predictive assessment of model fitness via realized
  discrepancies.
\newblock \emph{Statistica sinica}, pages 733--760, 1996.

\bibitem[Vehtari et~al.(2021)Vehtari, Gelman, Simpson, Carpenter, and
  B{\"u}rkner]{vehtari2021rank}
Aki Vehtari, Andrew Gelman, Daniel Simpson, Bob Carpenter, and Paul-Christian
  B{\"u}rkner.
\newblock Rank-normalization, folding, and localization: An improved $\hat{R}$
  for assessing convergence of mcmc (with discussion).
\newblock \emph{Bayesian analysis}, 16\penalty0 (2):\penalty0 667--718, 2021.

\bibitem[Gelman et~al.(2014)Gelman, Carlin, Stern, Dunson, Vehtari, and
  Rubin]{gelman_bayesian_2014}
Andrew Gelman, John~B. Carlin, Hal~Steven Stern, David~B. Dunson, Aki Vehtari,
  and Donald~B. Rubin.
\newblock \emph{Bayesian data analysis}.
\newblock CRC Press, Boca Raton, third edition edition, 2014.

\bibitem[Kumar et~al.(2019)Kumar, Carroll, Hartikainen, and
  Martin]{kumar2019arviz}
Ravin Kumar, Colin Carroll, Ari Hartikainen, and Osvaldo Martin.
\newblock Arviz a unified library for exploratory analysis of bayesian models
  in python.
\newblock \emph{Journal of Open Source Software}, 4\penalty0 (33):\penalty0
  1143, 2019.

\bibitem[Barber(2012)]{barber_bayesian_2012}
David Barber.
\newblock \emph{Bayesian Reasoning and Machine Learning}.
\newblock Cambridge University Press, Cambridge ; New York, 2012.

\end{thebibliography}
\end{document}